\pdfoutput=1
\documentclass[11pt]{article}

\usepackage[utf8]{inputenc}
\usepackage[T1]{fontenc}
\usepackage{booktabs}
\usepackage{amsfonts}
\usepackage{nicefrac}
\usepackage{microtype}
\usepackage{wrapfig}
\usepackage[table]{xcolor}
\PassOptionsToPackage{hyphens}{url}

\usepackage{./smile}
\usepackage{titlesec}

\colorlet{sectionblue}{templatecolor}

\titleformat{\section}
  {\Large\bfseries\color{black}}{\thesection}{1em}{}
\titleformat{\subsection}
  {\large\bfseries\color{black}}{\thesubsection}{1em}{}
\titleformat{\subsubsection}
  {\normalsize\bfseries\color{black}}{\thesubsubsection}{1em}{}
\titleformat{\paragraph}[runin]
  {\normalfont\normalsize\bfseries\color{black}}{}{0pt}{}

\mathtoolsset{showonlyrefs}

\newcommand{\bpf}{\bP_{\mathrm{f}}}
\newcommand{\bps}{\bP_{\mathrm{s}}}
\newcommand{\gmax}{\gamma_{\mathrm{max}}}
\newcommand{\gflat}{\gamma_{\mathrm{flat}}}

\usepackage{fancyhdr}
\newcommand\blfootnote[1]{%
  \begingroup
  \renewcommand\thefootnote{}\footnote{#1}%
  \addtocounter{footnote}{-1}%
  \endgroup
}

\title{
  \vskip-30pt
  \textbf{Towards Understanding Momentum Acceleration in River-Valley Loss Landscape}
}

\author{
  Miao Lu$^{1}$ \quad
  Zeyu Bian$^{2}$ \quad
  Kaiyue Wen$^{1}$ \quad
  Beining Wu$^{3}$ \\
  Siyu Chen$^{4}$ \quad
  Tianhao Wang$^{2}$ \quad
  Zhiyuan Li$^{5}$ \\[6pt]
  {\small $^{1}$Stanford University \quad $^{2}$University of California, San Diego \quad $^{3}$University of Chicago} \\[1pt]
  {\small $^{4}$Yale University \quad $^{5}$Toyota Technological Institute at Chicago}
}
\date{}

\begin{document}

\maketitle

\blfootnote{%
  The main work was done while ML was visiting TTIC.
  Author emails:
  \texttt{\{miaolu,kaiyuew\}@stanford.edu},
  \texttt{\{zebian,tianhaowang\}@ucsd.edu},
  \texttt{beiningw@uchicago.edu},
  \texttt{siyu.chen.sc3226@yale.edu}, and
  \texttt{zhiyuanli@ttic.edu}.
}

\pagestyle{plain}

\vspace{-10mm}
\begin{abstract}
The empirical success of pretraining large language models has inspired a deeper investigation into the underlying loss landscapes and the optimization dynamics. Recent empirical and theoretical study suggest that the training loss landscape often exhibits a ``river-valley'' structure, which features a low-loss manifold (river) flanked by sharp orthogonal directions with higher loss (mountains). In the long term, the optimization progress is determined primarily by the progress along the river. Within such a landscape, gradient descent with large learning rates can move faster along the river despite high apparent loss due to vertical oscillations, while a subsequent sharp decay in the learning rate suppresses these oscillations, revealing genuine optimization progress. This explains the recent success of warmup-stable-decay (WSD) learning rate scheduler which, unlike cosine scheduling, keeps stable high learning rate and decays before producing intermediate checkpoints.
Building on this foundation, in this work we take a step further and study the role of momentum within such a loss landscape. We establish theoretical analysis that characterizes how momentum accelerates optimization by stabilizing large learning rates that can not be tolerated by vanilla GD without deviating significantly from the river.
The enabled large learning rate in-turn gives greater speed along the river and makes faster essential progress in the long run. Another intriguing observation from theory is that for a river-valley landscape with very flat and slow-spinning river, the momentum itself does not contribute directly to acceleration in terms of the speed of tracking the river, while the main acceleration comes from the admissible larger learning rate. We further provide empirical validations, including experiments on synthetic functions and language model pretraining, confirming the predicted behaviors and offering insights into the joint tuning of learning rate and momentum.
\end{abstract}

\section{Introduction}

Momentum is a foundational component of modern optimizers and is widely used in training large neural networks. 
Variants of gradient descent with heavy-ball momentum \citep{polyak1964some}, e.g., Adam \citep{kingma2014adam} and Muon \citep{jordan2024muon}, are standard choices of optimizer in practice.
Despite this ubiquity, however, the theoretical conditions under which momentum actually accelerates learning remain poorly understood. 
Classical theory can only guarantee the acceleration of heavy-ball momentum within restricted settings with quadratic-like properties via improving the exponent in linear convergence \citep{polyak1964some,kassing2024polyak}.
However, the loss landscapes of neural networks are typically highly non-convex and linear convergence is rarely observed in practice. As a result, the conventional intuition -- that momentum accelerates local linear convergence -- may not fully explain its success in deep learning, calling for rethinking of the role of momentum in complex loss landscapes.
In particular, theoretically answering this question requires understanding momentum in the context of a tractable but still realistic and meaningful loss landscape beyond convex settings.

\begin{wrapfigure}[20]{r}{0.41\textwidth}
    \centering
    \vspace{-0.7\baselineskip}
    \captionsetup{width=\linewidth,font=small,justification=raggedright,
    singlelinecheck=false}
    \includegraphics[width=\linewidth]{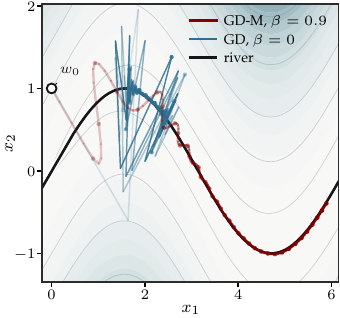}
    \caption{GD-M stabilizes a river-valley trajectory. At $\eta=0.4$, GD-M
    ($\beta=0.9$, red) damps transverse oscillations and moves along the black
    river, whereas GD ($\beta=0$, blue) keeps oscillating. Opacity increases
    with iteration.}
    \label{fig: river visualization}
    \vspace{-0.5\baselineskip}
\end{wrapfigure}
A recent work on optimization in language model pretraining has identified and proposed an intriguing ``river-valley’’ structure of the loss landscape \citep{wen2025understanding}: it contains a deep valley with a narrow, low-loss manifold at the bottom (``river''), surrounded by orthogonal steep directions (``mountains'').
Empirically, this loss landscape is evidenced by comparing the loss curves of the Warmup-Stable-Decay (WSD) learning rate scheduler \citep{hu2024minicpm} and the cosine learning rate scheduler \citep{loshchilov2017sgdr}: WSD first uses a constant large learning rate, causing faster descent along the river than cosine scheduler but with higher loss due to oscillations across valley walls. 
In the end, a sharp decay of the learning rate can suppress the oscillations, revealing the genuine optimization progress along the river by WSD and resulting in lower final loss than cosine scheduler. \citet{wen2025understanding} also probe LLM pretraining loss landscape to showcase the existence of such a structure. 
We illustrate such a landscape through a synthetic loss function in Figure~\ref{fig: river visualization}.

Within such a landscape of ``a deep valley with a river at its bottom’’, the long-term optimization progress is essentially governed by the motion along the river, making the main challenge how to move quickly along the low-loss river manifold while damping out the instability in the high-curvature valley directions. 
Motivated by the lack of theoretical underpinning of momentum in complex loss landscapes, in this paper, we theoretically analyze the dynamics of heavy-ball momentum gradient descent under the emerging setting of river-valley loss landscape. 
Within such a landscape, we ask the following questions: 
\vspace{-2mm}
\begin{center}
    \emph{What is the role of momentum?  How does the momentum parameter interplay with  learning rate?  \\ 
    How do they jointly accelerate optimization along the river?}
\end{center}

\subsection{Our Contributions}
Theoretically answering these questions requires a new analysis for GD with momentum in such a \emph{non-convex} loss landscape, especially, how momentum shapes the interactions between the benign progress in the river and the oscillating updates in the valley directions.
In this paper, we investigate the following GD iteration with heavy-ball momentum: given loss function $L$ and initial parameter and momentum $(w_0,m_0)$, the iteration goes as follows:
\begin{align}
   m_{k+1} = \beta\cdot m_k + (1-\beta)\cdot\nabla L(w_k),\quad w_{k+1} = w_k - \eta\cdot m_{k+1}.  \label{eq: gd momentum intro}
\end{align}
\paragraph{Mechanism of momentum acceleration in river-valley.} We develop a new picture for understanding \eqref{eq: gd momentum intro} in the river-valley loss landscape by proving that momentum can enlarge the largest tolerable learning rate by the leading factor $(1+\beta)/(1-\beta)$, i.e., $\eta_{\max}^{\mathrm{GD}\text{-}\mathrm{M}}\approx (1+\beta)/(1-\beta)\cdot\eta_{\max}^{\mathrm{GD}}$.
See Theorem~\ref{thm: main}.
This factor is classical for heavy-ball dynamics on quadratic losses; our contribution is to show that it continues to govern the leading stability threshold in a non-convex river-valley landscape, where the flat and sharp eigenspaces rotate and interact along the trajectory.
Compared to vanilla GD, momentum accelerates the progress along the river by enabling the use of a substantially larger learning rate.
In fact, momentum acts as a stabilizer in the presence of oscillations caused by aggressive choices of learning rate, which vanilla GD cannot tolerate.
The larger learning rate enables faster traversal along the river manifold, achieving acceleration.
Moreover, for flat and slow-spinning river, vanilla heavy-ball momentum itself does not alter the speed on the river much; the acceleration mainly comes from the choice of a larger learning rate.

\paragraph{Technical contributions for analyzing momentum.} 
We introduce new arguments to prove, under the largest tolerable learning rate, how the GD-momentum iteration can closely track the river without exploding in the face of oscillations across the valley directions, see Lemma~\ref{lem: flat dominance momentum main}. 
The proof introduces new induction argument to track the scale of both the gradient and the momentum terms projected to river and valley directions. 
We note that both the river and the valley directions of the gradient and the momentum interfere with each other, casting new and different challenges to tightly track the magnitudes of the gradient (and the momentum) compared to analyzing vanilla GD in river-valley landscape \citep{wen2025understanding}. 
The new analysis could also be of independent interest for future research on more advanced adaptive optimizers with momentum, e.g., Adam, Muon, in the river-valley loss landscape.
As a by-product of the spirit of the new induction argument, we also achieve a substantial improvement in terms of the largest learning rate for vanilla GD compared to \citet{wen2025understanding}, see Remark~\ref{rem: gd comparison}.
This actually deepens our understanding of the river-valley model in terms of the largest tolerable learning rate for vanilla GD and in turn helps to characterize the benefit of momentum.

\paragraph{Experiment on synthetic loss \& language model training.} To verify our theory, we first utilize a synthetic non-convex function with river-valley structure as a playground to check all the theoretical predictions.
The experiments (Figure~\ref{fig: experimemnt 1}) perfectly match our predictions on: (i) how momentum can accelerate the optimization; and (ii) the scaling of the largest tolerable learning rate with respect to the momentum parameter.
Moving forward to training language models, we verify our theoretical predictions by pretraining GPT-Neo \citep{gao2020pile} on the TinyStories dataset \citep{eldan2023tinystories} with (i) SGD-Momentum and (ii) Adam.
We observe a matching relationship between the largest tolerable learning rate and the momentum parameter $\beta$ in SGD-Momentum, and a similar trend with respect to $\beta_1$ in Adam.
All these LM training results verify our theoretical predictions on how momentum achieves acceleration by enabling a large learning rate training for faster traverse along the river manifold.

\paragraph{Paper outline.}
The paper is organized as following. 
The remaining of this section gives related works.
In Section~\ref{sec: preliminaries}, we introduce the river-valley loss landscape.
We review the theory on vanilla GD in river-valley landscape in Section~\ref{sec: gd}. 
Then we establish our main theory on GD with heavy-ball momentum in Section~\ref{sec: main}. 
We sketch the proofs of our results in Section~\ref{sec: proof outline}. 
The experiments on language model pretraining are in Section~\ref{sec: llm experiment}.

\subsection{Related Works}

\paragraph{Theoretical analysis of momentum.}
Momentum acceleration in convex optimization dates back to Polyak's heavy-ball method \citep{polyak1964some}.
Provable acceleration of heavy-ball momentum and its stochastic variants has been established for locally near-quadratic problems, including convex settings \citep{ghadimi2015global,liu2020improved,sebbouh2021almost,li2022last,dang2024noise,bollapragada2024fast,panaccelerated,wei2024accelerated} and non-convex objectives satisfying a PL condition \citep{wang2021modular,pmlr-v162-wang22p,kassing2024polyak}.
Beyond these structured regimes, however, our understanding of momentum in the non-convex landscapes arising in deep learning remains limited.
\citet{kim2026understanding} show that standard stationarity-based non-convex theory can yield worst-case comparisons unfavorable to momentum, motivating analyses with additional landscape structure.
For neural networks, \citet{wang2021modular} prove acceleration for wide ReLU and deep linear networks, while \citet{Plattner2022,wang2024the} study the marginal value of momentum in the small-learning-rate regime.
Moving beyond this regime, \citet{fu2023when} empirically observe that SGDM begins to outperform SGD at a matched effective learning rate only after this rate crosses a problem-dependent threshold, and associate this separation with momentum delaying abrupt sharpening.
\citet{phunyaphibarn2024catapults} study heavy-ball dynamics with large learning rates and provide evidence that momentum can prolong self-stabilizing catapults and bias the trajectory toward flatter solutions.
Concurrent work \citep{andreyev2026momentum} characterizes batch-size-dependent stochastic stability thresholds for momentum SGD.
These works focus on empirical trajectory phenomena or stochastic stability, whereas we prove stable river tracking and characterize the deterministic learning-rate enlargement caused by momentum in a rotating non-convex valley.
Besides heavy-ball momentum, another line of work studies Nesterov's momentum \citep{nesterov1983method,liu2022convergence,xu2024provable,liao2024provable}.

\paragraph{Large learning rate training.} 
The benefits of large-learning-rate training have received increasing attention \citep{he2016deep,xing2018walk,cohen2021gradient}.
Large learning rates can provably accelerate optimization in logistic regression \citep{wu2024large,meng2024gradient,zhang2025minimax} and non-homogeneous two-layer networks \citep{cai2024large}.
In the river-valley model, \citet{wen2025understanding} show that a larger learning rate can accelerate progress along the low-loss river manifold.
Work on the ``Edge of Stability'' studies how large learning rates shape oscillatory training, sharpness, and implicit bias \citep{cohen2021gradient,jastrzebski2021catastrophic,arora2022understanding,kong2020stochasticity,zhuunderstanding,wanglarge,andriushchenko2023sgd,wang2023good,wu2023implicit,lubenign,xu2024implicit}.
More recently, \citet{macdonald2025eos} quantify large-step dynamics in overparameterized least squares by decomposing the dynamics into components parallel and orthogonal to the manifold of minimizers.
Momentum is no longer absent from this line of literature: \citet{phunyaphibarn2024catapults} and \citet{andreyev2026momentum} study large-learning-rate momentum through catapult and stochastic-stability mechanisms, respectively.
Unlike these analyses of sharpness selection or stochastic stability, we establish stable tracking of an explicit river manifold and characterize the leading tolerable-learning-rate gain due to heavy-ball momentum.

\paragraph{Loss landscape in deep learning.} 
Understanding the interaction between loss geometry, optimization dynamics, and generalization is central to explaining the success of deep learning \citep{freeman2017topology,garipov2018loss}.
The works most closely related to ours consider ``river-valley''-like geometry: \citet{xing2018walk} propose a picture in which SGD bounces across a valley while moving along its floor; \citet{davis2024gradient} identify a manifold with vanishing gradient in the sharp Hessian directions under fourth-order growth; and \citet{wen2025understanding} formalize a river-valley model for language-model pretraining.
Together, these works motivate and provide evidence for the river-valley perspective.
Recent work further connects optimizer dynamics with low-curvature landscape geometry.
\citet{cohen2025central} model time-averaged Edge-of-Stability trajectories through central flows, \citet{dremov2025cooldown} empirically visualize river-valley structure during WSD cooldown, and \citet{song2025through} use the river perspective to analyze schedule-free language-model training.
These works support or extend the geometric picture, whereas our focus is a trajectory-level analysis of heavy-ball momentum under explicit river-valley assumptions.

\paragraph{Further related works.} We refer the reader to Appendix~\ref{sec: further related works} for further discussion of related work.

\section{Preliminaries}
\label{sec: preliminaries}

\paragraph{Notations.} We denote the eigenvalues of  $\bA\in\RR^{d\times d}$ as $\{\lambda_k(\bA)\}_{k=1}^d$  with $\lambda_1(\bA)\geq \cdots\geq \lambda_d(\bA)$. 
We denote $\{v_k(\bA)\}_{k=1}^d$ as the normalized eigenvectors of $\bA$ with $v_k(\bA)$ corresponding to $\lambda_k(\bA)$.
We denote $\|\bA\|_{\mathrm{Op}}$ as the operator norm of $\bA$.
For the third derivative tensor, we define $\|\nabla^3 L(w)\|_{\mathrm{Op}}:=\sup_{\|u\|_2=\|v\|_2=\|z\|_2=1}|\nabla^3 L(w)[u,v,z]|$, where $\nabla^3 L(w)[u,v,z]$ denotes the corresponding trilinear form.
We denote $\lfloor t \rfloor$ as the largest integer smaller or equal to $t$, and $\cB(w, r)$ as the Euclidean ball of radius $r$ centered at $w$. 
For $\beta\in(0,1)$, we use $\cO_{\beta}(\cdot)$ to hide factors polynomial in $1/(1-\beta)$.

\subsection{River-Valley Loss Landscape: Setups and Key Assumptions}

We use $w\in\RR^d$ to denote the parameters. 
We denote $L:\RR^d\to\RR$ as the loss function and assume $L$ is smooth.
The concept of the river-valley loss landscape is as follows.

\begin{assumption}[Existence of the manifold of river]\label{ass: existence}
    There exists a one-dimensional sub-manifold $\cM$ of $\mathbb{R}^d$ s.t. for $\forall w\in\cM$, $\nabla L(w)$ aligns with $v_d(\nabla^2L(w))$, i.e., $\nabla L(w)/\|\nabla L(w)\|_2 = v_d(\nabla^2L(w))$.
\end{assumption}

In words, along the river, the gradient $\nabla L(w)$ aligns with the 
flattest direction at $w$, i.e., $v_d(\nabla^2 L(w))$, which we refer to as the river direction or the \emph{flat direction}.
All the other directions
orthogonal to the river are dubbed as mountain directions or \emph{sharp directions}, corresponding to the steep
valley.
Intuitively, the river captures the path with the lowest loss locally, while the hill components reflect the additional loss incurred by the deviations from the river. 
Such a structure is ubiquitous in deep learning loss landscapes \citep{xing2018walk, dang2024noise, wen2025understanding}.

Conceptually, the optimization process goes as follows: 
The parameter $w$ starts from a neighborhood $\cU$ of the river $\cM$, and approaches the river $\cM$ by decreasing the loss.
Then the optimization proceeds by flowing downstream for further optimization progress, which in the long term, is determined primarily by the progress along the river.
To rigorously characterize this process, we introduce technical assumptions following the prior work \citet{wen2025understanding}.

\begin{assumption}[Existence and regularity of an open neighborhood around river]\label{ass: regularity}
    There exist an open set $\cU$ with $\cM\subset\cU$ and constants $\gflat,\gamma,\gmax,g_{\max},\kappa \geq 0$, s.t.: 
    \begin{enumerate}[nosep, leftmargin=5mm]
        \item \textbf{Diameter of}  $\cU$: For any $w\in\cM$, $\cB(w, 6g_{\max}/\gamma)\subset\cU$, where $\|\nabla L(u)\|_2\leq g_{\max}$ for all $u\in\cU$.
        \item \textbf{Eigen-gap}: $\gflat, \gamma, \gmax$, $\kappa$ satisfy that $\kappa\gamma\leq \gamma_{\mathrm{flat}} \leq \kappa^{1/2}\gmax$, $\kappa^{1/32}\gmax\leq\gamma\leq\gmax/200$, and for any $w\in\cU$, $\lambda_1(\nabla^2L(w))\leq \gamma_{\mathrm{max}}$, $\lambda_{d-1}(\nabla^2 L(w))>\gamma + 4\gamma_{\mathrm{flat}}$, and $|\lambda_d(\nabla^2L(w))|<\gamma_{\mathrm{flat}}$.
        \item \textbf{Slow spinning of flat direction and Hessian}: 
        For any $w\in\cU$, $ \|\nabla(v_d(\nabla^2 L(w)))\|_{\mathrm{Op}}\leq\kappa\gamma/2g_{\max}$ and $\|\nabla^3L(w)\|_{\mathrm{Op}}\leq \kappa\gamma^2/2g_{\max}$.
        \item \textbf{Uniqueness of the river} $\cM$: 
        For any $w\in\cU\setminus\cM$, $\nabla L(w)$ does not align with $v_d(\nabla^2 L(w))$.
    \end{enumerate}
\end{assumption}

We explain Assumption~\ref{ass: regularity} as follows.
The first one requires that the size of the neighborhood is not too small and fixes a uniform gradient bound on $\cU$.
The last one ensures the uniqueness of river in the neighborhood for the simplicity of analysis.
The second and the third assumptions are key to the river-valley model, which we explain in detail.
The second \emph{Eigen-gap} assumption posts conditions on the relative magnitude of the Hessian eigenvalues. It assumes a gap of $\gamma$ between the sharpness of the river direction and the other sharp directions, which formalizes the picture of a flat river and steep mountains. 
The third \emph{Slow spinning} assumption limits the spinning speed of the river by imposing third-order conditions on the loss.
Note that under the \emph{Eigen-gap} assumption, the bound on $\|\nabla^3 L(w)\|_\mathrm{Op}$ implies the bound on the spinning speed of the flat direction.

Throughout our analysis, both $\kappa\ll 1$ and $\gflat \gmax^{-1}\ll 1$ are treated as \emph{very small dimensionless constants}, guaranteeing the flatness as well as slow spinning of the river.
The relation $\kappa\gamma\leq\gflat$ allows the geometric errors caused by the rotation of the flat eigenspace to be absorbed into flat-direction curvature terms, whereas $\gflat\leq\kappa^{1/2}\gmax$ enforces a separation between the flat and sharp curvature scales.
The lower bound $\gamma/\gmax\geq\kappa^{1/32}$ keeps the eigengap sufficiently large relative to the Hessian-variation scale and is used in the uniform sharp-direction stability estimates.
The fixed upper bound $\gamma/\gmax\leq 1/200$ is used only in the neighborhood-containment argument for the interpolated trajectory in Appendix~\ref{subsec: proof gd momentum tracks the river closely}; its numerical value can be traded against the neighborhood radius in item~1.
The fractional exponents and numerical constants in these scale conditions are chosen for the proof and are not optimized.

\begin{definition}[River-valley landscape]
    If the loss function $L$ satisfies Assumption~\ref{ass: existence} and Assumption~\ref{ass: regularity}, we call the corresponding loss landscape a river-valley landscape.
\end{definition}

\paragraph{Reference flow.} 
Next, to formalize how close the iterates are to the river and how fast they flow downstream, we introduce the concept of the \emph{reference flow}. 
Specifically, the reference flow is a Riemannian gradient flow on the river, representing the dynamics of gradient flow on the loss when constrained to the river.
We denote the projection matrix onto the tangent space of the manifold $\cM$ at $w\in\cM$ as $\bP_{\cM}(w)$.
Given an initial point $x\in\cM$, the Riemannian gradient flow is given by
\begin{align}
    \frac{\mathrm{d}}{\mathrm{d}t}x(t) = -\bP_{\cM}(x(t))\nabla L(x(t)),\quad x(0)=x.
    \label{eq: manifold parametrization}
\end{align}
Throughout the paper, $x(t)$ always refers to points on the river $\cM$ or the reference flow, and $t$ is the continuous time index of the reference flow \eqref{eq: manifold parametrization}.

\paragraph{Examples of river-valley loss landscape.}
We provide concrete examples to illustrate the landscape.
The simplest one is a convex function $L(x_1,x_2) = \gamma\cdot x_1^2 / 2 - x_2$, where the river is just $x_1=0$ and it does not spin, i.e., $\kappa=0$.
Another example is a non-convex function 
$$L(x_1,x_2)=(x_2-\sin(x_1))^2-\alpha \cdot x_1$$ 
for some $\alpha\geq 0$. 
We visualize this example in the Figure~\ref{fig: river visualization}. 
The optimization progress for it is essentially determined by how fast the variable $x_1$ grows by tracking the river.

Moreover, in the context of language modeling, \citet{wen2025understanding} introduces a synthetic bigram language model to theoretically verify the existence of such landscape. 
They prove that the sharpness of the next-token prediction loss landscape correlates with uncertainty in the token distributions.  
The deterministic tokens, which are easier to predict, correspond to the flat directions of the loss, while the more stochastic tokens contribute to the sharper directions.

\subsection{Gradient Descent with Heavy-Ball Momentum}

Our theoretical analysis concerns vanilla GD and GD with heavy-ball momentum.
Starting from some $w_0\in\RR^d$, the vanilla GD iterates as following,
\begin{align}
    w_{k+1} &= w_k - \eta\cdot \nabla L(w_k),\tag{GD}\label{eq: gd} 
\end{align}
where $\eta$ denotes the learning rate.
Furthermore, for GD with heavy-ball momentum \citep{polyak1964some}, when initialized at some $(w_0, m_0)\in\mathbb{R}^d\times\mathbb{R}^d$, it iterates as 
\begin{align}
    m_{k+1} = \beta\cdot m_k + (1-\beta)\cdot\nabla L(w_k),\quad w_{k+1} = w_k - \eta\cdot m_{k+1},   \label{eq: gd momentum}\tag{GD-M}
\end{align}
where $\beta\in[0,1)$ is the momentum parameter.
In particular, when $\beta=0$, \eqref{eq: gd momentum} reduces to vanilla \eqref{eq: gd}.
There is another version of GDM where the $1-\beta$ is absent in the update of momentum, but it is equivalent to the above formulation up to simple reparametrization.

\section{Review: What Happens for Vanilla GD?}\label{sec: gd}

Before giving our results on \eqref{eq: gd momentum}, we first review what happens for vanilla 
\eqref{eq: gd} in the river-valley landscape.
As we mentioned in the previously, GD starting from the neighborhood would (i) converge close to the river within distance of order $\cO(\kappa/\gamma)$ (Theorem 3.1 of \citet{wen2025understanding}), and then (ii) closely track the river afterwards.
We recap the results for the second stage as follows.


\begin{theorem}[Informal, GD in river-valley landscape (improved version of Theorem 3.3 in \citealt{wen2025understanding})]\label{thm: gd}
    Suppose Assumptions~\ref{ass: existence} and \ref{ass: regularity} hold. 
    Let $\eta$ be a learning rate such that $\eta<\eta_{\max}^{\mathrm{GD}}$, where
    \begin{align}\label{eq:max_lr_gd}
        \eta_{\max}^{\mathrm{GD}}\approx \frac{2-\cO(\kappa + \gflat\gmax^{-1})}{\gmax}.
    \end{align}
    Then there exists a time index $T_0$ such that \eqref{eq: gd} with initialization $w_0\in\cM$ on the river satisfies that for any step $k$, there exists  $T(k)$ s.t. the following two things hold:
    \begin{enumerate}[nosep, leftmargin=6mm]
        \item GD stays close to the river: $\|x(T_0 + T(k))-w_k\|_2\leq \cO(\kappa/\gamma)$; 
        \item The speed on river is nearly proportional to $\eta$: $|T(k) - \eta k| \leq \cO(\kappa + \eta\gflat)\cdot \eta k$.
    \end{enumerate}
\end{theorem}

See formal version in Appendix~\ref{subsec: formal statement gd}. 
By Theorem~\ref{thm: gd}, when GD closely tracks the river, the larger the learning rate (if tolerable), the faster the iterates move along the river, which means faster optimization progress.
The theory partially demystifies the success and the non-trivial loss curve of WSD scheduler in practice \citep{hu2024minicpm}.
Using a large constant learning rate induces a larger training loss before learning rate decay due to oscillations in the sharp directions, but after the learning rate decay, the loss reflects the true progress it has made (Theorem 3.5 of \citet{wen2025understanding}).


One key factor of Theorem~\ref{thm: gd} is the upper bound on the learning rate, which ensures that the iterates remain close to the river without exploding due to oscillations between the valley.
Intuitively, the bottleneck of using a large learning rate lies in \emph{the curvature of sharp directions} and \emph{the spinning of the river}, for which $\eta_{\max}^{\mathrm{GD}}$ is derived.
Therefore, the dominant term is roughly $\eta_{\max}^{\mathrm{GD}} \approx 2/\gmax$, which matches the maximum learning rate for GD to optimize a quadratic loss whose largest Hessian eigenvalue is $\gmax$.

\begin{remark}[Comparison with \citet{wen2025understanding}]\label{rem: gd comparison}
    Theorem~\ref{thm: gd} improves the largest tolerable learning rate $\eta^{\mathrm{GD}}_{\max}$ compared to the original $\gamma/2\gmax^2$ in \citet{wen2025understanding}. 
    The new proofs are in Appendix~\ref{subsec: improved analysis}. 
    The improvement results from a tighter analysis of sharp direction dynamics.
\end{remark}

Given all the results so far, the natural questions are: how does the inclusion of momentum in the optimization iteration affect the dynamics of tracking the river, and how does the additional momentum parameter interplay with the learning rate?
We answer these questions in Section~\ref{sec: main}.

\section{Analysis of GD with Momentum in River-Valley}\label{sec: main}

\subsection{Main Theory: Momentum Acceleration in River-Valley Landscape}\label{subsec: main results}

We investigate \eqref{eq: gd momentum} near the river by assuming that the initialization $w_0$ is already on the river $w_0\in\cM$.
Our main theorem predicts how close and fast \eqref{eq: gd momentum} can track the river manifold.
For ease of presentation, we state an informal version of the main theorem as follows.


\begin{theorem}[GD-momentum in river-valley (informal version of Theorem~\ref{thm: main formal})]\label{thm: main}
    Suppose that Assumptions~\ref{ass: existence} and \ref{ass: regularity} hold.
    For a momentum parameter $\beta\in(0,0.99]$, let $\eta<\eta^{\mathrm{GD}\text{-}\mathrm{M}}_{\max}$, where
    \begin{align}\label{eq:max_lr_gdm}
        \hspace{-5mm}\eta^{\mathrm{GD}\text{-}\mathrm{M}}_{\max} \approx \left(\frac{1+\beta}{1-\beta} -\mathcal{O}_{\beta}\left(\kappa^{1/16}+\gflat\gmax^{-1}\right)\right)\cdot \frac{2}{\gmax}.
    \end{align}
    Define the post-transient iteration $k_\star:=\max\{0,\lceil\log(\beta\eta\gflat/(1-\beta))/\log\beta\rceil\}$.
    Then for \eqref{eq: gd momentum} with initialization $w_0\in\cM$ on the river and initial momentum $m_0=0$, there is a constant time shift $T_0$ such that for any integer $k\geq k_\star$,
    there exists another $T(k)$ s.t. the following two things hold:
    \begin{enumerate}[nosep, leftmargin=6mm]
        \item GD-M stays close to the river: $\|x(T_0 + T(k))-w_k\|_2 = \cO_\beta(\kappa^{3/4}/\gamma)$;
        \item The speed on the river is proportional to the learning rate: $|T(k) - \eta k| = \cO_\beta((\kappa+\eta\gflat)\cdot\eta k)$.
    \end{enumerate}
\end{theorem}

Here we omit details of the error terms to convey the essential message in a clean way.
Although Theorem~\ref{thm: main} is stated for positive momentum, setting $\beta=0$ in the update rule \eqref{eq: gd momentum} reduces it to \eqref{eq: gd}, and the leading learning-rate threshold in \eqref{eq:max_lr_gdm} is consistent with Theorem~\ref{thm: gd}.
We refer to Theorem~\ref{thm: main formal} for the complete statement of the above theorem.
We outline the proof sketch of Theorem~\ref{thm: main} in Section~\ref{sec: proof outline}, with all details in Appendix~\ref{sec: key proofs of main theorem}.
Appendix~\ref{sec: further discussions} provides an ODE interpretation of Theorem~\ref{thm: main}.
In the remainder of this section, we discuss the interpretation of the main theorem, as well as its predictions.

\subsection{Theoretical Predictions by Theorem~\ref{thm: main} and Numerical Verification}\label{subsec: predictions}

For a constant $\beta\in(0,1)$, e.g. $\beta=0.9$ as the default choice used in practice, $\cO_\beta(\cdot)$ is a constant.
Then since $\kappa$ and $\gflat$ are small, the error terms in Theorem~\ref{thm: main} are negligible.

\begin{figure*}[!htbp]
    \centering
    \begin{minipage}[t]{0.32\textwidth}
        \centering
        \includegraphics[height=0.934\linewidth]{figure/river.pdf}
    \end{minipage}
    \hfill
    \begin{minipage}[t]{0.32\textwidth}
        \centering
        \includegraphics[height=0.934\linewidth]{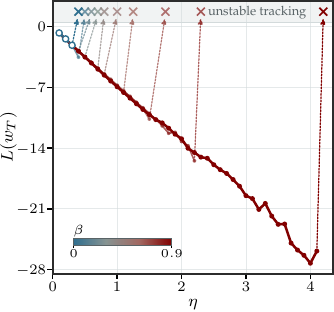}
    \end{minipage}
    \hfill
    \begin{minipage}[t]{0.32\textwidth}
        \centering
        \includegraphics[height=0.934\linewidth]{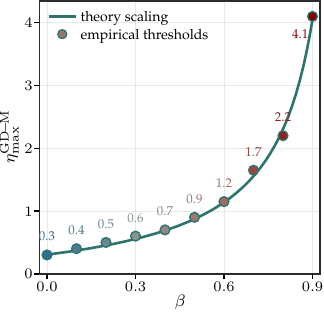}
    \end{minipage}
    \par\vspace{-0.4\baselineskip}
    \begin{minipage}[t]{0.32\textwidth}
        \centering
        {\small\textbf{(a)} Trajectories}
    \end{minipage}
    \hfill
    \begin{minipage}[t]{0.32\textwidth}
        \centering
        {\small\textbf{(b)} Final loss}
    \end{minipage}
    \hfill
    \begin{minipage}[t]{0.32\textwidth}
        \centering
        {\small\textbf{(c)} Stability threshold}
    \end{minipage}
    \caption{\small Numerical verification on
    $L(x_1,x_2)=2(x_2-\sin(x_1))^2-0.1x_1$, initialized at $w_0=(0,1)$.
    \textbf{(a)} During the first $160$ iterations at $\eta=0.4$, GD-M with
    $\beta=0.9$ damps transverse oscillations and advances along the river,
    whereas GD remains in persistent oscillations and makes no sustained
    progress. \textbf{(b)} After $T=1000$, stable runs attain nearly the same
    loss at fixed $\eta$, while larger $\beta$ permits larger stable learning
    rates and hence lower loss. Color encodes $\beta$; arrows point to
    categorical post-threshold markers, not loss values.
    \textbf{(c)} Gradient-colored circles and labels give empirical thresholds; the solid
    curve is $\eta_{\max}^{\mathrm{GD}}(1+\beta)/(1-\beta+0.05)$.}
    \label{fig: experimemnt 1}
\end{figure*}

\paragraph{Heavy-ball momentum: acceleration via stabilization.} 
As discussed in Section~\ref{sec: gd}, a key factor that determines the optimization progress along the river is the largest learning rate that the algorithm can tolerate.
For \eqref{eq: gd}, recall from \eqref{eq:max_lr_gd} that $\eta_{\max}^{\mathrm{GD}}\approx 2/\gmax$.
In contrast, for \eqref{eq: gd momentum}, in the regime of $\gflat \ll \gmax$ (very flat river) and $\kappa \ll 1$ (very slowly spinning river), it follows from \eqref{eq:max_lr_gdm} that the largest learning rate $\eta_{\max}^{\mathrm{GD}\text{-}{\mathrm{M}}}$ is larger than $\eta_{\max}^{\mathrm{GD}}$ by a multiplicative factor of $(1+\beta)/(1-\beta)$.
That is, 
\begin{align}
    \eta^{\mathrm{GD}\text{-}\mathrm{M}}_{\max} \approx \frac{1+\beta}{1-\beta} \cdot \eta^{\mathrm{GD}}_{\max}.\label{eq: predicted largest approximation}
\end{align}
A larger learning rate $\eta$ in-turn results in a faster speed along the river, because by Theorem~\ref{thm: main}, the pace of tracking the river is approximately proportional to $\eta$.

\paragraph{For flat \& slow-spinning river, momentum itself doesn't alter the speed on river much.}
Another interesting prediction by Theorem~\ref{thm: main} is that for a flat and slow spinning river ($\gflat \gmax^{-1}\ll 1$, $\kappa \ll 1$), the speed of tracking the river is not influenced much by the momentum parameter $\beta$ itself.
Indeed, the speed is approximately proportional to the learning rate and is similar to that of \eqref{eq: gd}.
Given that, the benefit of momentum for acceleration of tracking the river is mainly brought by its stabilization effects under an aggressive choice of larger learning rates $\eta$.
But still, if the landscape features a quickly spinning or relatively sharp river, the influences from the factors of $\gflat \eta$ and $\kappa$ are no longer negligible. 
In that case, the factor $\cO_\beta(\cdot)$ would become especially large when $\beta$ is very close to 1, for which large momentum could hurt optimization due to failure of tracking the river closely.
It is an interesting future work to investigate in that regime how momentum works and how the momentum parameter interplays with these geometric properties of the loss landscape in a quantitative manner.

\paragraph{Experiment: a 2-d non-convex loss.} We first verify our theory in a synthetic non-convex loss function $L(x_1,x_2) = 2\cdot(x_2-\sin (x_1))^2 - 0.1\cdot x_1$. 
We run \eqref{eq: gd momentum} with extensive choices of $(\beta,\eta)$ to validate theory. 
Results are in Figure~\ref{fig: experimemnt 1}.
We start from $w_0=(0,1)$, a point near the river, and run \eqref{eq: gd momentum} for $1000$ steps.
In Figure~\ref{fig: experimemnt 1}(a), we plot the first $160$ iterates of \eqref{eq: gd momentum} with $\beta=0.9$ and $\beta=0$ (corresponding to \eqref{eq: gd}), where both use $\eta=0.4$.
We can see \eqref{eq: gd momentum} with $\beta=0.9$ manages to stabilize the trajectory and tracks the river smoothly. 
In contrast, GD continues to oscillate across the valley and makes no sustained progress along the river.
This matches the picture of our theory.

Furthermore, we demonstrate quantitative results.
In Figure~\ref{fig: experimemnt 1}(b), we plot the final loss after $1000$ steps under different $(\beta,\eta)$ pairs.
Here $\beta\in\{0, 0.1,0.2, \ldots,0.8, 0.9\}$, while the learning rate is taken from $0.1$ to the threshold beyond which the iteration loses stable tracking of the river.
Color encodes $\beta$, and arrows connect the stable branches to categorical markers beyond the estimated thresholds; marker heights are not loss values.
We have three key observations from the  results: 
\begin{enumerate}[leftmargin=6mm,nolistsep]
    \item For larger $\beta$, the largest tolerable LR $\eta_{\max}^{\mathrm{GD}\text{-}{\mathrm{M}}}$ is larger.
    \item Given $\beta$, 
    the final loss decreases approximately linearly as the LR increases.
    \item Given a LR $\eta$, for those $\beta$ such that $\eta$ is tolerable, the choice of $\beta$ does not significantly change the final loss.
\end{enumerate} 
All of these results demonstrate our theoretical predictions given by main theorem.
Besides, we test $\beta=0.999$ and it does not converge for any $\eta$, which echos the condition on $\beta$ not being too close to $1$ when the river spins (see previous discussions). 
The detailed constraints on $\beta$ are in Theorem~\ref{thm: main formal}.

Finally, we extract the numerically obtained $\eta_{\max}^{\mathrm{GD}\text{-}{\mathrm{M}}}$ and plot its relation with $\beta$. 
See Figure~\ref{fig: experimemnt 1}(c).
As shown, the gradient-colored empirical thresholds nearly lie on the solid theoretical curve $\eta(\beta) := \eta^{\mathrm{GD}}_{\max}\cdot(1+\beta)/(1-\beta + 0.05)$, matching theoretical prediction \eqref{eq: predicted largest approximation}.
The $0.05$ factor is caused by the spinning and the curvature of the river.

\section{Proof Outline, Key Challenges, and Technical Novelty}\label{sec: proof outline}

In this section we sketch the proof of Theorem~\ref{thm: main}.
We first introduce several key concepts used in the proofs.
Then we outline the main proofs and point out the key challenges.
Then we walk through the key components of the proof, highlighting technical novelties.


\paragraph{Projections onto flat and sharp directions.} Recall that $v_d(\nabla^2 L(w))$ is the river direction of the loss Hessian at $w$.
We abbreviate it as $v_d(w)$.
On the river, $v_d(w)$ aligns with $\nabla L(w)$.
This motivates us to decompose the gradient and the momentum into the projections onto $v_d(w)$ and its orthogonal space. 
To this end, we define projection matrices $\bP_{\mathrm{f}}(w) := v_d(w) v_d(w)^\top$ and $\bP_{\mathrm{s}}(w):= \bI_d- \bP_{\mathrm{f}}(w)$.

\paragraph{Projection onto the river $\cM$.} To formalize the idea of how the iteration $\{w_k\}_{k\in\mathbb{N}}$ tracks the river, we introduce the projection onto the river of any point nearby. 
The definition of it is chosen as the limit of an ODE flow starting from the point to be projected. 
Formally, consider any $w\in\cU$ (the nice set near the river $\cM$, see Assumption~\ref{ass: regularity}). 
We define the ODE flow $\{\phi(w, t):w\in\cU, t\geq 0\}$ as $\phi(w, 0) = w$,
\begin{align}
    \frac{\mathrm{d}}{\mathrm{d}t}\phi(w, t) = -\bps(\phi(w, t))\nabla L(\phi(w, t)),\quad t\geq 0.
\end{align}
As is shown in Lemma~\ref{lem: existence of projection onto river and further properties}, the ODE flow has a limit $\lim_{t\rightarrow+\infty}\phi(w, t)\in\cM$, and we define it as $\Phi(w)$, i.e., the projection of $w$ onto the river. 
Other important properties are listed in Lemma~\ref{lem: existence of projection onto river and further properties}.

\paragraph{Continuous time linear interpolation of trajectory.}
To help characterize how the discrete-time sequence $\{w_k\}_{k\in\mathbb{N}}$ tracks the continuous-time river $\{x(t)\}_{t\geq 0}$, we consider a linear interpolation. 
For any $k\in\mathbb{N}$ and $\tau\in[0,1]$, we let
\begin{align}
    w_{k, \tau} := \tau\cdot w_{k+1} + (1-\tau)\cdot w_k = w_k - \tau\eta\cdot m_{k+1}.\label{eq: interpolation}
\end{align} 
We then obtain a continuous-time process $\{w_{\lfloor t\rfloor, t - \lfloor t\rfloor}\}_{t\geq 0}$.
In the proof, we iteratively show that $\{w_{\lfloor t\rfloor, t - \lfloor t\rfloor}\}_{t\geq 0}\subseteq\mathcal{U}$ to make sure that the projection to the river $\Phi(w_{\lfloor t\rfloor, t - \lfloor t\rfloor})$ is properly defined for any $t\geq 0$\footnote{Concretely, $\cB(w_{\lfloor t\rfloor, t - \lfloor t\rfloor},2g_{\max}/\gamma)\subseteq\mathcal{U}$, $\forall t\geq 0$ to ensure existence of $\Phi(w_{\lfloor t\rfloor, t - \lfloor t\rfloor})$, see Lemma~\ref{lem: existence of projection onto river and further properties}.}.
Thus, such an interpolated continuous-time sequence induces a trajectory on the river, i.e., $\{\Phi(w_{\lfloor t\rfloor, t - \lfloor t\rfloor})\}_{t\geq 0}$.
We initialize the reference flow \eqref{eq: manifold parametrization} at $x(0)=\Phi(w_{0,0})=w_0$ and define $S(t)$ to be the absolute reference-flow time associated with $\Phi(w_{\lfloor t\rfloor, t - \lfloor t\rfloor})$, i.e.,
\begin{align}
    x(S(t)) = \Phi\big(w_{\lfloor t\rfloor, t - \lfloor t\rfloor}\big).\label{eq: x S t phi}
\end{align}


\subsection{Proof Outline and Key Challenges}\label{subsec: outline}

\paragraph{Proof outline.}
    With the above setups, we consider the following route of proof:
\begin{enumerate}[nosep, leftmargin=6mm]
    \item We first prove that for \eqref{eq: gd momentum} 
    the gradient in the flat direction dominates the gradient and the momentum in the sharp directions, and the maximal learning rate is limited by the dynamics in sharp directions so that the iterates in the sharp directions do not explode.
    \item Then, we are able to prove that \eqref{eq: gd momentum} always stays close to the river. In more specific, we show that $w_k$ stays close to its projection onto the river $\Phi(w_k) $. 
    \item Finally, we prove that the absolute time map $S(t)$ grows nearly proportionally to $\eta$, which completes the proof.
\end{enumerate}

\paragraph{Key challenges.}
The essential challenge is to establish the conditions under which the iteration in such a non-convex landscape does not explode, and how the  momentum stabilizes the training process.
This requires a delicate analysis of the dynamics of both the gradient and the momentum in the river direction and the sharp directions. 
But due to the river spinning, the projections of the gradient and the momentum onto these directions interfere with each other in a complicated manner, necessitating a careful induction argument to clearly track all of their dynamics.



\subsection{Key Lemmas and Technical Novelties}


\paragraph{Analysis of projections onto flat and sharp directions.}
We first prove that the dynamics in the flat direction dominate in the sense that the gradient and momentum mostly live in the flat direction.
Also, the dynamics in the sharp direction determine the maximal tolerable learning rate such that the updates do not explode.
This is the following lemma.

\begin{lemma}[Projections onto flat and sharp directions (informal  and short version of Lemma~\ref{lem: flat dominance momentum})]\label{lem: flat dominance momentum main}
    Under the conditions and assumptions of Theorem~\ref{thm: main}, taking the step size $\eta$ and momentum parameter $\beta$ satisfying $\eta\leq \eta^{\mathrm{GD}\text{-}\mathrm{M}}_{\max}$ as defined in Theorem~\ref{thm: main}, for any iteration $k\in\NN$ and $\tau\in[0,1]$, it holds that
    \begin{align}
        \max\Big\{\|\bps(w_{k,\tau})\nabla L(w_{k,\tau})\|_2, \|\bps(w_{k,\tau}) m_{k+1}\|_2\Big\}\leq \cO_{\beta}\big(\kappa^{3/4} \cdot \|\bpf(w_{k,\tau})\nabla L(w_{k,\tau})\|_2\big),
    \end{align}
    where $\kappa$ characterizes the slow spinning speed of the river (item 3 of Assumption~\ref{ass: regularity}).
\end{lemma}

We refer the readers to Lemma~\ref{lem: flat dominance momentum} in Appendix~\ref{subsec: proof projection onto flat and sharp} for the full version of Lemma~\ref{lem: flat dominance momentum main} as well as the detailed proofs.
One of the main technical novelty here is a delicate induction argument to separately track the magnitude of the projections of the gradient and the momentum onto different directions in the river-valley landscape. 
Compared with the corresponding result for vanilla GD~(Lemma~\ref{lem: sharp dominance}), momentum enlarges the maximum learning rate by a multiplicative factor of $(1+\beta)/(1-\beta)$.
This is the key reason behind the larger learning rate blessed by momentum in river-valley.


\paragraph{GD with momentum tracks the river closely.}
Then we establish that the GD momentum trajectory can track the river closely by the following lemma.
It demonstrates that the interpolated continuous time trajectory $\{w_{\lfloor t\rfloor, t - \lfloor t\rfloor}\}_{t\geq 0}$ has a well-defined projection onto the river and the points in the trajectory are not far away from their projected counterparts.
See proofs in Appendix~\ref{subsec: proof gd momentum tracks the river closely}.

\begin{lemma}[The trajectory tracks the river closely (informal and short version of Lemma~\ref{lem: gd momentum tracks the river closely formal})]\label{lem: gd momentum tracks the river closely}
    Under the conditions and assumptions of Theorem~\ref{thm: main}, taking the learning rate $\eta$ and momentum parameter $\beta$ satisfying $\eta\leq \eta^{\mathrm{GD}\text{-}\mathrm{M}}_{\max}$ as defined in Theorem~\ref{thm: main}, then for any iteration $k\in\NN$ and $\tau\in[0,1]$ it holds that: (i) $\cB(w_{k,\tau},2g_{\max}/\gamma)\subset\cU$ and thus $\Phi(w_{k,\tau})$ exists; (ii) it holds that for any step $k\in\NN$ and $\tau\in[0,1]$,
    \begin{align}
    \|w_{k,\tau}  - \Phi(w_{k,\tau} )\|_2   \leq \cO_{\beta}\left(\frac{\kappa^{3/4} \cdot \|\bpf(w_{k,\tau} )\nabla L(w_{k,\tau} )\|_2}{\gamma + 2\gamma_{\mathrm{flat}}}\right).
\end{align}
\end{lemma}

\paragraph{Time index grows almost linearly w.r.t. learning rate.}
Finally, with the previous results, we show that the absolute time map $\{S(t)\}_{t\geq 0}$ defined in \eqref{eq: x S t phi} grows approximately linearly in $t$, with slope approximately equal to the learning rate $\eta$. See Appendix~\ref{subsec: proof lem time index grows almost linearly with learning rate} for the proof.

\begin{lemma}[Time index grows almost linearly with learning rate (informal and short version of Lemma~\ref{lem: time index grows almost linearly with learning rate formal})]\label{lem: time index grows almost linearly with learning rate}
    Under the conditions and assumptions of Theorem~\ref{thm: main}, taking the learning rate $\eta$ and momentum parameter $\beta$ satisfying $\eta\leq \eta^{\mathrm{GD}\text{-}\mathrm{M}}_{\max}$ as defined in Theorem~\ref{thm: main}, then for any iteration $k\geq \log(\beta\eta\gflat/(1-\beta))/\log\beta$ and $\tau\in(0,1)$, by letting $t=\tau + k$, the derivative of $S(t)$ exists and satisfies $|S'(t) - \eta|\leq \epsilon(\beta)\cdot\eta$.
    Equivalently, this bound holds almost everywhere on each interpolation interval, with the integer breakpoints understood through one-sided derivatives.
    Here  $\epsilon(\beta)$ is defined as $ \epsilon(\beta):=\cO_{\beta}(\kappa + \eta\gflat)$.
\end{lemma}

Integrating this derivative estimate over the post-transient interpolation intervals and absorbing the finite transient into the constant time shift $T_0$ gives the cumulative time-alignment bound in Theorem~\ref{thm: main formal}.

Combining Lemma~\ref{lem: gd momentum tracks the river closely} and Lemma~\ref{lem: time index grows almost linearly with learning rate} gives Theorem~\ref{thm: main}.

\section{Experiments on Language Model Training}\label{sec: llm experiment}

\paragraph{Experimental setup.}
We train GPT-Neo \citep{gao2020pile} (adjusted to about 20 million parameters) on the TinyStories dataset \citep{eldan2023tinystories}, containing about 0.5 billion tokens.
The training split contains about $2.2$ million examples, and the validation split contains about $22000$ examples.
We use batch size $32$ and context window size $256$.
We conduct two groups of experiments: one with SGD-Momentum, which is closest to our theory, and one with Adam, which tests whether the same mechanism appears in practical language model training.
For the Adam experiments, we train for $1$ epoch, fix $\beta_2=0.95$, and vary the learning rate $\eta$ and the first-order moment parameter $\beta_1$.
For the SGD-M experiments, we use the same model, data processing, and training protocol, but replace Adam by stochastic gradients with heavy-ball momentum and vary the momentum parameter $\beta$.
For both optimizers, we warm up the learning rate linearly for the first $100$ steps, and for the last $15\%$ of training steps, we decay the learning rate to $0.0001$ to reduce the oscillations between the valley and reveal the underlying progress along the river, similar to the WSD scheduler~\citep{hu2024minicpm,wen2025understanding}.
Additional experiment details and full grids are provided in Appendix~\ref{sec: experiment appendix}.

\paragraph{SGD-M: direct evidence for the learning-rate stabilization mechanism.}
We first test SGD-M, which removes Adam preconditioning and is closest to the optimizer analyzed in Theorem~\ref{thm: main}.
For each momentum parameter $\beta$, we grid search the largest learning rate under which the training loss curve does not diverge.
Figure~\ref{fig:sgd-m-sweep} shows that the largest tolerable learning rate $\eta_{\max}^{\mathrm{SGD}\text{-}\mathrm{M}}$ increases with $\beta$, matching the increasing trend predicted by \eqref{eq: predicted largest approximation}.
For visualization, we fit the adjusted scaling law $\eta_{\max}(q)=\eta_0(1+a+q)/(1+b-q)$ with $q=\beta$.
Here $\eta_0$ sets the overall scale, $a$ adds a constant numerator offset, and $b$ shifts the denominator away from its idealized singularity at $q=1$; together, these empirical corrections absorb non-asymptotic effects from river spinning, river curvature, and stochastic gradient noise.
They are descriptive fit parameters rather than additional constants predicted by Theorem~\ref{thm: main}.
Thus, we interpret this sweep as empirical support for the scaling trend rather than an exact realization of the idealized formula.

Table~\ref{tab:hp_loss} gives a complementary fixed-grid comparison.
With the large learning rate $\eta=1$, training without momentum diverges, while training with $\beta=0.9$ or $\beta=0.95$ converges to low training and validation losses.
In contrast, under the same stable learning rate $\eta=0.1$, changing $\beta$ from $0.9$ to $0.95$ only has a marginal effect on the final losses.
This supports the message of our theory that momentum accelerates primarily by stabilizing larger learning rates, rather than by substantially changing the river-direction speed at a fixed learning rate.

\paragraph{Adam: the same stabilization mechanism appears in practical training.}
We next test whether the same qualitative behavior appears for Adam, a standard optimizer for language models.
\par\noindent
We again first grid search the learning rate $\eta$ under different momentum $\beta_1$ for the largest learning rate such that the training loss curve does not diverge.
As shown in Figure~\ref{fig:adam-lr-sweep}, the largest tolerable learning rate under Adam increases as $\beta_1$ grows, which is consistent with the stabilization mechanism predicted by our theory.
When $\beta_1$ approaches $1$, we also observe a decrease in the largest tolerable learning rate, which is interpreted by the role of river curvature and spinning discussed in Section~\ref{subsec: predictions}.
The teal curve uses the same descriptive scaling fit with $q=\beta_1$ to summarize this trend; it is not a theoretical prediction for Adam.

\begin{figure}[t]
  \centering
  \captionsetup{font=footnotesize, skip=2pt}
  \begin{minipage}[t]{0.31\linewidth}
    \vspace{0pt}
    \centering
    \captionsetup{width=\linewidth}
    \begin{minipage}[c][1.55in][c]{\linewidth}
      \centering
      \includegraphics[width=\linewidth]{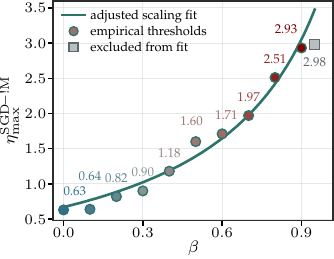}
    \end{minipage}
    \vspace{-1mm}
    \caption{SGD-M thresholds and adjusted scaling fit, with $(\eta_0,a,b)=(0.60,0.54,0.38)$.}
    \label{fig:sgd-m-sweep}
  \end{minipage}
  \hfill
  \begin{minipage}[t]{0.31\linewidth}
    \vspace{0pt}
    \centering
    \captionsetup{width=\linewidth}
    \begin{minipage}[c][1.55in][c]{\linewidth}
      \centering
      \includegraphics[width=\linewidth]{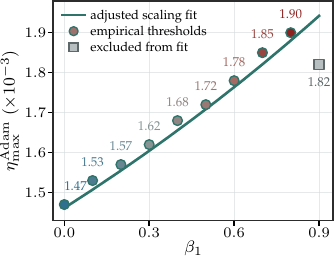}
    \end{minipage}
    \vspace{-1mm}
    \caption{Adam thresholds and adjusted scaling fit, with $(\eta_0,a,b)=(1.20\times 10^{-3},6.16,4.88)$.}
    \label{fig:adam-lr-sweep}
  \end{minipage}
  \hfill
  \begin{minipage}[t]{0.34\linewidth}
    \vspace{-1mm}
    \centering
    \captionsetup{width=\linewidth}
    \begin{minipage}[c][1.55in][c]{\linewidth}
      \hspace{6mm}
      \footnotesize
      \setlength{\tabcolsep}{4.8pt}
      \renewcommand{\arraystretch}{1.1}
      \raisebox{0.7mm}{%
      \begin{tabular}{@{}cccc@{}}
        \toprule[\lightrulewidth]
        $\eta$ & $\beta$ & Train & Val. \\
        \midrule
        $0.1$ & $0$    & 1.867 & 2.066 \\
        $0.1$ & $0.9$  & 1.843 & 2.065 \\
        $0.1$ & $0.95$ & 1.840 & 2.073 \\
        \midrule
        $1$ & $0$    & \textcolor{gray}{div.} & \textcolor{gray}{div.} \\
        $1$ & $0.9$  & 1.594 & 1.781 \\
        $1$ & $0.95$ & 1.602 & 1.783 \\
        \bottomrule[\lightrulewidth]
      \end{tabular}
      }%
    \end{minipage}
    \captionof{table}{Fixed-grid SGD-M final losses.}
    \label{tab:hp_loss}
  \end{minipage}
\end{figure}

We then examine whether the larger stability window translates into better optimization and validation performance.
Figure~\ref{fig:adam-loss-curves} plots the training and validation curves for
learning rates $0.0001$, $0.0006$, and $0.0008$, and momentum parameters
$0$, $0.9$, and $0.95$ over the first $15{,}000$ training steps.
For the small learning rate $\eta=0.0001$, different choices of $\beta_1$ lead to similar loss curves.
For the larger learning rates $\eta=0.0006$ and $\eta=0.0008$, Adam with $\beta_1\in\{0.9,0.95\}$ is more stable and reaches lower training and validation losses than Adam with $\beta_1=0$.
Moreover, with momentum, using a larger learning rate gives lower training and validation losses than using the small learning rate.
These observations again support the acceleration-by-stabilization explanation in practice: momentum improves training mainly by allowing a more aggressive stable learning rate.

\begin{figure}[H]
    \centering
    \captionsetup{font=footnotesize, skip=1pt}
    \includegraphics[width=0.94\linewidth]{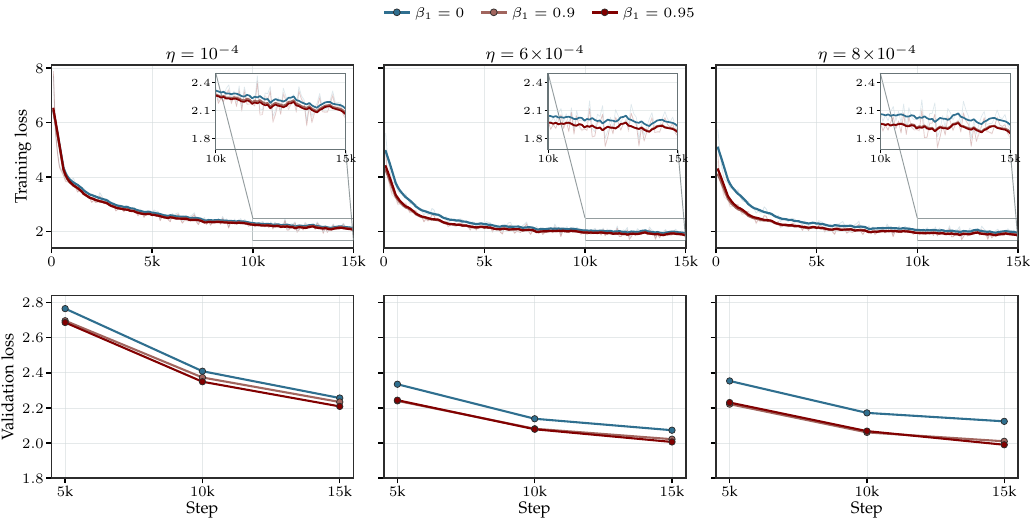}
    \caption{GPT-Neo training on TinyStories \citep{eldan2023tinystories} over the first $15{,}000$ steps; columns correspond to $\eta\in\{0.0001,0.0006,0.0008\}$. Top: per-batch training losses recorded every $100$ steps (faint) and their $900$-step moving averages (solid). Bottom: full-validation losses evaluated at steps $5{,}000$, $10{,}000$, and $15{,}000$.}
    \label{fig:adam-loss-curves}
\end{figure}

\section{Conclusions}

We prove that in river-valley landscape, momentum enables a significantly larger learning rate than vanilla GD by stabilizing oscillations across the sharp valley, thus accelerating progress along the flat, low-loss river manifold.
Our theory is validated by synthetic experiments and language model training.
Future directions include extending the analysis to adaptive optimization methods.

\section*{Acknowledgement}

T.W. acknowledges support from Google.

\FloatBarrier
\newpage
\bibliographystyle{ims}
\bibliography{reference}

\newpage
\appendix

{\color{sectionblue}\tableofcontents}

\newpage

\section{Further Related Works}\label{sec: further related works}

\paragraph{Analysis of heavy-ball momentum.}
Accelerated convergence of heavy-ball momentum (stochastic) gradient descent and its variants has been established in the locally near-quadratic setups, including the convex case \citep{ghadimi2015global,jain2018accelerating, liu2020improved, scieur2020universal, loizou2020momentum, sebbouh2021almost, paquette2021dynamics, li2022last, dang2024noise,bollapragada2024fast, panaccelerated, wei2024accelerated} and the non-convex case but with PL condition  \citep{ wang2021modular, pmlr-v162-wang22p, kassing2024polyak}, as well as the non-convex case with convergence to stationary points \citep{yan2018unified, gitman2019understanding,wangla20, xiong2020analytical, liu2020improved, sebbouh2021almost}. 
Beyond these convergence analysis of heavy-ball momentum, other lines of works try to understanding momentum from different perspectives.
\citet{wang2021modular} analyze the acceleration of heavy-ball momentum for wide ReLU neural networks and deep linear networks.
\citet{Plattner2022, wang2024the} argued the marginal value of the momentum in the small learning rate regime. 
\citet{ghoshimplicit,cattaneo2024implicit, papazov2024leveraging} study implicit bias of heavy-ball momentum to understand its influence on generalization.
\citet{lucasaggregated, chen2023bidirectional, pagliardini2024ademamix} consider modifications of the heavy-ball by mixture several exponential moving averages to take better advantage of past gradients of different patterns. 
\citet{ghoshimplicit, cattaneo2024implicit} study the implicit bias of heavy-ball momentum to understand its influence on generalization.

\paragraph{Analysis of loss landscape.}
Understanding the interactions between the loss landscape and the optimization dynamics as well as the generalization capabilities is key to demystify and further boost the success of deep learning \citep{freeman2017topology, garipov2018loss}.
The most relevant works to ours are \citet{xing2018walk, davis2024gradient, wen2025understanding}, which all consider a ``river-valley'' like landscape of the loss function to understand the behavior of optimizers.
\citet{xing2018walk} proposed the conceptual picture
where SGD locally bounces around the valley on top of a valley floor.
\citet{davis2024gradient} identified
the existence of a  manifold with vanishing
gradient within the sharp directions of the Hessian for smooth loss function that
exhibits fourth-order growth near minimizers.
All these works demonstrate the ubiquity of the river-valley like landscape. 
\citet{zhanggradient, zhang2020adaptive, paneigencurve, pan2022toward} study optimization dynamics in landscapes with large sharpness in certain directions and heavy-tailed noise, all of which hold a landscape similar to the ``river-valley'' one.
\citet{jiangfantastic, gunasekar2018implicit, blanc2020implicit, lihappens, ma2022beyond, lyu2022understanding, liu2023same,andriushchenko2023sgd, lubenign} examine the relationship between the loss landscape and the generalization properties and demonstrate the believe that flatter minima results in better generalization.
Finally, there is a line of works examining the algorithmic implicit bias towards minima with different kinds of properties for different variants of (stochastic) gradient descent, including \citet{soudry2018implicit, ji2019implicit,chizat2020implicit,blanc2020implicit, wudirection, lihappens, arora2022understanding, damianlabel, guand, lyu2022understanding, wen2023sharpness, lubenign, lyudichotomy}.
See also the references therein.

\section{Further Discussions}\label{sec: further discussions}

\paragraph{Interpretation of river-valley loss landscape.}
Here we provide a more concrete realistic situation that induces the river-valley structure. In the standard next-token prediction loss, where the uncertainty of the next token can vary significantly, the variations in the uncertainty can shape this type of loss landscape.
When the next token is predictable with high certainty, a larger learning rate can help the model learn faster.
In contrast, when the next token is inherently ambiguous, e.g., given a phrase like ``I am'', the model must learn a well-calibrated probability distribution, which may require smaller updates.
These fluctuations in uncertainty actually lead to the variations in the sharpness of the loss landscape, giving rise to the river-valley structure \citep{wen2025understanding}.

\paragraph{The stochastic setting.}
Beyond the deterministic setting in this paper, it is also interesting to consider the stochastic setting where gradients are corrupted by noise.
However, it has been shown by \citet{wang2024the} that when the learning rate is small for SGD (or when the batch size is small according to the linear scaling rule \citep{goyal2017accurate}), momentum has limited benefit. See Figure 1 of \citet{wang2024the}.
In contrast, when the batch size increases (more deterministic setting), the gap becomes significant.
Therefore, it is more meaningful to study the benefit of momentum for large batch size, which is close to the deterministic setting we study here.
We leave it for future work to systematically study SGD-momentum in river-valley.

\paragraph{Second-order ODE interpretation.}
We clarify why Theorem~\ref{thm: main} compares GD-momentum with a first-order reference flow on the river, even though momentum methods are often interpreted through a second-order ODE.
The key point is that the apparent mismatch is caused by both the parametrization of momentum and the very-flat-river regime.
Eliminating $m_k$ from \eqref{eq: gd momentum} gives the equivalent two-step recursion
\begin{align}
    w_{k+1}-2w_k+w_{k-1}
    =-(1-\beta)(w_k-w_{k-1})-\eta(1-\beta)\nabla L(w_k).
\end{align}
If $w_k$ is formally viewed as samples of a continuous curve $y(s)$ with time spacing $h$, i.e., $w_k\approx y(kh)$, then the corresponding second-order ODE is
\begin{align}
    \ddot y(s)+a\dot y(s)+ab\nabla L(y(s))=0,
    \qquad a:=\frac{1-\beta}{h},\quad b:=\frac{\eta}{h}.
\end{align}
This differs from the common unnormalized heavy-ball parametrization $w_{k+1}=w_k-\eta\nabla L(w_k)+\beta(w_k-w_{k-1})$, for which the gradient coefficient in the ODE is not multiplied by $1-\beta$.
Our parametrization is the normalized first-moment update used in \eqref{eq: gd momentum}; it is also the parametrization that matches the first-moment part of Adam and allows a controlled comparison with vanilla GD under the same learning-rate parameter $\eta$.

Now focus on the motion along the river, and denote the projected flat gradient by $g_{\mathrm{f}}(x)=\bP_{\cM}(x)\nabla L(x)$ for $x\in\cM$.
In the river-valley regime, the Hessian along the flat direction is small and the river spins slowly, quantified by $\gflat$ and $\kappa$.
Therefore, over a short segment of the river, $g_{\mathrm{f}}(x)$ varies slowly and can be approximated by a local vector $g_{\mathrm{loc}}$.
Projecting the above ODE onto this local flat direction yields the local model
\begin{align}
    \ddot y_{\mathrm{f}}(s)+a\dot y_{\mathrm{f}}(s)+ab g_{\mathrm{loc}}\approx 0.
\end{align}
Solving this linear equation gives
\begin{align}
    \dot y_{\mathrm{f}}(s)
    \approx \exp(-as)\bigl(\dot y_{\mathrm{f}}(0)+b g_{\mathrm{loc}}\bigr)-b g_{\mathrm{loc}}.
\end{align}
Thus, after the momentum burn-in period, the velocity satisfies $\dot y_{\mathrm{f}}(s)\approx -b g_{\mathrm{loc}}$.
Taking the discrete-time parametrization $h=1$, this gives the leading-order update direction $-\eta g_{\mathrm{loc}}$, which is the same first-order speed as vanilla GD with learning rate $\eta$.
The lower bound on $k$ in Theorem~\ref{thm: main} is precisely used to remove the transient momentum term; in the discrete proof, this transient appears through the factor $\beta^{k+1}$.

This explains why the time-tracking statement in Theorem~\ref{thm: main} takes the form $T(k)\approx \eta k$ and has no additional leading factor depending on $\beta$.
Momentum does not directly multiply the river speed under the normalized parametrization \eqref{eq: gd momentum}.
Instead, its main effect is through stability in the sharp directions: it permits a much larger tolerable learning rate, whose leading improvement is the factor $(1+\beta)/(1-\beta)$.
Once this larger $\eta$ is chosen, the first-order tracking relation $T(k)\approx \eta k$ translates the enlarged stable learning rate into faster progress along the river.

Finally, the local constant-gradient calculation above should not be interpreted as a global assumption.
Theorem~\ref{thm: main} compares the GD-momentum trajectory with the changing reference flow on the river,
\begin{align}
    \frac{\mathrm{d}}{\mathrm{d}T}x(T)
    =-\bP_{\cM}(x(T))\nabla L(x(T)).
\end{align}
The reference flow already captures the variation of $\bP_{\cM}(x(T))\nabla L(x(T))$ along the river.
The proof also controls this variation explicitly; for instance, Lemma~\ref{lem: time index grows almost linearly with learning rate formal} bounds the difference between the projected momentum direction and the tangent vector of the reference flow by terms of order $\kappa$ and $\eta\gflat$.
Along the reference flow, the loss satisfies
\begin{align}
    \frac{\mathrm{d}}{\mathrm{d}T}L(x(T))
    =-\|\bP_{\cM}(x(T))\nabla L(x(T))\|_2^2.
\end{align}
Hence, if the loss is lower bounded, the projected flat gradient cannot remain bounded away from zero forever.
The second-order ODE calculation is therefore only a local explanation of the post-burn-in behavior in a very flat segment of the river, while the theorem itself tracks the globally changing river flow under the stated assumptions.

\newpage

\section{Key Proofs of the Main Theorem}\label{sec: key proofs of main theorem}

\subsection{Formal Statement of Theorem~\ref{thm: main}}\label{subsec: formal statement}

\begin{theorem}[GD-momentum in river valley loss landscape (formal version)]\label{thm: main formal}
    Suppose that Assumptions~\ref{ass: existence} and \ref{ass: regularity} hold. 
    Define the following functions of momentum parameter $\beta\in(0,1)$, 
    \begin{align}
        \mathsf{A}(\beta)&:= \frac{16(1+\beta)}{1-\beta}\left(2+\frac{5(1+\beta)}{1-\beta}\right)\left(\frac{8}{(1-\beta)^3}+3\right)^2\left(2+\frac{4(1+\beta)}{1-\beta}+\frac{5(1+\beta)^2}{(1-\beta)^2}\right),\\
        \mathsf{B}(\beta)&:=2\cdot \left(\frac{16}{(1-\beta)^3}+6\right)^{\frac{1}{2}} \cdot \left(\frac{10}{1-\beta} + 1\right)\cdot\frac{2}{1-\beta},\quad \mathsf{C}(\beta):=1+5\cdot \frac{1+\beta}{1-\beta}, \\ 
        \mathsf{D}_1(\beta)&:=2\big(1+\mathsf{A}(\beta) + \mathsf{B}(\beta)\big),\quad \mathsf{D}_2(\beta):=8\big(1+6\mathsf{C}(\beta)\big)\cdot\frac{1+\beta}{1-\beta},
    \end{align}
    and we denote $\varsigma(\beta)\in\RR$ as (for simplicity we sometimes omit the dependency on $\beta$ when it makes no confusion) 
    \begin{align}
        \varsigma(\beta):= \mathsf{D}_1(\beta)\cdot\kappa^{3/16} + \mathsf{D}_2(\beta)\cdot\gflat\gmax^{-1}.\label{eq: varsigma}
    \end{align}
    Now we take the learning rate $\eta$ and momentum parameter $\beta$ satisfying 
    \begin{align}
        \beta \leq \min\big\{0.99,1-\varsigma(\beta)\big\},
        \qquad \eta\in \mathcal{I}(\beta).\label{eq: condition on beta and eta}
    \end{align}
    where $\varsigma(\beta)$ is defined in \eqref{eq: varsigma}, and the interval $\mathcal{I}(\beta)$ is defined as $\mathcal{I}(\beta):=\cI_1\cap  \cI_2\cap  \cI_3\cap  \cI_4$ with\footnote{The exact formula of the right endpoint of $\cI_4$ is given in \eqref{eq: cI 4 upper}. Here we only present the asymptotic expansion for better readability.}
    \begin{align}
         \cI_1&:= \left(\frac{\kappa^{1/32}}{\gmax}, \,\,\frac{2\cdot (1+\beta)-\kappa^{1/16}}{1-\beta}\cdot\frac{1}{\gmax}\right),\label{eq: intervals formal}\\ 
         \cI_2&:= \left(0,\,\,\frac{1}{16(1+6\mathsf{C}(\beta))}\cdot\frac{1}{\gflat}\right),\\
         \cI_3&:=\left(0,\,\,\frac{1-\beta}{8\beta\mathsf{C}(\beta)+2\beta(1+6\mathsf{C}(\beta))}\cdot\frac{1}{\gflat}\right),\\
        \cI_{4}&:=\left(
        \left(\mathsf{D}_1(\beta)\cdot\kappa^{5/32}
        + \mathsf{D}_2(\beta)\cdot \kappa^{15/32}\right)\cdot \frac{2}{\gmax},
        \right. \notag\\
        &\qquad\left.
        \bigg(\frac{1+\beta}{1-\beta}
        -\frac{5\beta^4 -2\beta^3 +6\beta^2 -2\beta + 1}{2(1-\beta)^4}\cdot\varsigma
        - \mathcal{O}(\varsigma^2)\bigg)\cdot \frac{2}{\gmax}
        \right).
    \end{align}
    Define the post-transient iteration
    \begin{align}
        k_\star:=\max\left\{0,\left\lceil\frac{\log(\beta\eta\gflat/(1-\beta))}{\log\beta}\right\rceil\right\}.
    \end{align}
    Then there exists a constant time shift $T_0$ such that the GD-momentum iteration \eqref{eq: gd momentum} with initialization $w_0\in\cM$ on the river and initial momentum $m_0=0$ satisfies that for any integer $k\geq k_\star$, there exists another $T(k)$ such that the following two things hold:
    \begin{enumerate}[nosep, leftmargin=6mm]
        \item GD-Momentum stays close to the river: $\|x(T_0 + T(k))-w_k\|_2\leq 12\mathsf{B}(\beta)\cdot\kappa^{3/4}\cdot g_{\max}/\gamma$;
        \item The speed on the river is approximately proportional to the learning rate: $|T(k) - \eta\cdot k| \leq \epsilon(\beta)\cdot \eta(k-k_\star)\leq \epsilon(\beta)\cdot\eta k$, where
        \begin{align}
        \epsilon(\beta):=9\kappa + \left(6\big(1+6\mathsf{C}(\beta)\big) + 4 + \frac{100\beta\mathsf{C}(\beta)}{1-\beta}\right)\cdot\eta\gflat + o(\kappa + \eta\gflat).
    \end{align}
    \end{enumerate}
\end{theorem}

\begin{remark}[Absolute and post-transient time coordinates]\label{rem: post-transient time shift}
    The map $S(t)$ in \eqref{eq: x S t phi} records the absolute reference-flow time associated with the interpolated GD-M trajectory.
    Theorem~\ref{thm: main formal} instead uses a post-transient clock whose ideal value is $\eta k$ at the first post-transient iteration $k_\star$.
    Specifically, define
    \begin{align}
        T_0:=S(k_\star)-\eta k_\star,
        \qquad
        T(k):=S(k)-T_0.
    \end{align}
    Then $T(k_\star)=\eta k_\star$ and, for every $k\geq k_\star$,
    \begin{align}
        T_0+T(k)=S(k),
        \qquad
        x(T_0+T(k))=\Phi(w_k).
    \end{align}
    Thus, $T_0$ absorbs the finite discrepancy accumulated before $k_\star$, while $T(k)=\eta k_\star+S(k)-S(k_\star)$ measures subsequent reference-flow progress from the calibrated baseline $\eta k_\star$.
\end{remark}

\begin{remark}[About the range of learning rate and momentum parameter]
Here we remark on the valid range of learning rate $\eta$ and momentum $\beta$ specified in \eqref{eq: condition on beta and eta} and \eqref{eq: intervals formal}. 
Firstly, the condition on momentum parameter $\beta$ in \eqref{eq: condition on beta and eta} is mild since when both $\kappa$ and $\gflat/\gmax$ are considered sufficiently small in our theory, the right hand side of $\beta\leq 1-\varsigma(\beta)$ is close to $1$. 
That being said, we need $\beta$ not too close to $1$ due to the existence of slow spinning and the small curvature of the river manifold. 

Now given a fixed $\beta$ satisfying the first inequality in \eqref{eq: condition on beta and eta}, we discuss the non-emptiness of the interval $\cI(\beta)$ defined in \eqref{eq: intervals formal}.
The second and third interval constraints $\cI_2$ and $\cI_3$ are flat-direction stability constraints.
Their right endpoints scale with $1/\gflat$, up to $\beta$-dependent constants, and therefore remain compatible with the sharp-direction constraints when the flat curvature $\gflat$ is sufficiently small relative to $\gmax$ and when $\beta$ is not chosen so close to $1$ that these constants dominate.
In this regime, the upper endpoint of $\cI(\beta)$ is the minimum of the four right endpoints, with $\cI_1$ and $\cI_4$ giving the sharp-direction stability threshold.
For the left endpoints, when $\kappa \ll 1$, both left endpoints of $\cI_1$ and $\cI_4$ are close to zero and are smaller than the corresponding right endpoints.
Therefore, under the above scale separation, the intersection $\cI(\beta)=\cI_1\cap \cI_2\cap \cI_3\cap \cI_4$ is non-empty; when the flat-direction constraints are inactive, its right endpoint is captured by $\cI_1$ and $\cI_4$ and approximates $2(1+\beta)/(1-\beta)\cdot 1/\gmax$.
\end{remark}

\begin{proof}[Proof of Theorem~\ref{thm: main formal}]
    The proof of Theorem~\ref{thm: main formal} is outlined in Section~\ref{sec: proof outline} and is illustrated in Figure~\ref{fig: theorem c1 proof architecture}. 
    The proof is decomposed into several key lemmas (Lemmas~\ref{lem: flat dominance momentum}, \ref{lem: gd momentum tracks the river closely formal}, and \ref{lem: time index grows almost linearly with learning rate formal}) whose proofs are provided in the subsequent sections.
    Set $T_0$ and $T(k)$ as in Remark~\ref{rem: post-transient time shift}.
    Lemma~\ref{lem: gd momentum tracks the river closely formal}, together with $\|\bpf(w_k)\nabla L(w_k)\|_2\leq g_{\max}$ and $x(T_0+T(k))=\Phi(w_k)$, proves the first conclusion.
    Moreover, Lemma~\ref{lem: time index grows almost linearly with learning rate formal} holds almost everywhere on $[k_\star,k]$, so
    \begin{align}
        |T(k)-\eta k|
        =|S(k)-S(k_\star)-\eta(k-k_\star)| \leq \int_{k_\star}^{k}|S'(t)-\eta|\,\mathrm{d}t
        \leq \epsilon(\beta)\eta(k-k_\star).
    \end{align}
    This proves the second conclusion.
\end{proof}

\paragraph{Proof architecture.}
Figure~\ref{fig: theorem c1 proof architecture} summarizes the dependency structure of the proof.
We refer to estimates evaluated at the discrete iterates $w_k$ as endpoint bounds: the endpoints of the $k$-th interpolation segment are $w_{k,0}=w_k$ and $w_{k,1}=w_{k+1}$.
The labels (I1)--(I4) refer to the four endpoint induction conditions \eqref{eq: induction 1}--\eqref{eq: induction 4}.
At each iteration, they are updated in the order (I1), (I3), (I2), and finally (I4); the resulting bounds then serve as the induction hypotheses at the next iteration.

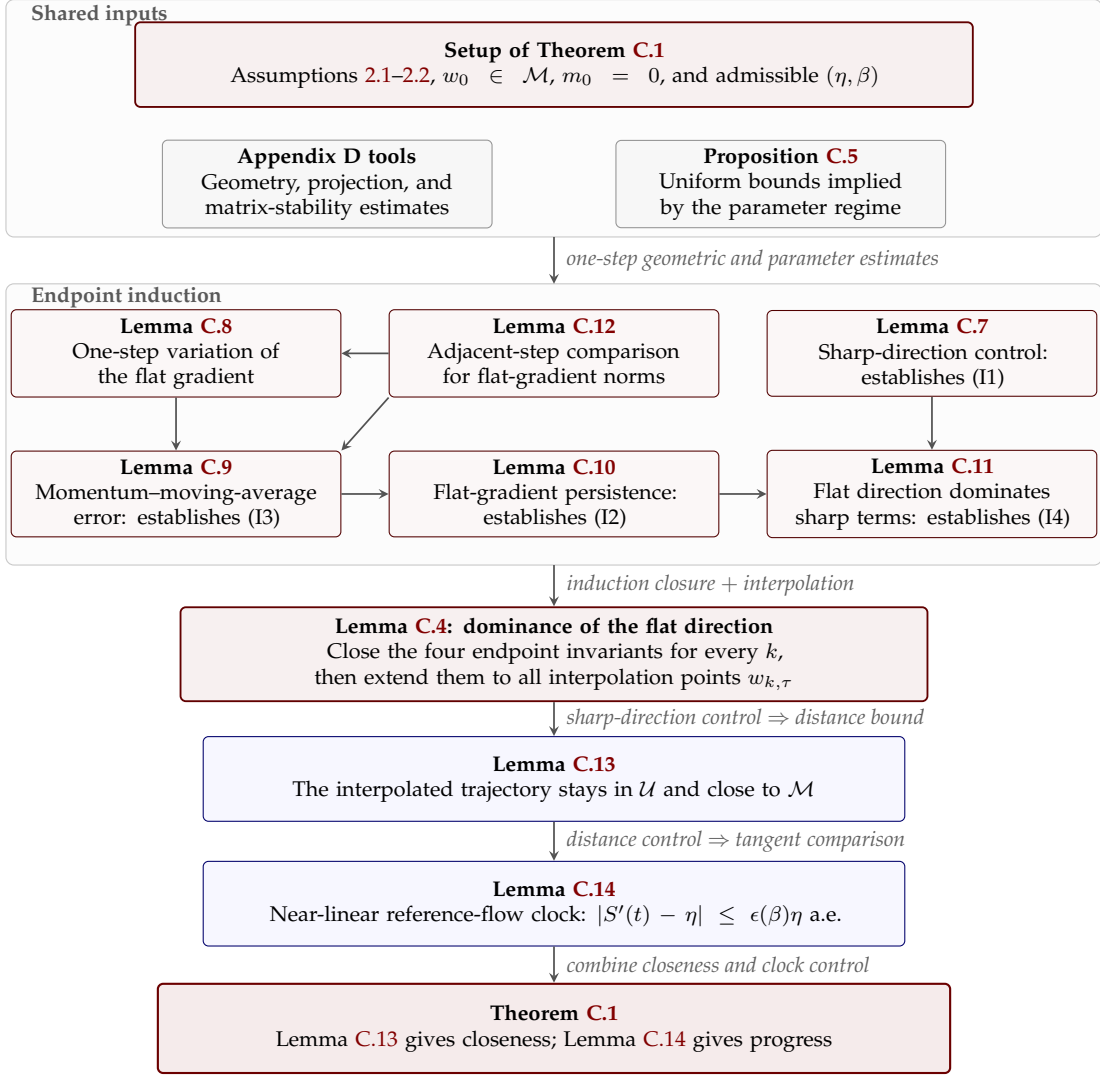
\begin{figure}[H]
    \centering
    \begin{tikzpicture}[
        >=stealth,
        every node/.style={font=\scriptsize},
        support/.style={
            draw=black!42,
            fill=black!3,
            rounded corners=2pt,
            align=center,
            text width=4.15cm,
            minimum height=1.02cm,
            inner sep=3pt
        },
        lemma/.style={
            draw=templatecolor!62!black,
            fill=templatecolor!3,
            rounded corners=2pt,
            align=center,
            text width=4.15cm,
            minimum height=1.10cm,
            inner sep=3pt
        },
        keylemma/.style={
            draw=templatecolor!75!black,
            fill=templatecolor!6,
            line width=0.7pt,
            rounded corners=2pt,
            align=center,
            text width=9.6cm,
            minimum height=1.12cm,
            inner sep=4pt
        },
        consequence/.style={
            draw=blue!42!black,
            fill=blue!3,
            rounded corners=2pt,
            align=center,
            text width=5.75cm,
            minimum height=1.12cm,
            inner sep=4pt
        },
        theoremresult/.style={
            draw=templatecolor!82!black,
            fill=templatecolor!8,
            line width=0.8pt,
            rounded corners=2pt,
            align=center,
            text width=10.2cm,
            minimum height=1.18cm,
            inner sep=4pt
        },
        flow/.style={->, draw=black!68, line width=0.65pt},
        flowlabel/.style={
            fill=white,
            inner sep=1.5pt,
            font=\scriptsize\itshape,
            text=black!62
        }
    ]
        \path[draw=black!22, fill=black!1, rounded corners=3pt]
            (-7.25,0.76) rectangle (7.25,-2.38);
        \node[anchor=west, font=\scriptsize\bfseries, text=black!65]
            at (-7.05,0.56) {Shared inputs};

        \node[keylemma, text width=10.8cm] (setup) at (0,-0.10)
        {\textbf{Setup of Theorem~\ref{thm: main formal}}\\
        Assumptions~\ref{ass: existence}--\ref{ass: regularity}, $w_0\in\cM$, $m_0=0$, and admissible $(\eta,\beta)$};

        \node[support] (geometry) at (-3.0,-1.68)
        {\textbf{Appendix D tools}\\
        Geometry, projection, and matrix-stability estimates};
        \node[support] (parameters) at (3.0,-1.68)
        {\textbf{Proposition~\ref{prop: useful properties}}\\
        Uniform bounds implied by the parameter regime};

        \path[draw=black!22, fill=black!1, rounded corners=3pt]
            (-7.25,-3.00) rectangle (7.25,-6.72);
        \node[anchor=west, font=\scriptsize\bfseries, text=black!65]
            at (-7.05,-3.18) {Endpoint induction};

        \node[lemma] (flatvariation) at (-5.0,-3.92)
        {\textbf{Lemma~\ref{lem: flat gradient one step change}}\\
        One-step variation of\\the flat gradient};
        \node[lemma] (auxiliary) at (0,-3.92)
        {\textbf{Lemma~\ref{lem: auxiliary inequality 1}}\\
        Adjacent-step comparison\\for flat-gradient norms};
        \node[lemma] (sharp) at (5.0,-3.92)
        {\textbf{Lemma~\ref{lem: sharp direction}}\\
        Sharp-direction control:\\establishes (I1)};

        \node[lemma] (inductionone) at (-5.0,-5.78)
        {\textbf{Lemma~\ref{lem: induction 1}}\\
        Momentum--moving-average\\error: establishes (I3)};
        \node[lemma] (inductiontwo) at (0,-5.78)
        {\textbf{Lemma~\ref{lem: induction 2}}\\
        Flat-gradient persistence:\\establishes (I2)};
        \node[lemma] (inductionthree) at (5.0,-5.78)
        {\textbf{Lemma~\ref{lem: induction 3}}\\
        Flat direction dominates\\sharp terms: establishes (I4)};

        \node[keylemma] (flatdominance) at (0,-7.90)
        {\textbf{Lemma~\ref{lem: flat dominance momentum}: dominance of the flat direction}\\
        Close the four endpoint invariants for every $k$, then extend them to all interpolation points $w_{k,\tau}$};

        \node[consequence, text width=9.0cm] (tracking) at (0,-9.55)
        {\textbf{Lemma~\ref{lem: gd momentum tracks the river closely formal}}\\
        The interpolated trajectory stays in $\cU$ and close to $\cM$};
        \node[consequence, text width=9.0cm] (time) at (0,-11.20)
        {\textbf{Lemma~\ref{lem: time index grows almost linearly with learning rate formal}}\\
        Near-linear reference-flow clock: $|S'(t)-\eta|\leq\epsilon(\beta)\eta$ a.e.};

        \node[theoremresult] (theorem) at (0,-12.85)
        {\textbf{Theorem~\ref{thm: main formal}}\\
        Lemma~\ref{lem: gd momentum tracks the river closely formal} gives closeness; Lemma~\ref{lem: time index grows almost linearly with learning rate formal} gives progress};

        \draw[flow] (auxiliary.west) -- (flatvariation.east);
        \draw[flow] (flatvariation.south) -- (inductionone.north);
        \draw[flow] (auxiliary.south west) -- (inductionone.north east);
        \draw[flow] (inductionone.east) -- (inductiontwo.west);
        \draw[flow] (inductiontwo.east) -- (inductionthree.west);
        \draw[flow] (sharp.south) -- (inductionthree.north);

        \draw[flow] (0,-2.38) -- node[flowlabel, anchor=west, xshift=3pt]
            {one-step geometric and parameter estimates} (0,-3.00);
        \draw[flow] (0,-6.72) -- node[flowlabel, anchor=west, xshift=3pt]
            {induction closure $+$ interpolation} (flatdominance.north);
        \draw[flow] (flatdominance.south) -- node[flowlabel, anchor=west, xshift=3pt]
            {sharp-direction control $\Rightarrow$ distance bound} (tracking.north);
        \draw[flow] (tracking.south) -- node[flowlabel, anchor=west, xshift=3pt]
            {distance control $\Rightarrow$ tangent comparison} (time.north);
        \draw[flow] (time.south) -- node[flowlabel, anchor=west, xshift=3pt]
            {combine closeness and clock control} (theorem.north);
    \end{tikzpicture}
    \caption{Proof-dependency graph for Theorem~\ref{thm: main formal}.
    Lemma~\ref{lem: sharp direction} supplies (I1); Lemmas~\ref{lem: flat gradient one step change} and \ref{lem: auxiliary inequality 1} support Lemma~\ref{lem: induction 1}, after which Lemmas~\ref{lem: induction 2} and \ref{lem: induction 3} establish (I2) and (I4) in sequence and close the endpoint induction.
    Lemma~\ref{lem: flat dominance momentum} extends these bounds to the interpolated trajectory.
    Lemma~\ref{lem: gd momentum tracks the river closely formal} then proves proximity to the river, and Lemma~\ref{lem: time index grows almost linearly with learning rate formal} converts the flat-direction control into near-linear reference-flow time.
    Arrows inside the central block show the one-step induction dependencies, while the vertical spine records how the resulting endpoint bounds are converted into the two conclusions of Theorem~\ref{thm: main formal}.}
    \label{fig: theorem c1 proof architecture}
\end{figure}

\newpage
\subsection{Proofs for Projections onto Flat and Sharp Directions}\label{subsec: proof projection onto flat and sharp}

\subsubsection{Main Result and Proof Outline}

\begin{lemma}[Dominance of the flat direction]\label{lem: flat dominance momentum}
    Suppose that the momentum iteration \eqref{eq: gd momentum} starts from an initial point $w_0\in\cM$ on the river and an initial momentum $m_0=0$. Under Assumptions~\ref{ass: existence} and \ref{ass: regularity}, taking the learning rate $\eta$ and momentum parameter $\beta$ satisfying the conditions in Theorem~\ref{thm: main formal},
    then for any iteration $k\in\NN$ and $\tau\in[0,1]$, the following conclusions hold simultaneously,
    \begin{align}
    &\|\bpf(w_{k,\tau})\nabla L(w_{k,\tau})\|_2 \geq \big(1-\underline{q}_k\big)\cdot \|\bpf(w_{k})\nabla L(w_{k})\|_2, \label{eq: flat dominance conclusion 1} \\
    &b_{k}\cdot \|\bpf(w_k)\nabla L(w_k)\|_2\geq \Bigg\|(1-\beta)\cdot\sum_{j=0}^{k}\beta^j\cdot \bpf(w_k)\nabla L(w_k) - \bpf(w_k)m_{k+1}\Bigg\|_2, \label{eq: flat dominance conclusion 2} \\
       & \max\Big\{\|\bps(w_{k,\tau})\nabla L(w_{k,\tau})\|_2, \|\bps(w_{k,\tau}) m_{k+1}\|_2\Big\}\leq c_{\mathrm{int}} \cdot \|\bpf(w_{k,\tau})\nabla L(w_{k,\tau})\|_2,\label{eq: flat dominance conclusion 3}
    \end{align}
    where $b_k$, $\underline{q}_k$, and $c_{\mathrm{int}}$ are defined as follows:
    \begin{align}
        c_{\mathrm{int}}&:= 6\mathsf{B}(\beta)\cdot \kappa^{3/4},\\
        b_k&:= 8 \beta\mathsf{C}(\beta)\cdot \eta\gflat\cdot\sum_{j=0}^k\beta^j\cdot\Big(1+2\eta\gflat\cdot \big(1+6\mathsf{C}(\beta)\big)\Big)^j , \\
        \underline{q}_k&:=2\eta\gflat \cdot\Big(1+2\mathsf{C}(\beta)\cdot\big(1+4\eta\gflat\cdot \big(1+6\mathsf{C}(\beta)\big)\big)\Big)\cdot (1-\beta)\cdot\sum_{j=0}^k\beta^j,
    \end{align}
    and the constants $\mathsf{B}(\beta)$ and $\mathsf{C}(\beta)$ are defined in Theorem~\ref{thm: main formal}.
\end{lemma}

\begin{proof}[Proof of Lemma~\ref{lem: flat dominance momentum}]
    The key idea is to first prove endpoint bounds by induction over $k$ for four quantities: the gradient and momentum in the sharp and flat directions. We then extend these endpoint bounds to every $\tau\in[0,1]$ by controlling the variation along each interpolation segment.
    
    \paragraph{Induction objective.} We are going to inductively prove the following augmented results: for each $k\in\NN$, the following four induction conditions hold:
    \vspace{-1mm}
    \begin{itemize}[leftmargin=3.5mm]
        \item Firstly, the sharp direction does not explode:
    \begin{align}
         &\left\|\boldsymbol{P}_k^{\frac{1}{2}}\boldsymbol{O}^\top\boldsymbol{V}_k^\top\left(\begin{matrix}
        \bps(w_{k+1})m_{k+1}\\
        \bps(w_{k+1})\nabla L(w_{k+1})
    \end{matrix}
    \right) \right\|_2 \\
    &\qquad \leq\big(1-\tfrac{1}{2}\overline{q}_k\big)\cdot  \big(1+\mathsf{A}(\beta)\cdot \kappa^{3/16}\big)\cdot\left\|\boldsymbol{P}_{k-1}^{\frac{1}{2}}\boldsymbol{O}^\top\boldsymbol{V}_{k-1}^\top\!\!\left(\begin{matrix}
        \bps(w_{k})m_{k}\\
        \bps(w_{k})\nabla L(w_{k}) 
    \end{matrix}
    \right)\right\|_2 \\
    &\qquad \qquad + 
    \mathsf{B}(\beta)\cdot\kappa^{15/16}\cdot \left\|\left(\begin{matrix}
        \beta\cdot m_k\\
        (1-\beta)\cdot \nabla L(w_{k})
    \end{matrix}
    \right)\right\|_2.\label{eq: induction 1}
    \end{align}
    \item Secondly, the flat direction gradient keeps its magnitude: 
    \begin{align}
        \!\!\!\! \|\bpf(w_{k+1})\nabla L(w_{k+1})\|_2 \geq (1-\underline{q}_k)\cdot \|\bpf(w_{k})\nabla L(w_{k})\|_2. \label{eq: induction 2}
    \end{align}
    \item Thirdly, the flat direction momentum resembles the moving average of flat direction gradients: 
    \begin{align}
        \!\!\!\! b_{k}\cdot \|\bpf(w_k)\nabla L(w_k)\|_2 \geq \Bigg\|(1-\beta)\cdot\sum_{j=0}^{k}\beta^j\cdot \bpf(w_k)\nabla L(w_k) - \bpf(w_k)m_{k+1}\Bigg\|_2. \label{eq: induction 3}
    \end{align}
    \item Lastly, the flat direction gradient dominates the sharp direction gradient and momentum: 
    \begin{align}
    \begin{aligned}
        &\left\|\boldsymbol{P}_k^{\frac{1}{2}}\boldsymbol{O}^\top\boldsymbol{V}_k^\top\left(\begin{matrix}
        \bps(w_{k+1})m_{k+1}\\
        \bps(w_{k+1})\nabla L(w_{k+1})
    \end{matrix}
    \right) \right\|_2 \leq  c \cdot  \|\bpf(w_{k+1})\nabla L(w_{k+1})\|_2,\\
     &\max\Big\{\|\bps(w_{k+1})\nabla L(w_{k+1})\|_2, \|\bps(w_{k+1}) m_{k+1}\|_2\Big\} \leq c \cdot \|\bpf(w_{k+1})\nabla L(w_{k+1})\|_2.\label{eq: induction 4}
     \end{aligned}
    \end{align}
    \end{itemize}
    Additionally, the induction condition \eqref{eq: induction 4} holds for a dummy step $k=-1$.
    Here $\mathsf{A}(\beta)$, $\mathsf{B}(\beta)$ are defined in Theorem~\ref{thm: main formal}, $\{\boldsymbol{P}_k,\boldsymbol{O},\boldsymbol{V}_k\}_{k\geq 0}$ are matrices which we introduce later in the proof, $\{b_k\}_{k\geq 0}$, $\{\overline{q}_k\}_{k\geq 0}$, $\{\underline{q}_k\}_{k\geq 0}$ are positive sequences, $c>0$ is a constant, defined as following:
    \begin{align}
        c&:= 3\mathsf{B}(\beta)\cdot \kappa^{3/4},\\
        b_k&:= 8 \beta\mathsf{C}(\beta)\cdot \eta\gflat\cdot\sum_{j=0}^k\beta^j\cdot\Big(1+2\eta\gflat\cdot \big(1+6\mathsf{C}(\beta)\big)\Big)^j , \\
        \underline{q}_k&:=2\eta\gflat \cdot\Big(1+2\mathsf{C}(\beta)\cdot\big(1+4\eta\gflat\cdot \big(1+6\mathsf{C}(\beta)\big)\big)\Big)\cdot (1-\beta)\cdot\sum_{j=0}^k\beta^j,\\
        \overline{q}_k&:=\left(\max_{i\in[d-1]}\frac{(2\beta^3-2\beta^2+3\beta-1)\cdot (\eta\lambda_{k,i})^2 + 2(1-\beta^2)\cdot\eta\lambda_{k,i} + 2(1-\beta)^2(1+\beta)}{ \eta \lambda_{k,i}(1-\beta)^2\cdot(2(1+\beta)-(1-\beta) \eta \lambda_{k,i})}\right)^{-1}.
    \end{align}

    \paragraph{Correctness for step $k=0$.}
    Firstly we show that all of \eqref{eq: induction 1}, \eqref{eq: induction 2}, \eqref{eq: induction 3}, and \eqref{eq: induction 4} are correct for step $k=0$, and additionally \eqref{eq: induction 4} hold for a dummy step $k=-1$. 
    \vspace{-1mm}
    \begin{itemize}[leftmargin=3.5mm]
        \item The correctness of \eqref{eq: induction 1} at step $k=0$ is given by Lemma~\ref{lem: sharp direction} for step $k=0$. 
        \item The correctness of \eqref{eq: induction 3} at step $k=0$ is trivial since the right hand side of \eqref{eq: induction 3} is $0$. 
        \item The correctness of \eqref{eq: induction 4} at step $k=-1$ is also trivial since both sides are $0$.
        \item The correctness of \eqref{eq: induction 2} at step $k=0$ conditioning on the correctness of \eqref{eq: induction 3} at step $k=0$ and \eqref{eq: induction 4} at step $k=-1$ is given by Lemma~\ref{lem: induction 2}. 
        \item Finally the correctness of \eqref{eq: induction 4} at step $k=0$ conditioning on the correctness of \eqref{eq: induction 1},  \eqref{eq: induction 2}, and \eqref{eq: induction 3} at step $k=0$ and \eqref{eq: induction 4} at step $k=-1$ is given by Lemma~\ref{lem: induction 3}.
    \end{itemize}

    \paragraph{Main induction argument.}
    Now suppose that all of \eqref{eq: induction 1}, \eqref{eq: induction 2}, \eqref{eq: induction 3}, and \eqref{eq: induction 4} are correct up to some step $k-1$ with $k\geq 1$.
    To prove the correctness for step $k$, we adopt the following strategy:
    \begin{itemize}[leftmargin=3.5mm]
        \item Condition \eqref{eq: induction 1} can be directly proved correct for all $k$ (proved by Lemma~\ref{lem: sharp direction}). Then:
        \item First prove that Condition \eqref{eq: induction 3} is correct (with Lemma~\ref{lem: induction 1}). 
        \item Then prove that Condition \eqref{eq: induction 2} is correct (with Lemma~\ref{lem: induction 2}). 
        \item Finally prove that Condition \eqref{eq: induction 4} is correct (with Lemma~\ref{lem: induction 3}). 
    \end{itemize}
    In the sequel, we show how to obtain \eqref{eq: induction 2}, \eqref{eq: induction 3}, and \eqref{eq: induction 4} for step $k$ following the above strategy.
    Before the proof, we first collect some useful properties needed here (and later induction proofs) in the following.

    \begin{proposition}[Useful properties]\label{prop: useful properties}
    Under the setup of Theorem~\ref{thm: main formal}, we have the following results:
    \begin{enumerate}[nosep, leftmargin=6mm]
        \item Under the condition that $\eta\in \cI_1$, we have that 
    \begin{align}
        1+2\eta\gmax+\frac{\eta\gamma}{2} \leq 1+5\cdot\frac{1+\beta}{1-\beta}:=\mathsf{C}(\beta).\label{eq: c 3}
    \end{align}
    \item If we denote $q(\beta)$ as following
    \begin{align}
        q(\beta):=2\eta\gflat\cdot \big(1+6\mathsf{C}(\beta)\big),
    \end{align}
    then by $\eta\in \cI_2$, we have that 
        \begin{align}
              \eta \leq  \frac{1}{16 (1+6\mathsf{C}(\beta))}\cdot\frac{1}{\gflat}\quad\Rightarrow \quad q(\beta)\leq \frac{1}{8}.\label{eq: eta condition 2}
        \end{align}
    \item By $\eta\in \cI_3$, we have that for any $k\in\NN$,
        \begin{align}
            \eta \leq  \frac{1-\beta}{8\beta\mathsf{C}(\beta) + 2\beta\cdot(1+6\mathsf{C}(\beta)) }\cdot\frac{1}{\gflat}\quad \Rightarrow\quad b_k\leq 1.\label{eq: eta condition 3}
        \end{align}
    \item By $\eta\in \cI_3$, we also have that for any $k\in\NN$,
        \begin{align}
            b_k\leq (1-\beta)\cdot \sum_{j=0}^k\beta^j. \label{eq: b k weighted sum bound}
        \end{align}
    \item By $\eta\in \cI_3$, we also have that for any $k\in\NN$,
        \begin{align}
            b_k\leq \frac{24\beta\mathsf{C}(\beta)}{1-\beta}\cdot\eta\gflat.\label{eq: b k refined bound}
        \end{align}
    \item Under the condition on $\beta$ in Theorem~\ref{thm: main formal}, we have
        \begin{align}
            c\leq \frac{3}{8}(1-\beta)<\min\left\{1-\beta,\frac{1}{2}\right\}. \label{eq: c small bound}
        \end{align}
    \item By $\eta\in\cI_1\cap\cI_2\cap\cI_3$,  we have that for any $k\in\NN$,
        \begin{align}
            \underline{q}_k \leq \frac{q(\beta)}{2}.\label{eq: q k bound}
        \end{align}
    \end{enumerate}
    \end{proposition}

    \begin{proof}[Proof of Proposition~\ref{prop: useful properties}]
        We prove the seven items one by one in the following.

        \noindent\emph{Item 1.}
        This follows directly from $\eta\in\cI_1$ and the bound $\gamma\leq\gmax/200$ in item~2 of Assumption~\ref{ass: regularity}.

        \noindent\emph{Item 2.}
        This follows directly from the definition of $q(\beta)$ and $\eta\in\cI_2$.

        \noindent\emph{Item 3.}
        By \eqref{eq: eta condition 3}, we have
        \begin{align}
            8\beta\mathsf{C}(\beta)\eta\gflat+\beta q(\beta)\leq 1-\beta.
        \end{align}
        Hence, for any $k\in\NN$,
        \begin{align}
            b_k\leq \frac{8\beta\mathsf{C}(\beta)\eta\gflat}{1-\beta-\beta q(\beta)}\leq 1.
        \end{align}

        \noindent\emph{Item 4.}
        Again by \eqref{eq: eta condition 3}, for any $k\in\NN$,
        \begin{align}
            b_k
            \leq 8\beta\mathsf{C}(\beta)\eta\gflat
            \cdot
            \frac{1-\beta}{1-\beta\cdot(1+q(\beta))}\cdot 
            \sum_{j=0}^k\beta^j \leq (1-\beta)\cdot \sum_{j=0}^k\beta^j. \label{eq: proof b k weighted sum bound}
        \end{align}

        \noindent\emph{Item 5.}
        Again by \eqref{eq: eta condition 3} and $\mathsf{C}(\beta)\geq 1$,
        \begin{align}
            \beta q(\beta)\leq \frac{2\beta(1+6\mathsf{C}(\beta))}{8\beta\mathsf{C}(\beta)+2\beta(1+6\mathsf{C}(\beta))}\cdot(1-\beta)=\frac{1+6\mathsf{C}(\beta)}{1+10\mathsf{C}(\beta)}\cdot(1-\beta)
            \leq \frac{2}{3}\cdot(1-\beta).
        \end{align}
        Therefore,
        \begin{align}
            b_k
            \leq \frac{8\beta\mathsf{C}(\beta)\eta\gflat}{1-\beta-\beta q(\beta)}
            \leq \frac{24\beta\mathsf{C}(\beta)}{1-\beta}\cdot\eta\gflat.
        \end{align}

        \noindent\emph{Item 6.}
        The definitions imply $\mathsf{B}(\beta)\leq\mathsf{D}_1(\beta)$ and $\mathsf{D}_1(\beta)\geq2$. Together with $\beta\leq 1-\varsigma(\beta)$ in Theorem~\ref{thm: main formal}, this gives
        \begin{align}
            \mathsf{B}(\beta)\kappa^{3/16}
            &\leq \mathsf{D}_1(\beta)\kappa^{3/16}
            \leq \varsigma(\beta)
            \leq 1-\beta,\\
            \kappa^{9/16}
            &=\big(\kappa^{3/16}\big)^3
            \leq \mathsf{D}_1(\beta)^{-3}
            \leq \frac{1}{8}.
        \end{align}
        Therefore,
        \begin{align}
            c=3\mathsf{B}(\beta)\kappa^{3/4}
            =3\big(\mathsf{B}(\beta)\kappa^{3/16}\big)\kappa^{9/16}
            \leq \frac{3}{8}(1-\beta)
            <\min\left\{1-\beta,\frac{1}{2}\right\}.
        \end{align}

        \noindent\emph{Item 7.}
        Equation~\eqref{eq: eta condition 2} implies $q(\beta)\leq 1/8$.
        Since $\mathsf{C}(\beta)\geq 1$, for any $k\in\NN$,
        \begin{align}
            \underline{q}_k
            \leq 2\eta\gflat\cdot \Big(1+2\mathsf{C}(\beta)\cdot\big(1+2q(\beta)\big)\Big)\leq \eta\gflat\cdot\big(1+6\mathsf{C}(\beta)\big)
            =\frac{q(\beta)}{2}.
        \end{align}
        This completes the proof of Proposition~\ref{prop: useful properties}.
    \end{proof}

    \noindent
    Now let's get back to the inductive proof of Lemma~\ref{lem: flat dominance momentum}.

    ~\\
    \noindent
    \emph{Proof of \eqref{eq: induction 3} for step $k$.}
    By Lemma~\ref{lem: induction 1}, given the correctness of \eqref{eq: induction 2}, \eqref{eq: induction 3}, and \eqref{eq: induction 4} up to step $k-1$, we have that 
    \begin{align}
        &\left\|(1-\beta)\cdot\sum_{j=0}^{k}\beta^j\cdot \bpf(w_k)\nabla L(w_k) - \bpf(w_k)m_{k+1} \right\|_2\leq \widetilde{b}_k\cdot\|\bpf(w_k)\nabla L(w_k)\|_2,\quad\text{where} \\ 
        &\widetilde{b}_k:=\big(1+2\underline{q}_{k-1}\big)\cdot \Bigg(\beta b_{k-1}+ \beta(1-\beta) \sum_{j=0}^{k-1}\beta^j \underline{q}_{k-1} + \kappa\eta\gamma\beta\Bigg(1+(1-\beta)\sum_{j=0}^{k-1}\beta^j+b_{k-1}\Bigg)\Bigg).
    \end{align}
    Then it suffices to prove that the above $\widetilde{b}_k \leq b_k$.
    Using \eqref{eq: eta condition 2}, \eqref{eq: eta condition 3}, and \eqref{eq: q k bound}, we have the following, 
    \allowdisplaybreaks
    \begin{align}
        \widetilde{b}_k&=\big(1+2\underline{q}_{k-1}\big)\cdot \Bigg(\beta b_{k-1}+ \beta(1-\beta) \sum_{j=0}^{k-1}\beta^j \underline{q}_{k-1} + \kappa\eta\gamma\beta\Bigg(1+(1-\beta)\sum_{j=0}^{k-1}\beta^j+b_{k-1}\Bigg)\Bigg)\\
        &\leq \big(1+q(\beta)\big)\cdot \Bigg(\beta b_{k-1}+ \frac{\beta(1-\beta)}{2} \sum_{j=0}^{k-1}\beta^j q(\beta) + \kappa\eta\gamma\beta\Bigg(1+(1-\beta)\sum_{j=0}^{k-1}\beta^j+b_{k-1}\Bigg)\Bigg)\\
        & \leq 8\beta\mathsf{C}(\beta)\cdot \eta\gflat\cdot\sum_{j=1}^k\beta^j\cdot\big(1+q(\beta)\big)^j  + \beta\cdot \eta\gflat\cdot\big(4+6\mathsf{C}(\beta)\big)\cdot \big(1+q(\beta)\big)\\
        &\leq  8\beta\mathsf{C}(\beta)\cdot \eta\gflat\cdot\sum_{j=1}^k\beta^j\cdot\big(1+q(\beta)\big)^j  +  8\beta\mathsf{C}(\beta)\cdot\eta \gflat \\
        &= 8\beta\mathsf{C}(\beta)\cdot\eta\gflat\cdot\sum_{j=0}^k\beta^j\cdot\big(1+q(\beta)\big)^j\\
        &=b_k,
    \end{align}
    where the first inequality uses \eqref{eq: q k bound}, the second inequality combines (i) the definition of $b_{k-1}$ and (ii) \eqref{eq: eta condition 3} to bound $b_{k-1}\leq 1$, and (iii) item 2 of Assumption~\ref{ass: regularity} to obtain $\kappa\gamma\leq \gflat$.
    The third inequality uses \eqref{eq: eta condition 2} to bound $q(\beta)\leq 1/8$ and uses $\mathsf{C}(\beta)\geq 6$.
    This proves that $\widetilde{b}_k\leq b_k$ and thus the induction condition \eqref{eq: induction 3} hold for step $k$.

    ~\\
    \noindent
    \emph{Proof of \eqref{eq: induction 2} for step $k$.}
    By Lemma~\ref{lem: induction 2}, given the correctness of \eqref{eq: induction 2}, \eqref{eq: induction 3}, and \eqref{eq: induction 4} up to step $k-1$ plus \eqref{eq: induction 3} for step $k$, we have the following,
    \begin{align}
        &\|\bpf(w_{k+1})\nabla L(w_{k+1})\|_2 \geq \big(1-\widetilde{\underline{q}}_k\big)\cdot \|\bpf(w_{k})\nabla L (w_{k})\|_2,\quad\text{where}\\
        &\widetilde{\underline{q}}_k := \eta\gflat\cdot \Bigg((1-\beta)\sum_{j=0}^k\beta^j +b_k\Bigg)\\
        &\quad\quad+ 2\eta\gamma\kappa \big(1+2\eta\gmax+\eta\gamma/2\big)\cdot\Bigg((1-\beta)+\beta \Bigg((1-\beta)\sum_{j=0}^{k-1}\beta^j+b_{k-1}\Bigg) \big(1+2\underline{q}_{k-1}\big)\Bigg).
    \end{align}
    Then it suffices to prove that $\widetilde{\underline{q}}_k\leq \underline{q}_k$.
    To this end, consider the following, 
    \begin{align}
        \widetilde{\underline{q}}_k &\leq \eta\gflat\cdot \Bigg((1-\beta)\cdot\sum_{j=0}^k\beta^j +b_k \\
        &\qquad + 2\mathsf{C}(\beta) \cdot \Bigg((1-\beta)+\beta\cdot \Bigg((1-\beta)\cdot\sum_{j=0}^{k-1}\beta^j+b_{k-1}\Bigg)\cdot \big(1+2q(\beta)\big)\Bigg)\Bigg), \label{eq: tilde q first bound}
    \end{align}
    where we apply (i) item 2 of Assumption~\ref{ass: regularity} to obtain that $\kappa \gamma\leq \gflat$, (ii) inequality \eqref{eq: c 3} to bound $1+\eta\gmax+\eta\gamma/2$, and (iii) inequality \eqref{eq: q k bound} to bound $\underline{q}_{k-1}\leq q(\beta)$.
    To proceed, notice that by the definition of $b_k$,
    \begin{align}
        (1-\beta)\cdot\sum_{j=0}^k\beta^j +b_k &=(1-\beta)\cdot \sum_{j=0}^k\beta^j \\
        &
        \qquad + 8\beta\mathsf{C}(\beta)\cdot \eta\gflat\cdot\sum_{j=0}^k\beta^j\cdot\Big(1+2\eta\gflat\cdot \big(1+6\mathsf{C}(\beta)\big)\Big)^j.\label{eq: beta sum b k}
    \end{align}
    The second term on the right hand side of the above equality can be bounded by
    \begin{align}
        &8\beta\mathsf{C}(\beta)\cdot \eta\gflat\cdot\sum_{j=0}^k\beta^j\cdot\Big(1+2\eta\gflat\cdot \big(1+6\mathsf{C}(\beta)\big)\Big)^j\\
        &\qquad \leq 8\beta\mathsf{C}(\beta)\cdot \eta\gflat\cdot\max_{k'\in\NN}\left\{\frac{\sum_{j=0}^{k'}\beta^j\cdot(1+2\eta\gflat\cdot (1+6\mathsf{C}(\beta)))^j}{\sum_{j=0}^{k'}\beta^j}\right\} \cdot \sum_{j=0}^k\beta^j\\
        &\qquad = 8\beta\mathsf{C}(\beta)\cdot\eta\gflat\cdot \frac{1-\beta}{1-\beta\cdot(1+2\eta\gflat\cdot(1+6\mathsf{C}(\beta)))}\cdot \sum_{j=0}^k\beta^j\\
        &\qquad \leq (1-\beta)\cdot\sum_{j=0}^k\beta^j,\label{eq: bound b k}
    \end{align}
    where the last inequality uses $\eta\in\cI_3$.
    Consequently, by \eqref{eq: beta sum b k} and \eqref{eq: bound b k}, we have that 
    \begin{align}
        (1-\beta)\cdot\sum_{j=0}^k\beta^j +b_k  \leq 2(1-\beta)\cdot\sum_{j=0}^k\beta^j.\label{eq: beta sum plus b k bound}
    \end{align}
    Therefore, combining \eqref{eq: tilde q first bound} and \eqref{eq: beta sum plus b k bound}, we have that 
    \allowdisplaybreaks
    \begin{align}
        \widetilde{\underline{q}}_k&\leq \eta\gflat\cdot \Bigg(2(1-\beta)\cdot\sum_{j=0}^k\beta^j + 2\mathsf{C}(\beta)\Bigg((1-\beta)+\beta\cdot2(1-\beta)\cdot\sum_{j=0}^{k-1}\beta^j\Bigg)\cdot\big(1+2q(\beta)\big)\Bigg) \\
        &\leq \eta\gflat\cdot \Bigg(2(1-\beta)\cdot\sum_{j=0}^k\beta^j + 4\mathsf{C}(\beta)\cdot(1-\beta)\cdot\sum_{j=0}^{k}\beta^j\cdot\big(1+2q(\beta)\big)\Bigg) \\
        &=2\eta\gflat\cdot\Big(1+2\mathsf{C}(\beta)\cdot\big(1+2q(\beta)\big)\Big)\cdot(1-\beta)\cdot \sum_{j=0}^{k}\beta^j\\
        &=\underline{q}_k
    \end{align}
    This proves $\widetilde{\underline{q}}_k\leq \underline{q}_k$, finishing the proof of induction condition \eqref{eq: induction 2} for step $k$. The same comparison holds at $k=0$ under the conventions $b_{-1}=\underline q_{-1}=0$ and $\sum_{j=0}^{-1}\beta^j=0$.
    
    ~\\
    \noindent
    \emph{Proof of \eqref{eq: induction 4} for step $k$.} This follows by Lemma~\ref{lem: induction 3} without extra proofs. Now we have finished the endpoint induction. 
    
    ~\\
    \noindent
    \emph{Interpolation.} It remains to extend the endpoint bounds to $w_{k,\tau}$ for every $\tau\in[0,1]$. Denote
    \begin{align}
        g_j&:=\nabla L(w_j),
        \quad G_j:=\bpf(w_j)g_j,\quad j\in\mathbb{N}, \notag\\
        G_{k,\tau}&:=\bpf(w_{k,\tau})\nabla L(w_{k,\tau}),
        \quad c_{\mathrm{end}}:=3\mathsf{B}(\beta)\kappa^{3/4}.
    \end{align}
    First, Lemma~\ref{lem: induction 2} and the bound $\widetilde{\underline q}_k\leq\underline q_k$ established above give
    \begin{align}
        \|G_{k,\tau}-G_k\|_2
        \leq \widetilde{\underline q}_k\|G_k\|_2
        \leq \underline q_k\|G_k\|_2. \label{eq: interpolated flat gradient variation}
    \end{align}
    This proves \eqref{eq: flat dominance conclusion 1}. Moreover, Proposition~\ref{prop: useful properties} gives $\underline q_k\leq q(\beta)/2\leq1/16$, and hence
    \begin{align}
        \|G_k\|_2\leq \frac{16}{15}\|G_{k,\tau}\|_2. \label{eq: interpolation flat endpoint comparison 1}
    \end{align}
    Lemma~\ref{lem: flat gradient one step change} further yields
    \begin{align}
        \|G_{k+1}\|_2
        &\leq \frac{1}{1-\underline q_k(1+2\underline q_k)}\|G_k\|_2
        \leq \frac{2048}{1785}\|G_{k,\tau}\|_2, \label{eq: interpolation flat endpoint comparison 2}
    \end{align}
    where we used $\underline q_k(1+2\underline q_k)\leq9/128$ and \eqref{eq: interpolation flat endpoint comparison 1}.
    We next record some consequences of the endpoint induction. The momentum recursion, induction conditions \eqref{eq: induction 3}--\eqref{eq: induction 4}, and \eqref{eq: b k weighted sum bound} imply
    \begin{align}
        \|\bps(w_k)m_{k+1}\|_2
        &\leq\beta\|\bps(w_k)m_k\|_2+(1-\beta)\|\bps(w_k)g_k\|_2 \leq c_{\mathrm{end}}\|G_k\|_2, \notag\\
        \|\bpf(w_k)m_{k+1}\|_2
        &\leq\left((1-\beta)\sum_{j=0}^k\beta^j+b_k\right)\|G_k\|_2
        \leq2\|G_k\|_2, \notag\\
        \|m_{k+1}\|_2&\leq\frac52\|G_k\|_2, \notag \\
        \|g_k\|_2&\leq\frac32\|G_k\|_2, \label{eq: interpolation endpoint norm bounds}
    \end{align}
    where we used $c_{\mathrm{end}}\leq1/2$. The same gradient bound holds with $k+1$ in place of $k$. Also, the definition of $\mathsf{B}(\beta)$ and $\eta\in\cI_1$ give
    \begin{align}
        \eta\gamma\leq 2\cdot \frac{1+\beta}{1-\beta},
        \quad
        \max\{\eta\gamma,(\eta\gamma)^2\}\leq\mathsf{B}(\beta).
    \end{align}
    Therefore, since $\kappa\leq1$,
    \begin{align}
        \eta\gamma\kappa\leq \mathsf{B}(\beta)\kappa^{3/4}=\frac{c_{\mathrm{end}}}{3},
        \qquad
        \kappa(\eta\gamma)^2\leq\frac{c_{\mathrm{end}}}{3}. \label{eq: interpolation parameter bounds}
    \end{align}

    Now we can interpolate the momentum bound. By Lemma~\ref{lem: change sharp}, \eqref{eq: interpolation endpoint norm bounds}, \eqref{eq: interpolation parameter bounds}, and \eqref{eq: interpolation flat endpoint comparison 1},
    \begin{align}
        \|\bps(w_{k,\tau})m_{k+1}\|_2
        &\leq \|\bps(w_k)m_{k+1}\|_2+\eta\gamma\kappa\|m_{k+1}\|_2 \notag\\
        &\leq \frac{11}{6}c_{\mathrm{end}}\|G_k\|_2
        \leq \frac{88}{45}c_{\mathrm{end}}\|G_{k,\tau}\|_2
        <2c_{\mathrm{end}}\|G_{k,\tau}\|_2. \label{eq: interpolated sharp momentum bound}
    \end{align}
    It then remains to control the interpolated gradient. The established endpoint bounds imply that $\|g_j\|_2\leq(1+c_{\mathrm{end}})\|G_j\|_2\leq3\|G_j\|_2/2$ for $j\in\{k,k+1\}$. Thus, Lemma~\ref{lem: change sharp} and \eqref{eq: interpolation parameter bounds} give
    \begin{align}
        \|\bps(w_{k,\tau})g_j\|_2
        &\leq \|\bps(w_j)g_j\|_2+\eta\gamma\kappa\|g_j\|_2
        \leq\frac{3}{2}c_{\mathrm{end}}\|G_j\|_2,
        \qquad j\in\{k,k+1\}.
    \end{align}
    Taylor's theorem along the interpolation segment gives
    \begin{align}
        \nabla L(w_{k,\tau})=(1-\tau)g_k+\tau g_{k+1}+R_{k,\tau},
    \end{align}
    where item 3 of Assumption~\ref{ass: regularity} and the uniform momentum bound imply
    \begin{align}
        \|R_{k,\tau}\|_2
        &\leq\frac{\eta^2}{2}\sup_{s\in[0,1]}\|\nabla^3L(w_{k,s})\|_{\mathrm{Op}}\|m_{k+1}\|_2^2\notag\\
        &\leq\frac{\kappa(\eta\gamma)^2}{4}\|m_{k+1}\|_2
        \leq\frac{5}{24}c_{\mathrm{end}}\|G_k\|_2.\label{eq: interpolated gradient remainder}
    \end{align}
    Combining the last three displays with \eqref{eq: interpolation flat endpoint comparison 1} and \eqref{eq: interpolation flat endpoint comparison 2}, we obtain
    \begin{align}
        \|\bps(w_{k,\tau})\nabla L(w_{k,\tau})\|_2
        &\leq\left(\frac{3}{2}\cdot\frac{2048}{1785}
        +\frac{5}{24}\cdot\frac{16}{15}\right)c_{\mathrm{end}}\|G_{k,\tau}\|_2\notag\\
        &<2c_{\mathrm{end}}\|G_{k,\tau}\|_2.\label{eq: interpolated sharp gradient bound}
    \end{align}
    Since $c_{\mathrm{int}}=2c_{\mathrm{end}}=6\mathsf{B}(\beta)\kappa^{3/4}$, \eqref{eq: interpolated sharp momentum bound} and \eqref{eq: interpolated sharp gradient bound} prove \eqref{eq: flat dominance conclusion 3}. Together with the endpoint induction, this completes the proof of Lemma~\ref{lem: flat dominance momentum}.
\end{proof}

\subsubsection{Lemmas for Induction Argument}

\begin{lemma}[Uniform momentum bound]\label{lem: uniform momentum bound}
    Suppose that Assumption~\ref{ass: regularity} holds and that $w_j\in\cU$ for all $j\leq k$.
    If the GD-momentum iteration \eqref{eq: gd momentum} is initialized with $m_0=0$, then
    \begin{align}
        \|m_{k+1}\|_2\leq g_{\max}.
    \end{align}
\end{lemma}

\begin{proof}[Proof of Lemma~\ref{lem: uniform momentum bound}]
    By unrolling the momentum recursion,
    \begin{align}
        m_{k+1}=(1-\beta)\cdot \sum_{j=0}^{k}\beta^{k-j}\nabla L(w_j).
    \end{align}
    Since Assumption~\ref{ass: regularity} gives $\|\nabla L(w)\|_2\leq g_{\max}$ for all $w\in\cU$, we obtain
    \begin{align}
        \|m_{k+1}\|_2
        \leq (1-\beta)\cdot\sum_{j=0}^{k}\beta^{k-j}\|\nabla L(w_j)\|_2
        \leq (1-\beta)\cdot\sum_{j=0}^{k}\beta^{k-j}g_{\max}
        \leq g_{\max}.
    \end{align}
    This completes the proof of Lemma~\ref{lem: uniform momentum bound}.
\end{proof}

\begin{lemma}[Analysis of the sharp directions]\label{lem: sharp direction}
    Suppose that Assumptions~\ref{ass: existence} and \ref{ass: regularity} hold.
    Take the learning rate $\eta$ and momentum parameter $\beta$ satisfying
    \begin{align}
        \eta\in \cI_1:= \left(\frac{\kappa^{1/32}}{\gmax}, \,\,\frac{2\cdot (1+\beta)-\kappa^{1/16}}{1-\beta}\cdot\frac{1}{\gmax}\right),\label{eq: I 1}
    \end{align}
    then there exist orthogonal matrices $\boldsymbol{O}, \{\boldsymbol{V}_k\}_{k\geq 0}\subset\RR^{2d\times 2d}$ and positive definite matrices $\{\boldsymbol{P}_k\}_{k\geq 0}\subset\RR^{2d\times 2d}$ such that for any $k\in \NN$, it holds that
    \begin{align}
        \left\|\boldsymbol{P}_k^{\frac{1}{2}}\boldsymbol{O}^\top\boldsymbol{V}_k^\top\left(\begin{matrix}
        \bps(w_{k+1})m_{k+1}\\
        \bps(w_{k+1})\nabla L(w_{k+1})
    \end{matrix}
    \right) \right\|_2 
    & \leq\big(1-\tfrac{1}{2}\overline{q}_k\big)\cdot
    \big(1+\mathsf{A}(\beta)\cdot \kappa^{3/16}\big) \notag\\
    &\qquad\cdot\left\|\boldsymbol{P}_{k-1}^{\frac{1}{2}}\boldsymbol{O}^\top\boldsymbol{V}_{k-1}^\top\!\!\left(\begin{matrix}
        \bps(w_{k})m_{k}\\
        \bps(w_{k})\nabla L(w_{k}) 
        \end{matrix}
    \right)\right\|_2 \notag\\
    &\qquad +
    \mathsf{B}(\beta)\cdot\kappa^{15/16}\cdot \left\|\left(\begin{matrix}
        \beta\cdot m_k\\
        (1-\beta)\cdot \nabla L(w_{k})
    \end{matrix}
    \right)\right\|_2.
    \end{align}
Here $\overline{q}_k$, $\mathsf{A}(\beta)$ and $\mathsf{B}(\beta)$ are defined as
\begin{align}
    \overline{q}_k&:=\left(\max_{i\in[d-1]}\frac{(2\beta^3-2\beta^2+3\beta-1)\cdot (\eta\lambda_{k,i})^2 + 2(1-\beta^2)\cdot\eta\lambda_{k,i} + 2(1-\beta)^2(1+\beta)}{ \eta \lambda_{k,i}(1-\beta)^2\cdot(2(1+\beta)-(1-\beta) \eta \lambda_{k,i})}\right)^{-1},\\
    \mathsf{A}(\beta)&:=     \frac{16(1+\beta)}{1-\beta}\left(2+\frac{5(1+\beta)}{1-\beta}\right)\left(\frac{8}{(1-\beta)^3}+3\right)^2\left(2+\frac{4(1+\beta)}{1-\beta}+\frac{5(1+\beta)^2}{(1-\beta)^2}\right),\\
    \mathsf{B}(\beta)&:=2\cdot \left(\frac{16}{(1-\beta)^3}+6\right)^{\frac{1}{2}} \cdot \left(\frac{10}{1-\beta} + 1\right)\cdot\frac{2}{1-\beta},
\end{align}
and $\lambda_{k,1}\geq \lambda_{k,2}\geq \cdots \geq  \lambda_{k,d-1}$ are the largest $d-1$ eigenvalues of $\nabla^2 L(w_{k})$.
\end{lemma}

\begin{proof}[Proof of Lemma~\ref{lem: sharp direction}]
    By the update rule of \eqref{eq: gd momentum}, we have that
    \begin{align}
        \bps(w_{k+1})m_{k+1} &= \beta\cdot \bps(w_k)m_k + (1-\beta)\bps(w_k)\nabla L(w_k)+ \big(\bps(w_{k+1}) - \bps(w_k)\big) m_{k+1},
    \end{align}
    and that
    \allowdisplaybreaks
    \begin{align}
        \bps(w_{k+1})\nabla L(w_{k+1})
        &= \bps(w_{k})\nabla L(w_{k}) - \eta\cdot \bps(w_k)\nabla^2 L(w_k)\bps(w_k)m_{k+1}  \\
        &\qquad -\eta\cdot\int_0^1
        \begin{aligned}[t]
        \Big(&\bps(w_{k,\tau})\nabla^2 L(w_{k,\tau})\bps(w_{k,\tau})\\[-2pt]
        &-\bps(w_k)\nabla^2 L(w_k)\bps(w_k)\Big)m_{k+1}
        \end{aligned}
        \mathrm{d}\tau \\
        &\qquad  \qquad - \eta\cdot \int_0^1\nabla \bps(w_{k,\tau})[m_{k+1}]\nabla L(w_{k,\tau})\mathrm{d}\tau.
    \end{align}
    Therefore, we have the following iteration formula,
    \begin{align}
        \left(\begin{matrix}
        \bps(w_{k+1})m_{k+1}\\
        \bps(w_{k+1})\nabla L(w_{k+1})
    \end{matrix}
    \right) = \boldsymbol{T}_k\left(\begin{matrix}
        \bps(w_{k})m_{k}\\
        \bps(w_{k})\nabla L(w_{k}) 
    \end{matrix}
    \right)+ \boldsymbol{E}_k\label{eq: iteration sharp}
    \end{align}
    where $\boldsymbol{T}_k$ and $\boldsymbol{E}_k$ are defined as following,
    \begin{align}
        \boldsymbol{T}_k&:= \left(\begin{matrix}
       \beta\cdot \bps(w_k)& (1-\beta)\cdot \bps(w_k)\\
       -\eta\beta\cdot \bps(w_k)\nabla^2 L(w_k)\bps(w_k) & \bps(w_k)\big(\boldsymbol{I}_d - \eta(1-\beta)\cdot\nabla^2 L(w_k)\big)\bps(w_k)
    \end{matrix}
    \right),\\
    \boldsymbol{E}_k &:= \left(\begin{matrix}
        E_{k,1}\\
        E_{k,2}
    \end{matrix}
    \right),
    \end{align}
    and $E_{k,1}$ and $E_{k,2}$ are defined as following respectively,
    \begin{align}
        E_{k,1} &:= \big(\bps(w_{k+1}) - \bps(w_k)\big) m_{k+1} \\
        E_{k,2} &:= -\eta\cdot\int_0^1\Big(\bps(w_{k,\tau})\nabla^2 L(w_{k,\tau})\bps(w_{k,\tau}) -  \bps(w_k)\nabla^2 L(w_k)\bps(w_k)\Big)m_{k+1}\mathrm{d}\tau \\
        & \qquad - \eta\cdot \int_0^1\nabla \bps(w_{k,\tau})[m_{k+1}]\nabla L(w_{k,\tau})\mathrm{d}\tau.
    \end{align}
    We will repeatedly use the following one-step variation estimates.
    By Lemma~\ref{lem: uniform momentum bound}, $\|m_{k+1}\|_2\leq g_{\max}$, and hence $\|w_{k,\tau}-w_k\|_2\leq \eta g_{\max}$ for all $\tau\in[0,1]$.
    Therefore, Assumption~\ref{ass: regularity} implies
    \begin{align}
        \|\bps(w_{k,\tau})-\bps(w_k)\|_{\mathrm{Op}}
        &\leq \eta\gamma\kappa,\nonumber\\
        \|\nabla^2L(w_{k,\tau})-\nabla^2L(w_k)\|_{\mathrm{Op}}
        &\leq \frac{1}{2}\eta\gamma^2\kappa.
        \label{eq: one step variation estimates}
    \end{align}
    The same estimates with $(k,\tau)$ replaced by $(k-1,1)$ control the differences between $w_k$ and $w_{k-1}$.
    
    \paragraph{Step 1: Analysis of matrix $\boldsymbol{T}_k$.}
    We denote the eigenvalues of the matrix $\bps(w_k)\nabla^2 L(w_k)\bps(w_k)$ as $\lambda_{k,1}\geq \lambda_{k,2}\geq \cdots \geq  \lambda_{k,d}$, and we denote the corresponding orthonormal eigenvectors as $\{v_{k,i}\}_{i=1}^d$. 
    Specifically, we have that $\lambda_{k,d}=0$ and that $v_{k,i} = v_i(\nabla^2L(w_k))$ for any $i\in [d]$. 
    Then we denote the orthogonal matrix $\boldsymbol{V}_k\in \RR^{2d\times 2d}$ and a diagonal matrix $\boldsymbol{\Lambda}_k\in\RR^{d\times d}$ as 
    \begin{align}
        \boldsymbol{V}_k:=\left(\begin{matrix}
        v_{k,1},\cdots,v_{k,d} & \boldsymbol{0}\\
        \boldsymbol{0} & v_{k,1},\cdots,v_{k,d},
    \end{matrix}
    \right),\quad \boldsymbol{\Lambda}_{k}:=\mathrm{diag}\Big(\lambda_{k,1},\cdots,\lambda_{k,d-1},0\Big)
    \end{align} 
    which allows us to transform the matrix $\boldsymbol{T}_k$ to 
    \begin{align}
        \boldsymbol{T}_k = \boldsymbol{V}_k \left(\begin{matrix}
        \beta\cdot \boldsymbol{I}_{1:d-1} & (1-\beta)\cdot \boldsymbol{I}_{1:d-1}\\
        -\eta\beta\cdot \boldsymbol{\Lambda}_k & \boldsymbol{I}_{1:d-1} - \eta(1-\beta)\cdot\boldsymbol{\Lambda}_{k}
    \end{matrix}
    \right)\boldsymbol{V}_k^\top.
    \end{align}
    Furthermore, we can find another constant orthogonal matrix $\boldsymbol{O}\in\RR^{2d\times 2d}$ such that 
    \begin{align}
        \boldsymbol{T}_k &= \boldsymbol{V}_k \boldsymbol{O}\mathrm{diag}\Big(\boldsymbol{T}_{k,1},\cdots,\boldsymbol{T}_{k,d-1},\boldsymbol{0}_{2\times 2}\Big)\boldsymbol{O}^\top\boldsymbol{V}_k^\top,\quad\text{where}\\
        \boldsymbol{T}_{k,i} &= \left(\begin{matrix}
        \beta & (1-\beta)\\
        -\beta\cdot \eta\lambda_{k,i} & 1 - (1-\beta)\cdot\eta\lambda_{k,i}
    \end{matrix}
    \right)\in \RR^{2\times 2},\quad \forall i\in[d-1].
    \end{align}
    Now for each block $\boldsymbol{T}_{k,i}$ with $i\in[d-1]$, consider the following Lyapunov equation, 
    \begin{align}\label{eq: block Lyapunov equation}
        \boldsymbol{T}_{k,i}^\top \boldsymbol{P}_{k,i}\boldsymbol{T}_{k,i} = \boldsymbol{P}_{k,i}-\boldsymbol{I}_2.
    \end{align}
    By $\eta\in \cI_1$, especially using that $\eta\gmax < 2(1+\beta)/(1-\beta)$, the spectral radius of $\boldsymbol{T}_{k,i}$ satisfies that $\rho(\boldsymbol{T}_{k,i})<1$ for any $k\geq 0$ and $i\in[d-1]$ (see Proposition~\ref{prop: eigen values}).
    Therefore, the above Lyapunov equation admits a unique positive definite solution $\boldsymbol{P}_{k,i}\in\RR^{2\times 2}$, given by 
    \begin{align}
         \boldsymbol{P}_{k,i} \!= \!\!\sum_{j=0}^{+\infty} (\boldsymbol{T}_{k,i}^\top)^j (\boldsymbol{T}_{k,i})^j = \left(\begin{matrix}
\frac{2(1-\beta)+(2\beta^3-2\beta^2+3\beta-1) \eta \lambda_{k,i}}{(1-\beta)^2\cdot(2(1+\beta)-(1-\beta) \eta \lambda_{k,i})} & \frac{-\beta(2 \beta-\eta \lambda_{k,i})}{(1-\beta)\cdot(2(1+\beta)-(1-\beta) \eta \lambda_{k,i})} \\
\frac{-\beta(2 \beta-\eta \lambda_{k,i})}{(1-\beta)\cdot(2(1+\beta)-(1-\beta) \eta \lambda_{k,i})} & \frac{-2(-1+\beta^2-\beta \eta \lambda_{k,i})}{ \eta \lambda_{k,i}(1-\beta)\cdot(2(1+\beta)-(1-\beta) \eta \lambda_{k,i})}
\end{matrix}\right) \succ \boldsymbol{I}_2 .    
    \end{align}
    Also, from the Lyapunov equation we can write $\boldsymbol{P}_{k,i}^{\frac{1}{2}}\boldsymbol{T}_{k,i} = \boldsymbol{Q}_{k,i}(\boldsymbol{P}_{k,i}-\boldsymbol{I}_2)^{\frac{1}{2}}$ for some othogonal matrix $\boldsymbol{Q}_{k,i}$.
   Now we combine these $\boldsymbol{P}_{k,i}$'s to form a block diagonal matrix $\boldsymbol{P}_k\in\RR^{2d\times 2d}$ as
    \begin{align}
        \boldsymbol{P}_k := \mathrm{diag}\Big(\boldsymbol{P}_{k,1},\cdots,\boldsymbol{P}_{k,d-1},\boldsymbol{I}_2\Big).
    \end{align}

    \paragraph{Step 2: Bounding the next step in the sharp direction: main analysis.}
    With the analysis of matrix $\boldsymbol{T}_k$, we can continue the analysis for upper bounding the sharp direction. 
    We have 
    \begin{align}
        \left(\begin{matrix}
        \bps(w_{k+1})m_{k+1}\\
        \bps(w_{k+1})\nabla L(w_{k+1})
    \end{matrix}
    \right) 
    & = \boldsymbol{T}_k\left(\begin{matrix}
        \bps(w_{k})m_{k}\\
        \bps(w_{k})\nabla L(w_{k}) 
    \end{matrix}
    \right)+ \boldsymbol{E}_k  \\
    &=\boldsymbol{V}_k \boldsymbol{O}\mathrm{diag}\Big(\boldsymbol{T}_{k,1},\cdots,\boldsymbol{T}_{k,d-1},\boldsymbol{0}_{2\times 2}\Big)\boldsymbol{O}^\top\boldsymbol{V}_k^\top\left(\begin{matrix}
        \bps(w_{k})m_{k}\\
        \bps(w_{k})\nabla L(w_{k}) 
    \end{matrix}
    \right)+ \boldsymbol{E}_k,
    \end{align}
    and thus 
    \begin{align}
        &\boldsymbol{P}_k^{\frac{1}{2}}\boldsymbol{O}^\top\boldsymbol{V}_k^\top\left(\begin{matrix}
        \bps(w_{k+1})m_{k+1}\\
        \bps(w_{k+1})\nabla L(w_{k+1})
    \end{matrix}
    \right) 
    \\
    &\qquad  = \boldsymbol{P}_k^{\frac{1}{2}}\mathrm{diag}\Big(\boldsymbol{T}_{k,1},\cdots,\boldsymbol{T}_{k,d-1},\boldsymbol{0}_{2\times 2}\Big)\boldsymbol{O}^\top\boldsymbol{V}_k^\top\left(\begin{matrix}
        \bps(w_{k})m_{k}\\
        \bps(w_{k})\nabla L(w_{k}) 
    \end{matrix}    \right) + \boldsymbol{P}_k^{\frac{1}{2}}\boldsymbol{O}^\top\boldsymbol{V}_k^\top \boldsymbol{E}_k.
    \end{align}
    Consider the $\|\cdot\|_2$-norm on both sides, we have
    \begin{align}
        &\left\|\boldsymbol{P}_k^{\frac{1}{2}}\boldsymbol{O}^\top\boldsymbol{V}_k^\top\left(\begin{matrix}
        \bps(w_{k+1})m_{k+1}\\
        \bps(w_{k+1})\nabla L(w_{k+1})
        \end{matrix}
        \right) \right\|_2 \label{eq: iteration origin}\\
        &\qquad \leq \left\|\boldsymbol{P}_k^{\frac{1}{2}}\mathrm{diag}\Big(\boldsymbol{T}_{k,1},\cdots,\boldsymbol{T}_{k,d-1},\boldsymbol{0}_{2\times 2}\Big)\boldsymbol{O}^\top\boldsymbol{V}_k^\top\left(\begin{matrix}
            \bps(w_{k})m_{k}\\
            \bps(w_{k})\nabla L(w_{k}) 
        \end{matrix}    \right) \right\|_2 + \left\|\boldsymbol{P}_k^{\frac{1}{2}}\boldsymbol{O}^\top\boldsymbol{V}_k^\top \boldsymbol{E}_k\right\|_2.
    \end{align}
    In {\bf Step 2} we focus on the first term on the right-hand side.
    For any $x\in\RR^2$, the Lyapunov equation \eqref{eq: block Lyapunov equation} implies
    \begin{align}
        \|\boldsymbol{P}_{k,i}^{1/2}\boldsymbol{T}_{k,i}x\|_2^2
        = x^\top(\boldsymbol{P}_{k,i}-\boldsymbol{I}_2)x.
    \end{align}
    Since $\boldsymbol{P}_{k,i}\succeq \boldsymbol{I}_2$ and
    $\boldsymbol{P}_{k,i}\preceq \|\boldsymbol{P}_{k,i}\|_{\mathrm{Op}}\boldsymbol{I}_2$, we have
    $\boldsymbol{I}_2\succeq \|\boldsymbol{P}_{k,i}\|_{\mathrm{Op}}^{-1}\boldsymbol{P}_{k,i}$ and hence
    \begin{align}
        \boldsymbol{P}_{k,i}-\boldsymbol{I}_2
        \preceq
        \left(1-\|\boldsymbol{P}_{k,i}\|_{\mathrm{Op}}^{-1}\right)\boldsymbol{P}_{k,i}.
    \end{align}
    Therefore, by the definition of $\boldsymbol{P}_k$,
    \begin{align}
        &\left\|\boldsymbol{P}_k^{\frac{1}{2}}\mathrm{diag}\Big(\boldsymbol{T}_{k,1},\cdots,\boldsymbol{T}_{k,d-1},\boldsymbol{0}_{2\times 2}\Big)\boldsymbol{O}^\top\boldsymbol{V}_k^\top\left(\begin{matrix}
            \bps(w_{k})m_{k}\\
            \bps(w_{k})\nabla L(w_{k}) 
        \end{matrix}    \right) \right\|_2^2 \nonumber\\
        &\qquad \leq
        \max_{i\in[d-1]}\left(1-\|\boldsymbol{P}_{k,i}\|_{\mathrm{Op}}^{-1}\right)
        \left\|\boldsymbol{P}_{k}^{\frac{1}{2} }\boldsymbol{O}^\top \boldsymbol{V}_{k}^\top\left(\begin{matrix}
            \bps(w_{k})m_{k}\\
            \bps(w_{k})\nabla L(w_{k}) 
        \end{matrix}    \right) \right\|_2^2.
    \end{align}
    Since $\|\boldsymbol{P}_{k,i}\|_{\mathrm{Op}}\leq \mathrm{Tr}(\boldsymbol{P}_{k,i})$, we have
    \begin{align}
        \max_{i\in[d-1]}\left(1-\|\boldsymbol{P}_{k,i}\|_{\mathrm{Op}}^{-1}\right)
        \leq 1-\left(\max_{i\in[d-1]}\mathrm{Tr}(\boldsymbol{P}_{k,i})\right)^{-1}
        :=1-\overline{q}_k,
    \end{align}
    where $\overline{q}_k$ is the squared-norm contraction parameter. Taking square roots and using $\sqrt{1-x}\leq 1-x/2$ for $x\in[0,1]$, we obtain
    \begin{align}
        &\left\|\boldsymbol{P}_k^{\frac{1}{2}}\mathrm{diag}\Big(\boldsymbol{T}_{k,1},\cdots,\boldsymbol{T}_{k,d-1},\boldsymbol{0}_{2\times 2}\Big)\boldsymbol{O}^\top\boldsymbol{V}_k^\top\left(\begin{matrix}
            \bps(w_{k})m_{k}\\
            \bps(w_{k})\nabla L(w_{k}) 
        \end{matrix}    \right) \right\|_2 \nonumber\\
        &\qquad \leq
        \big(1-\tfrac{1}{2}\overline{q}_k\big)
        \left\|\boldsymbol{P}_{k}^{\frac{1}{2} }\boldsymbol{O}^\top \boldsymbol{V}_{k}^\top\left(\begin{matrix}
            \bps(w_{k})m_{k}\\
            \bps(w_{k})\nabla L(w_{k}) 
        \end{matrix}    \right) \right\|_2. \label{eq: first term sharp}
    \end{align}
    The trace term in $\overline{q}_k$ is given by
    \begin{align}
        \mathrm{Tr}(\boldsymbol{P}_{k,i}) &= \frac{2(1-\beta)+(2\beta^3-2\beta^2+3\beta-1) \eta \lambda_{k,i}}{(1-\beta)^2\cdot(2(1+\beta)-(1-\beta) \eta \lambda_{k,i})} -\frac{2(-1+\beta^2-\beta \eta \lambda_{k,i})}{ \eta \lambda_{k,i}(1-\beta)\cdot(2(1+\beta)-(1-\beta) \eta \lambda_{k,i})}\\
        & =  \frac{(2\beta^3-2\beta^2+3\beta-1)\cdot (\eta\lambda_{k,i})^2 + 2(1-\beta^2)\cdot\eta\lambda_{k,i} + 2(1-\beta)^2(1+\beta)}{ \eta \lambda_{k,i}(1-\beta)^2\cdot(2(1+\beta)-(1-\beta) \eta \lambda_{k,i})}\\
        & = \frac{(2\beta^3-2\beta^2+3\beta-1)\cdot (\eta\lambda_{k,i})^2 + 2(1-\beta^2)\cdot\eta\lambda_{k,i} + 2(1-\beta)^2(1+\beta)}{ \eta \lambda_{k,i}(1-\beta)^2\cdot D(\eta\lambda_{k,i})},
        \label{eq: trace P k i}
    \end{align}
    where we denote $D(x):=2(1+\beta)-(1-\beta)x$ for any $x\in\RR$ for simplicity.

    By $\eta\in\cI_1$, we have that
    \begin{align}
        \eta\leq \frac{2(1+\beta)-\kappa^{1/16}}{(1-\beta)\cdot \gmax} \leq \frac{2(1+\beta)-\kappa^{1/16}}{(1-\beta)\cdot  \lambda_{k,i}},\quad \forall k\geq 0,\,i\in[d-1], \\ 
        \eta\geq \frac{\kappa^{1/32}}{\gmax} \geq \frac{\kappa^{1/32}}{ \lambda_{k,i}}\cdot\frac{\gamma}{\gmax} \geq\frac{\kappa^{1/16}}{ \lambda_{k,i}},\quad \forall k\geq 0,\,i\in[d-1],
    \end{align}
    where we have used item 2 of Assumption~\ref{ass: regularity} that $\gamma/\gmax\geq \kappa^{1/32}$.
    Therefore, we can lower bound
    \begin{align}
       D(\eta\lambda_{k,i}) \geq \kappa^{1/16},\quad \eta\lambda_{k,i} \geq \kappa^{1/16},\quad \forall k\geq 0,\,i\in[d-1].
    \end{align}
    Moreover 
    \begin{align}
        \eta\lambda_{k,i}\le \eta \gamma_{\max } \leq \frac{2(1+\beta)}{1-\beta} \leq \frac{4}{1-\beta},\quad \forall k\geq 0,\,i\in[d-1].
    \end{align}
    
    Now we can bound $\|\boldsymbol{P}_{k,i}\|_{\mathrm{Op}}$ by 
    \begin{align}
    \|\boldsymbol{P}_{k,i}\|_{\mathrm{Op}} \le \mathrm{Tr} (\boldsymbol{P}_{k,i})&\le \frac{2(1-\beta) + 2\eta\lambda_{k,i}}{(1-\beta)^2D(\eta\lambda_{k,i})} + \frac{2(1+\beta)}{\eta\lambda_{k,i}D(\eta\lambda_{k,i})}+\frac{2\beta}{(1-\beta)D(\eta\lambda_{k,i})}\\
    &\le \frac{2}{1-\beta}\kappa^{-1/16} + \frac{8}{(1-\beta)^3}\kappa^{-1/16} + 4\kappa^{-1/8}+\frac{2}{1-\beta}\kappa^{-1/16} \\
    &\le \left(\frac{16}{(1-\beta)^3}+6\right)\kappa^{-1/8} \label{eq: upper bound on sub-block P operator norm}
    \end{align}
    We now define the ambient Lyapunov weight by rotating $\boldsymbol{P}_k$ in the basis of $\boldsymbol{V}_k \boldsymbol{O}$:
\begin{align}
    \widetilde{\boldsymbol{P}}_k:=\boldsymbol{V}_k\boldsymbol{O} \boldsymbol{P}_k\boldsymbol{O}^\top \boldsymbol{V}_k^\top .
\end{align}
Then 
\begin{align}
    \left\|\boldsymbol{P}_k^{1/2}\boldsymbol{O}^\top \boldsymbol{V}_k^\top \left(\begin{matrix}
            \bps(w_{k})m_{k}\\
            \bps(w_{k})\nabla L(w_{k}) 
        \end{matrix}    \right)\right\|_2^2
    =
    \left(\begin{matrix}
            \bps(w_{k})m_{k}\\
            \bps(w_{k})\nabla L(w_{k}) 
        \end{matrix}    \right)^\top\widetilde{\boldsymbol{P}}_k\left(\begin{matrix}
            \bps(w_{k})m_{k}\\
            \bps(w_{k})\nabla L(w_{k}) 
        \end{matrix}    \right)
\end{align}
and 
\begin{align}
    \boldsymbol{T}_k^\top\widetilde{\boldsymbol{P}}_k\boldsymbol{T}_k
    =
    \widetilde{\boldsymbol{P}}_k-\boldsymbol{I}_{2d},
    \qquad
    \widetilde{\boldsymbol{P}}_k\succeq \boldsymbol{I}_{2d}
\end{align}
by transforming the Lyapunov equation~\eqref{eq: block Lyapunov equation}.
We next prove the following perturbation bound:
\begin{align}\label{eq: bound on difference of tilde P_k}
    \|\widetilde{\boldsymbol{P}}_k-\widetilde{\boldsymbol{P}}_{k-1}\|_{\mathrm{Op}}
    \le
    \mathsf{A}(\beta)\kappa^{3/4}.
\end{align}
For better organizing the proof, we introduce following constants in $\beta$:
\begin{align}
    \mathsf{R}_\beta:=\frac{2(1+\beta)}{1-\beta}, \quad   \mathsf{C}_{P}(\beta):=\frac{16}{(1-\beta)^3}+6,\quad
    \mathsf{C}_{T}(\beta):=2+\mathsf{R}_\beta, \quad   \mathsf{C}_{\Delta}(\beta):=
    \mathsf{R}_\beta\left(2+\frac52\mathsf{R}_\beta\right),
\end{align}
and 
\begin{align}
    \mathsf{A}(\beta)
    &:=
    \mathsf{C}_{P}(\beta)^2
    \left(
        2\mathsf{C}_{T}(\beta)\mathsf{C}_{\Delta}(\beta)
        +
        \mathsf{C}_{\Delta}(\beta)^2
    \right)\\
    &=\frac{16(1+\beta)}{1-\beta}\left(2+\frac{5(1+\beta)}{1-\beta}\right)\left(\frac{8}{(1-\beta)^3}+3\right)^2\left(2+\frac{4(1+\beta)}{1-\beta}+\frac{5(1+\beta)^2}{(1-\beta)^2}\right).
\end{align}
Next we define the Lyapunov operator
\begin{align}
     L_k(\boldsymbol{X}):=\boldsymbol{X}-\boldsymbol{T}_k^\top \boldsymbol{X}\boldsymbol{T}_k.
\end{align}
Since $\rho(\boldsymbol{T}_k)<1$ by \textbf{Step 1}, \( L_k\) is invertible and
\begin{align}
      L_k^{-1}(\boldsymbol{X})
    =
    \sum_{j=0}^{\infty}
    (\boldsymbol{T}_k^\top)^j\boldsymbol{X}\boldsymbol{T}_k^j .
\end{align}
By definition of $\widetilde{\boldsymbol{P}}_k$ we have
\begin{align}\label{eq: Lyapunov operator on P_k}
      L_k(\widetilde{\boldsymbol{P}}_k)=\boldsymbol{I}_{2d},
    \quad   L_{k}^{-1}(\boldsymbol{I}_{2d}) = \widetilde{\boldsymbol{P}}_k, \quad
      L_{k-1}(\widetilde{\boldsymbol{P}}_{k-1})=\boldsymbol{I}_{2d}. 
\end{align}
For any symmetric matrix \(\boldsymbol{X}=\boldsymbol{X}^\top\), we have
\begin{align}
    -\|\boldsymbol{X}\|_{\mathrm{Op}}\boldsymbol{I}_{2d}
    \preceq \boldsymbol{X}
    \preceq
    \|\boldsymbol{X}\|_{\mathrm{Op}}\boldsymbol{I}_{2d} \quad \Rightarrow     \quad -\|\boldsymbol{X}\|_{\mathrm{Op}}\widetilde{\boldsymbol{P}}_k
    \preceq
      L_k^{-1}(\boldsymbol{X})
    \preceq
    \|\boldsymbol{X}\|_{\mathrm{Op}}\widetilde{\boldsymbol{P}}_k.
\end{align}
Therefore
\begin{align}
    \left\|  L_k^{-1}(\boldsymbol{X})\right\|_{\mathrm{Op}}
    \le
    \|\boldsymbol{X}\|_{\mathrm{Op}}\|\widetilde{\boldsymbol{P}}_k\|_{\mathrm{Op}}.
\end{align}
Apply $L_k$ on $\widetilde{\boldsymbol{P}}_k-\widetilde{\boldsymbol{P}}_{k-1}$ and by identity~\eqref{eq: Lyapunov operator on P_k}, we obtain
\begin{align}
      L_k(\widetilde{\boldsymbol{P}}_k-\widetilde{\boldsymbol{P}}_{k-1})
    &=
    \boldsymbol{T}_k^\top\widetilde{\boldsymbol{P}}_{k-1}\boldsymbol{T}_k
    -
    \boldsymbol{T}_{k-1}^\top\widetilde{\boldsymbol{P}}_{k-1}\boldsymbol{T}_{k-1}.
\end{align}
Let $\Delta \boldsymbol{T}_k:=\boldsymbol{T}_k-\boldsymbol{T}_{k-1}$, then
\begin{align}
    \boldsymbol{T}_k^\top\widetilde{\boldsymbol{P}}_{k-1}\boldsymbol{T}_k
    -
    \boldsymbol{T}_{k-1}^\top\widetilde{\boldsymbol{P}}_{k-1}\boldsymbol{T}_{k-1}
    =
    \boldsymbol{T}_{k-1}^\top\widetilde{\boldsymbol{P}}_{k-1}\Delta \boldsymbol{T}_k
    +
    \Delta \boldsymbol{T}_k^\top\widetilde{\boldsymbol{P}}_{k-1}\boldsymbol{T}_{k-1}
    +
    \Delta \boldsymbol{T}_k^\top\widetilde{\boldsymbol{P}}_{k-1}\Delta \boldsymbol{T}_k .
\end{align}
Thus
\begin{align}
    \left\|\widetilde{\boldsymbol{P}}_k-\widetilde{\boldsymbol{P}}_{k-1}\right\|_{\mathrm{Op}}
    &= \left\|L_k^{-1}(L_k(\widetilde{\boldsymbol{P}}_k-\widetilde{\boldsymbol{P}}_{k-1}))\right\|_{\mathrm{Op}} \notag\\
    &\le
    \|\widetilde{\boldsymbol{P}}_k\|_{\mathrm{Op}}
    \|\widetilde{\boldsymbol{P}}_{k-1}\|_{\mathrm{Op}}
    \left(
        2\|\boldsymbol{T}_{k-1}\|_{\mathrm{Op}}\|\Delta \boldsymbol{T}_k\|_{\mathrm{Op}}
        +
        \|\Delta \boldsymbol{T}_k\|_{\mathrm{Op}}^2
    \right).
\end{align}
From definition we have 
\begin{align}
    \|\widetilde{\boldsymbol{P}}_t\|_{\mathrm{Op}} = \|\boldsymbol{V}_t\boldsymbol{O} \boldsymbol{P}_t\boldsymbol{O}^\top \boldsymbol{V}_t^\top\|_{\mathrm{Op}}=\|\boldsymbol{P}_t\| _{\mathrm{Op}},\quad t \in \{k-1,k\}.
\end{align}
And from equation~\eqref{eq: upper bound on sub-block P operator norm}
\begin{align}
    \left\|\boldsymbol{P}_t\right\|_{\mathrm{Op}}=\max \left\{1, \max _{i \in[d-1]}\left\|\boldsymbol{P}_{t, i}\right\|_{\mathrm{Op}}\right\} \leq \mathsf{C}_{P}(\beta) \kappa^{-1 / 8} .
\end{align}
It remains to bound $\|\boldsymbol{T}_{k-1}\|_{\mathrm{Op}}$ and \(\|\Delta \boldsymbol{T}_k\|_{\mathrm{Op}}\). 
Recall
\begin{align}
    \boldsymbol{T}_t&=
    \begin{pmatrix}
        \beta  \bps(w_t) & (1-\beta) \bps(w_t)\\
        -\eta\beta \bps(w_t) \nabla^2L(w_t) \bps(w_t) &  \bps(w_t)-\eta(1-\beta)\bps(w_t) \nabla^2L(w_t) \bps(w_t)
    \end{pmatrix}\\
    &=\left(\begin{array}{cc}
    \beta \bps(w_t) & (1-\beta) \bps(w_t) \\
    0 & \bps(w_t)
    \end{array}\right) \notag\\
    &\qquad+\eta\left(\begin{array}{cc}
    0 & 0 \\
    -\beta \bps(w_t) \nabla^2L(w_t) \bps(w_t) & -(1-\beta) \bps(w_t) \nabla^2L(w_t) \bps(w_t)
\end{array}\right) .
\end{align}
Therefore,
\begin{align}
    \left\|\boldsymbol{T}_{k-1}\right\|_{\mathrm{Op}} \leq 2+\eta \gamma_{\max }\le
    \mathsf{C}_{T}(\beta),
\end{align}
where we used $\|\bps(w_t)\|_{\mathrm{Op}}=1, \|\bps(w_t) \nabla^2L(w_t) \bps(w_t)\|_{\mathrm{Op}}\le\gamma_{\mathrm{max}}$.
Similarly, for $\Delta \boldsymbol{T}_{k}$:
\begin{align}
    \|\Delta \boldsymbol{T}_k\|_{\mathrm{Op}}
    &\le
    2\|\bps(w_k)-\bps(w_{k-1})\|_{\mathrm{Op}}
    \\
    &\qquad+
    \eta\|\bps(w_k)\nabla^2L(w_{k})\bps(w_k)-\bps(w_{k-1})\nabla^2L(w_{k-1})\bps(w_{k-1})\|_{\mathrm{Op}}.
\end{align}
Using the one-step variation estimates in \eqref{eq: one step variation estimates},
\begin{align}
    \|\bps(w_k)-\bps(w_{k-1})\|_{\mathrm{Op}}
    \le
    \eta\gamma\kappa,
    \quad
    \|\nabla^2L(w_k)-\nabla^2L(w_{k-1})\|_{\mathrm{Op}}
    \le
    \frac12\eta\gamma^2\kappa,
\end{align}
and moreover,
\begin{align}
    &\left\|
    \begin{aligned}
    &\bps(w_k)\nabla^2L(w_{k})\bps(w_k)\\[-2pt]
    &\qquad-\bps(w_{k-1})\nabla^2L(w_{k-1})\bps(w_{k-1})
    \end{aligned}
    \right\|_{\mathrm{Op}}\\
    &\qquad \le
    2\gamma_{\max}\|\bps(w_k)-\bps(w_{k-1})\|_{\mathrm{Op}}+
    \|\nabla^2L(w_k)-\nabla^2L(w_{k-1})\|_{\mathrm{Op}}
    \\
    &\qquad \le
    2\eta\gamma\gamma_{\max}\kappa
    +
    \frac12\eta\gamma^2\kappa .
\end{align}
Consequently,
\begin{align}
    \|\Delta \boldsymbol{T}_k\|_{\mathrm{Op}}
    &\le
    \eta\gamma\kappa
    \left(
        2+2\eta\gamma_{\max}+\frac12\eta\gamma
    \right)\le
    \mathsf{R}_\beta
    \left(
        2+\frac52\mathsf{R}_\beta
    \right)\kappa
    =
    \mathsf{C}_{\Delta}(\beta)\kappa,
\end{align}
where we used \(\gamma\le\gamma_{\max}\) and
\begin{align}
    \eta\gamma_{\max}
    \le
    \frac{2(1+\beta)}{1-\beta}
    =
    \mathsf{R}_\beta .
\end{align}
Putting these estimates together gives
\begin{align}
    \left\|\widetilde{\boldsymbol{P}}_k-\widetilde{\boldsymbol{P}}_{k-1}\right\|_{\mathrm{Op}}
    &\le
    \mathsf{C}_{P}(\beta)^2\kappa^{-1/4}
    \left(
        2\mathsf{C}_{T}(\beta)\mathsf{C}_{\Delta}(\beta)\kappa
        +
        \mathsf{C}_{\Delta}(\beta)^2\kappa^2
    \right)
    \\
    &\le
    \mathsf{C}_{P}(\beta)^2
    \left(
        2\mathsf{C}_{T}(\beta)\mathsf{C}_{\Delta}(\beta)
        +
        \mathsf{C}_{\Delta}(\beta)^2
    \right)\kappa^{3/4}
    \\
    &=
    \mathsf{A}(\beta)\kappa^{3/4},
\end{align}
where the second inequality uses \(\kappa\le1\).  This proves~\eqref{eq: bound on difference of tilde P_k}.
Finally, define the sharp-direction state vector
\begin{align}
    \boldsymbol{z}_k^{\mathrm{sh}}
    :=
    \begin{pmatrix}
        \bps(w_k)m_k\\
        \bps(w_k)\nabla L(w_k)
    \end{pmatrix}.
\end{align}
Then
\begin{align}
    \left\|\boldsymbol{P}_k^{1/2}\boldsymbol{O}^\top \boldsymbol{V}_k^\top
    \boldsymbol{z}_k^{\mathrm{sh}}\right\|_2^2
    &=
    (\boldsymbol{z}_k^{\mathrm{sh}})^\top
    \widetilde{\boldsymbol{P}}_k\boldsymbol{z}_k^{\mathrm{sh}}
    \\
    &\le
    (\boldsymbol{z}_k^{\mathrm{sh}})^\top
    \widetilde{\boldsymbol{P}}_{k-1}\boldsymbol{z}_k^{\mathrm{sh}}
    +
    \|\widetilde{\boldsymbol{P}}_k-\widetilde{\boldsymbol{P}}_{k-1}\|_{\mathrm{Op}}
    \|\boldsymbol{z}_k^{\mathrm{sh}}\|_2^2
    \\
    &\le
    \left(
        1+\mathsf{A}(\beta)\kappa^{3/4}
    \right)
    (\boldsymbol{z}_k^{\mathrm{sh}})^\top
    \widetilde{\boldsymbol{P}}_{k-1}\boldsymbol{z}_k^{\mathrm{sh}}
    \\
    &=
    \left(
        1+\mathsf{A}(\beta)\kappa^{3/4}
    \right)
    \left\|\boldsymbol{P}_{k-1}^{1/2}\boldsymbol{O}^\top \boldsymbol{V}_{k-1}^\top
    \boldsymbol{z}_k^{\mathrm{sh}}\right\|_2^2.
\end{align}
Taking square roots and using
\(1+\mathsf{A}(\beta)\kappa^{3/4}>1, \kappa^{3/4}\le\kappa^{3/16}\) gives
\begin{align}\label{eq: step 2 conclusion}
    \left\|\boldsymbol{P}_k^{1/2}\boldsymbol{O}^\top \boldsymbol{V}_k^\top
    \boldsymbol{z}_k^{\mathrm{sh}}\right\|_2
    \le
    \left(
        1+\mathsf{A}(\beta)\kappa^{3/16}
    \right)
    \left\|\boldsymbol{P}_{k-1}^{1/2}\boldsymbol{O}^\top \boldsymbol{V}_{k-1}^\top
    \boldsymbol{z}_k^{\mathrm{sh}}\right\|_2 .
\end{align}

    \paragraph{Step 3: Bounding the next step in the sharp direction: remaining terms.}
    It then remains to control the remaining terms due to the spinning of the river, i.e., $\|\boldsymbol{P}_k^{1/2}\boldsymbol{O}^\top\boldsymbol{V}_k^\top\boldsymbol{E}_k\|_2$.
    To this end, consider 
    \begin{align}
        \left\|\boldsymbol{P}_k^{\frac{1}{2}}\boldsymbol{O}^\top\boldsymbol{V}_k^\top\boldsymbol{E}_k\right\|_2 \leq \left\|\boldsymbol{P}_k^{\frac{1}{2}}\right\|_{\mathrm{Op}}\cdot \|\boldsymbol{E}_k\|_2.
    \end{align}
    For the first term $\|\boldsymbol{P}_k^{1/2}\|_{\mathrm{Op}}$, as we have shown in \textbf{Step 2}, it holds that 
    \begin{align}
        \|\boldsymbol{P}_k^{1/2}\|_{\mathrm{Op}} = \|\boldsymbol{P}_k\|_{\mathrm{Op}}^{1/2} \leq \left(\frac{16}{(1-\beta)^3}+6\right)^{\frac{1}{2}}\cdot \kappa^{-1/16}.\label{eq: b k inverse bound}
    \end{align}
    For the second term $\|\boldsymbol{E}_k\|_2$, by its definition in \eqref{eq: iteration sharp}, we have that 
    \begin{align}
        \|\boldsymbol{E}_k\|_2 \leq \|E_{k,1}\|_2 + \|E_{k,2}\|_2,
    \end{align}
    where we can bound $\|E_{k,1}\|_2$ and $\|E_{k,2}\|_2$ respectively.
    Firstly, the term $\|E_{k,1}\|_2$ can be bounded by 
    \begin{align}
        \|E_{k,1}\|_2 &\leq \left\|\big(\bps(w_{k+1}) - \bps(w_k)\big) m_{k+1}\right\|_2 \leq \kappa\eta\gamma \cdot \big(\beta\cdot\|m_k\|_2+(1-\beta)\cdot\|\nabla L(w_k)\|_2\big),\label{eq: e k 1}
    \end{align}
    where the second inequality uses \eqref{eq: one step variation estimates}.
    Secondly, the term $\|E_{k,2}\|_2$ can be  bounded  by 
    \allowdisplaybreaks
    \begin{align}
        \|E_{k,2}\|_2 &\leq \left\|\eta\cdot\int_0^1\Big(\bps(w_{k,\tau})\nabla^2 L(w_{k,\tau})\bps(w_{k,\tau}) -  \bps(w_k)\nabla^2 L(w_k)\bps(w_k)\Big)m_{k+1}\mathrm{d}\tau\right\|_2 \\
        &\qquad + \left\|\eta\cdot \int_0^1\nabla \bps(w_{k,\tau})[m_{k+1}]\nabla L(w_{k,\tau})\mathrm{d}\tau\right\|_2,\label{eq: e k 2}
    \end{align}
    We further have the following bounds for the right hand side of the above, 
    \begin{align}
        &\left\|\eta\cdot\int_0^1\Big(\bps(w_{k,\tau})\nabla^2 L(w_{k,\tau})\bps(w_{k,\tau}) -  \bps(w_k)\nabla^2 L(w_k)\bps(w_k)\Big)m_{k+1}\mathrm{d}\tau\right\|_2 \\
        &\qquad  \leq \left\|\eta\cdot\int_0^1\big(\bps(w_{k,\tau}) - \bps(w_k)\big)\nabla^2 L(w_{k,\tau})\bps(w_{k,\tau}) m_{k+1}\mathrm{d}\tau\right\|_2 \\
        &\qquad \qquad + \left\|\eta\cdot\int_0^1\bps(w_k)\big(\nabla^2 L(w_{k,\tau}) - \nabla^2 L(w_k)\big)\bps(w_{k,\tau}) m_{k+1}\mathrm{d}\tau\right\|_2 \\
        &\qquad\qquad + \left\|\eta\cdot\int_0^1\bps(w_k)\nabla^2 L(w_k)\big(\bps(w_{k,\tau})-\bps(w_{k}) \big) m_{k+1}\mathrm{d}\tau\right\|_2  \\
        &\qquad \leq 2\eta\cdot \eta\gamma\kappa\cdot \gmax\cdot \|m_{k+1}\|_2 + \eta\cdot \frac{1}{2}\eta\gamma^2\kappa\cdot \|m_{k+1}\|_2\\
        &\qquad \leq \eta\gamma\kappa\cdot\big(2\eta\gmax + \eta\gamma/2\big)\cdot \big(\beta\cdot\|m_k\|_2+(1-\beta)\cdot\|\nabla L(w_k)\|_2\big),\label{eq: e k 2 1}
    \end{align}
    where the second inequality uses \eqref{eq: one step variation estimates} together with item 2 of Assumption~\ref{ass: regularity}.
    Also,
    \begin{align}
        \!\!\!\!\!\!\!\!\!\!\!\!\!\!\!\!\!\!\!\!\left\|\eta\cdot \int_0^1\nabla \bps(w_{k,\tau})[m_{k+1}]\nabla L(w_{k,\tau})\mathrm{d}\tau\right\|_2 \leq \eta\gamma\kappa\cdot \big(\beta\cdot\|m_k\|_2+(1-\beta)\cdot\|\nabla L(w_k)\|_2\big),\label{eq: e k 2 2}
    \end{align}
    where we apply Lemma~\ref{lem: river spinning}.
    Consequently, by \eqref{eq: b k inverse bound}, \eqref{eq: e k 1}, \eqref{eq: e k 2}, \eqref{eq: e k 2 1}, and \eqref{eq: e k 2 2}, we have
    \begin{align}
        \left\|\boldsymbol{P}_k^{\frac{1}{2}}\boldsymbol{O}^\top\boldsymbol{V}_k^\top\boldsymbol{E}_k\right\|_2
        &\leq \left(\frac{16}{(1-\beta)^3}+6\right)^{\frac{1}{2}}
        \eta\gamma\cdot\big(2+\eta\gmax+\eta\gamma/2\big) \notag\\
        &\qquad\cdot \kappa^{15/16}\cdot
        \big(\beta\|m_k\|_2+(1-\beta)\|\nabla L(w_k)\|_2\big)\\
        &\leq \mathsf{B}(\beta)\cdot \kappa^{15/16}\cdot \left\|\left(\begin{matrix}
        \beta \cdot m_k\\
        (1-\beta)\cdot \nabla L(w_{k})
    \end{matrix}
    \right)\right\|_2,\label{eq: step 3 conclusion}
    \end{align}
    where we define $\mathsf{B}(\beta)$ as 
    \begin{align}
        \mathsf{B}(\beta):=2\cdot \left(\frac{16}{(1-\beta)^3}+6\right)^{\frac{1}{2}} \cdot \left(\frac{10}{1-\beta} + 1\right)\cdot\frac{2}{1-\beta}.
    \end{align}

    \paragraph{Step 4: Concluding the proof.}
    Finally, combining \eqref{eq: first term sharp}, \eqref{eq: step 2 conclusion} and \eqref{eq: step 3 conclusion} , we can conclude that     
    \begin{align}
        &\left\|\boldsymbol{P}_k^{\frac{1}{2}}\boldsymbol{O}^\top\boldsymbol{V}_k^\top\left(\begin{matrix}
        \bps(w_{k+1})m_{k+1}\\
        \bps(w_{k+1})\nabla L(w_{k+1})
    \end{matrix}
    \right) \right\|_2 \\
    &\qquad \leq\big(1-\tfrac{1}{2}\overline{q}_k\big)\cdot  \big(1+\mathsf{A}(\beta)\cdot \kappa^{3/16}\big)\cdot\left\|\boldsymbol{P}_{k-1}^{\frac{1}{2}}\boldsymbol{O}^\top\boldsymbol{V}_{k-1}^\top\!\!\left(\begin{matrix}
        \bps(w_{k})m_{k}\\
        \bps(w_{k})\nabla L(w_{k}) 
    \end{matrix}
    \right)\right\|_2 \\
    &\qquad \qquad + 
    \mathsf{B}(\beta)\cdot \kappa^{15/16}\cdot \left\|\left(\begin{matrix}
        \beta\cdot m_k\\
        (1-\beta)\cdot \nabla L(w_{k})
    \end{matrix}
    \right)\right\|_2
    \end{align}
    This completes the proof of Lemma~\ref{lem: sharp direction}.
\end{proof}

\begin{lemma}[One-step variation of the flat-direction gradient]\label{lem: flat gradient one step change}
    Suppose that Assumptions~\ref{ass: existence} and \ref{ass: regularity} hold.
    Suppose that the momentum iteration \eqref{eq: gd momentum} starts from an initial point $w_0\in\cM$ on the river and an initial momentum $m_0=0$.
    Take the learning rate $\eta$ and momentum parameter $\beta$ satisfying $\eta\in \cI_1\cap \cI_2\cap \cI_3$ and $\beta\leq1-\varsigma(\beta)$.
    Suppose that the induction conditions \eqref{eq: induction 2}, \eqref{eq: induction 3}, and \eqref{eq: induction 4} hold up to step $k-1$ for some $k\geq 1$.
    Then
    \begin{align}
        \|\bpf(w_k)\nabla L(w_k)-\bpf(w_{k-1})\nabla L(w_{k-1})\|_2
        \leq \underline{q}_{k-1}\big(1+2\underline{q}_{k-1}\big)\cdot \|\bpf(w_k)\nabla L(w_k)\|_2 .
    \end{align}
\end{lemma}

\begin{proof}[Proof of Lemma~\ref{lem: flat gradient one step change}]
    Denote $G_\ell:=\bpf(w_\ell)\nabla L(w_\ell)$ for each $\ell$.
    For the boundary case $k=1$, we use the convention that $\sum_{j=0}^{-1}\beta^j=0$, $b_{-1}=0$, and $\underline{q}_{-1}=0$.
    We first bound the perturbation from $G_{k-1}$ to $G_k$ in terms of $G_{k-1}$.
    By Taylor's formula along the segment $\{w_{k-1,\tau}\}_{\tau\in[0,1]}$,
    \allowdisplaybreaks
    \begin{align}
        G_k-G_{k-1}
        &= -\eta\cdot \bpf(w_{k-1})\nabla^2 L(w_{k-1})\bpf(w_{k-1})m_k \nonumber\\
        &\qquad - \eta\cdot \int_0^1\Big(\bpf(w_{k-1,\tau})\nabla^2 L(w_{k-1,\tau})\bpf(w_{k-1,\tau}) \nonumber\\
        &\qquad\qquad- \bpf(w_{k-1})\nabla^2 L(w_{k-1})\bpf(w_{k-1})\Big)m_k\mathrm{d}\tau \nonumber\\
        &\qquad - \eta\cdot \int_0^1\nabla \bpf(w_{k-1,\tau})[m_k]\nabla L(w_{k-1,\tau})\mathrm{d}\tau .
    \end{align}
    The first term on the right-hand side is controlled by induction condition~\eqref{eq: induction 3} at step $k-1$:
    \begin{align}
        &\eta\cdot\|\bpf(w_{k-1})\nabla^2 L(w_{k-1})\bpf(w_{k-1})m_k\|_2 \nonumber\\
        &\qquad \leq \eta\gflat\cdot\Bigg((1-\beta)\sum_{j=0}^{k-1}\beta^j+b_{k-1}\Bigg)\cdot \|G_{k-1}\|_2 .\label{eq: one step flat gradient first term}
    \end{align}
    For the two integral terms, the same estimates as in \eqref{eq: e k 2 1} and Lemma~\ref{lem: river spinning}, with the index shifted from $k$ to $k-1$, give
    \begin{align}
        &\left\|\eta\cdot \int_0^1\Big(\bpf(w_{k-1,\tau})\nabla^2 L(w_{k-1,\tau})\bpf(w_{k-1,\tau}) - \bpf(w_{k-1})\nabla^2 L(w_{k-1})\bpf(w_{k-1})\Big)m_k\mathrm{d}\tau\right\|_2 \nonumber\\
        &\qquad\leq \eta\gamma\kappa\cdot\big(2\eta\gmax+\eta\gamma/2\big)\cdot\big(\beta\|m_{k-1}\|_2+(1-\beta)\|\nabla L(w_{k-1})\|_2\big),\label{eq: one step flat gradient second term}\\
        &\left\|\eta\cdot \int_0^1\nabla \bpf(w_{k-1,\tau})[m_k]\nabla L(w_{k-1,\tau})\mathrm{d}\tau\right\|_2 \nonumber\\
        &\qquad\leq \eta\gamma\kappa\cdot\big(\beta\|m_{k-1}\|_2+(1-\beta)\|\nabla L(w_{k-1})\|_2\big).\label{eq: one step flat gradient third term}
    \end{align}
    We next justify the full-vector bound used for the last two displays. For $k\geq2$, induction condition~\eqref{eq: induction 3} at step $k-2$ and Lemma~\ref{lem: auxiliary inequality 1} give
    \begin{align}
        \|\bpf(w_{k-2})m_{k-1}\|_2
        \leq \Bigg((1-\beta)\sum_{j=0}^{k-2}\beta^j+b_{k-2}\Bigg)\|G_{k-2}\|_2\leq R_{k-2}\|G_{k-1}\|_2,
    \end{align}
    where
    \begin{align}
        R_{k-2}:=\Bigg((1-\beta)\sum_{j=0}^{k-2}\beta^j+b_{k-2}\Bigg)\big(1+2\underline{q}_{k-2}\big).
    \end{align}
    For $k\geq3$, the momentum recursion and induction condition~\eqref{eq: induction 4} at step $k-3$ yield
    \begin{align}
        \|\bps(w_{k-2})m_{k-1}\|_2
        &\leq \beta\|\bps(w_{k-2})m_{k-2}\|_2
        +(1-\beta)\|\bps(w_{k-2})\nabla L(w_{k-2})\|_2\nonumber\\
        &\leq c\|G_{k-2}\|_2
        \leq c\big(1+2\underline q_{k-2}\big)\|G_{k-1}\|_2.
    \end{align}
    The same bound also holds for $k=2$ because $m_0=0$ and $w_0\in\cM$ imply that
    $\bps(w_0)m_1=(1-\beta)\bps(w_0)\nabla L(w_0)=0$. Consequently,
    \begin{align}
        \|m_{k-1}\|_2
        \leq\Big(R_{k-2}+c\big(1+2\underline q_{k-2}\big)\Big)\|G_{k-1}\|_2.
    \end{align}
    Induction condition~\eqref{eq: induction 4} at step $k-2$ also gives
    \begin{align}
        \|\nabla L(w_{k-1})\|_2\leq(1+c)\|G_{k-1}\|_2.
    \end{align}
    Hence, using \eqref{eq: c small bound}, $\underline q_{k-2}\leq1/2$, and $R_{k-2}\geq1-\beta$, we obtain
    \begin{align}
        &\beta\|m_{k-1}\|_2+(1-\beta)\|\nabla L(w_{k-1})\|_2 \nonumber\\
        &\qquad\leq\Big(\beta R_{k-2}+(1-\beta)+c\big(1+2\beta\underline q_{k-2}\big)\Big)\|G_{k-1}\|_2\nonumber\\
        &\qquad\leq2\big((1-\beta)+\beta R_{k-2}\big)\|G_{k-1}\|_2\nonumber\\
        &\qquad=2\Bigg((1-\beta)+\beta\Bigg((1-\beta)\sum_{j=0}^{k-2}\beta^j+b_{k-2}\Bigg)\big(1+2\underline{q}_{k-2}\big)\Bigg)\|G_{k-1}\|_2.\label{eq: one step flat gradient momentum bound}
    \end{align}
    Here the second inequality follows because
    \begin{align}
        c\big(1+2\beta\underline q_{k-2}\big)
        \leq(1-\beta)(1+\beta)
        \leq(1-\beta)+\beta R_{k-2}.
    \end{align}
    For the boundary case $k=1$, $m_0=0$ and $w_0\in\cM$ imply
    $\nabla L(w_0)=\bpf(w_0)\nabla L(w_0)=G_0$.
    Thus the left-hand side of \eqref{eq: one step flat gradient momentum bound} equals $(1-\beta)\|G_0\|_2$, so the same bound holds under the stated conventions.
    Combining \eqref{eq: one step flat gradient first term}, \eqref{eq: one step flat gradient second term}, \eqref{eq: one step flat gradient third term}, and \eqref{eq: one step flat gradient momentum bound}, we obtain
    \begin{align}
        \|G_k-G_{k-1}\|_2\leq \widehat{q}_{k-1}\cdot \|G_{k-1}\|_2,\label{eq: one step flat gradient preliminary}
    \end{align}
    where
    \begin{align}
        \widehat{q}_{k-1}
        &:=\eta\gflat\cdot \Bigg((1-\beta)\sum_{j=0}^{k-1}\beta^j+b_{k-1}\Bigg) \nonumber\\
        &\qquad +2\eta\gamma\kappa\cdot\big(1+2\eta\gmax+\eta\gamma/2\big)\cdot\Bigg((1-\beta)+\beta\Bigg((1-\beta)\sum_{j=0}^{k-2}\beta^j+b_{k-2}\Bigg)\big(1+2\underline{q}_{k-2}\big)\Bigg).
    \end{align}
    It remains to compare $\widehat{q}_{k-1}$ with $\underline{q}_{k-1}$.
    By item 2 of Assumption~\ref{ass: regularity} and Proposition~\ref{prop: useful properties}, $\kappa\gamma\leq \gflat$, $1+2\eta\gmax+\eta\gamma/2\leq\mathsf{C}(\beta)$, and, for $k\geq 2$, $\underline{q}_{k-2}\leq q(\beta)$; for $k=1$, this bound holds by the convention $\underline{q}_{-1}=0$.
    Moreover, \eqref{eq: b k weighted sum bound} gives $b_t\leq(1-\beta)\sum_{j=0}^t\beta^j$ for every $t\geq 0$.
    Therefore,
    \begin{align}
        &(1-\beta)+\beta\Bigg((1-\beta)\sum_{j=0}^{k-2}\beta^j+b_{k-2}\Bigg)\big(1+2\underline{q}_{k-2}\big) \nonumber\\
        &\qquad\leq 2(1-\beta)\sum_{j=0}^{k-1}\beta^j\big(1+2q(\beta)\big).\label{eq: one step flat gradient second factor}
    \end{align}
    Hence,
    \begin{align}
        \widehat{q}_{k-1}
        &\leq \eta\gflat\cdot\Bigg(2(1-\beta)\sum_{j=0}^{k-1}\beta^j + 2\mathsf{C}(\beta)\cdot2(1-\beta)\sum_{j=0}^{k-1}\beta^j\cdot\big(1+2q(\beta)\big)\Bigg) \nonumber\\
        &=2\eta\gflat\cdot\Big(1+2\mathsf{C}(\beta)\big(1+2q(\beta)\big)\Big)\cdot(1-\beta)\sum_{j=0}^{k-1}\beta^j \nonumber\\
        &=\underline{q}_{k-1}.\label{eq: one step flat gradient q comparison}
    \end{align}
    Combining \eqref{eq: one step flat gradient preliminary} and \eqref{eq: one step flat gradient q comparison} gives
    \begin{align}
        \|G_k-G_{k-1}\|_2\leq \underline{q}_{k-1}\cdot \|G_{k-1}\|_2 .
    \end{align}
    Finally, induction condition~\eqref{eq: induction 2} at step $k-1$ and Lemma~\ref{lem: auxiliary inequality 1} imply
    \begin{align}
        \|G_{k-1}\|_2\leq \big(1+2\underline{q}_{k-1}\big)\cdot\|G_k\|_2 .
    \end{align}
    The desired conclusion follows by substituting the last display into the previous one.
\end{proof}

\begin{lemma}[Induction argument 1]\label{lem: induction 1}
    Suppose that Assumptions~\ref{ass: existence} and \ref{ass: regularity} hold.
    Suppose that the momentum iteration \eqref{eq: gd momentum} starts from an initial point $w_0\in\cM$ on the river and an initial momentum $m_0=0$.
    Take the learning rate $\eta$ and momentum parameter $\beta$ satisfying $\eta\in \cI_1\cap \cI_2\cap \cI_3$.
    Also, suppose that the induction conditions \eqref{eq: induction 2}, \eqref{eq: induction 3}, and \eqref{eq: induction 4} hold up to step $k-1$ for some $k\geq 1$, and that \eqref{eq: induction 4} holds for the dummy step $k=-1$.
    Then for step $k$,
    \begin{align}
        &\left\|(1-\beta)\cdot\sum_{j=0}^{k}\beta^j\cdot \bpf(w_k)\nabla L(w_k) - \bpf(w_k)m_{k+1} \right\|_2\leq \widetilde{b}_k\cdot\|\bpf(w_k)\nabla L(w_k)\|_2,\quad\text{where} \\ 
        &\widetilde{b}_k:=\big(1+2\underline{q}_{k-1}\big)\cdot \Bigg(\beta b_{k-1}+ \beta(1-\beta) \sum_{j=0}^{k-1}\beta^j \underline{q}_{k-1} + \kappa\eta\gamma\beta\Bigg(1+(1-\beta)\sum_{j=0}^{k-1}\beta^j+b_{k-1}\Bigg)\Bigg).
    \end{align}
\end{lemma}

\begin{proof}[Proof of Lemma~\ref{lem: induction 1}]
By definition, we have the following decomposition of the left hand side,
\allowdisplaybreaks
    \begin{align}
        &\left\|(1-\beta)\cdot\sum_{j=0}^{k}\beta^j\cdot \bpf(w_k)\nabla L(w_k) - \bpf(w_k)m_{k+1} \right\|_2\\
        &\qquad = \left\|(1-\beta)\cdot\sum_{j=1}^{k}\beta^j\cdot \bpf(w_k)\nabla L(w_k) - \beta\cdot \bpf(w_k) m_k\right\|_2 \\
        &\qquad \leq  \left\|\beta\cdot \Bigg((1-\beta)\cdot\sum_{j=0}^{k-1}\beta^j\cdot \bpf(w_{k-1})\nabla L(w_{k-1}) - \bpf(w_{k-1})m_{k}\Bigg)\right\|_2 \\
        &\qquad\qquad + \left\|\beta(1-\beta)\cdot\sum_{j=0}^{k-1}\beta^j\cdot \Big(\bpf(w_k)\nabla L(w_k) - \bpf(w_{k-1})\nabla L(w_{k-1})\Big) \right\|_2\\
        &\qquad\qquad + \left\|\beta\cdot \Big(\bpf(w_{k-1}) - \bpf(w_k)\Big) m_k\right\|_2.
    \end{align}
    Now using induction condition~\eqref{eq: induction 3} at step $k-1$ and Lemma~\ref{lem: flat gradient one step change}, we can derive that
    \begin{align}
        &\left\|(1-\beta)\cdot\sum_{j=0}^{k}\beta^j\cdot \bpf(w_k)\nabla L(w_k) - \bpf(w_k)m_{k+1} \right\|_2\\
        &\qquad \leq  \Bigg(\beta\cdot b_{k-1}+ \beta(1-\beta)\cdot \sum_{j=0}^{k-1}\beta^j\cdot \underline{q}_{k-1}\Bigg)\cdot(1+2\underline{q}_{k-1})\cdot \|\bpf(w_k)\nabla L(w_k)\|_2 \\
        &\qquad\qquad + \beta\cdot \left\|\Big(\bpf(w_{k-1}) - \bpf(w_k)\Big) m_k\right\|_2,
    \end{align}
    where we have also applied Lemma~\ref{lem: auxiliary inequality 1}.
    Furthermore, the last term on the right hand side above is bounded by the following,
    \begin{align}
        \left\|\Big(\bpf(w_{k-1}) - \bpf(w_k)\Big) m_k\right\|_2 &\leq \kappa \eta\gamma\cdot \|m_k\|_2  \\
        &\leq \kappa \eta\gamma\cdot \|\bps(w_{k-1})m_k\|_2  + \kappa \eta\gamma\cdot \|\bpf(w_{k-1})m_k\|_2,
    \end{align}
    where the first inequality uses Lemma~\ref{lem: change sharp}.
    On the one hand, by the induction conditions \eqref{eq: induction 2} and \eqref{eq: induction 3} at step $k-1$, applying Lemma~\ref{lem: auxiliary inequality 1}, we have that 
    \begin{align}
        \|\bpf(w_{k-1})m_k\|_2  \leq \Bigg((1-\beta)\cdot\sum_{j=0}^{k-1}\beta^j+b_{k-1}\Bigg)\cdot (1+2\underline{q}_{k-1})\cdot \|\bpf(w_k)\nabla L(w_k)\|_2. 
    \end{align}
    On the other hand, by the induction condition \eqref{eq: induction 4} at step $k-2$ and  condition \eqref{eq: induction 2} at step $k-1$, 
    \begin{align}
         \|\bps(w_{k-1})m_k\|_2  &\leq \|\bps(w_{k-1})m_{k-1}\|_2 + \|\bps(w_{k-1})\nabla L(w_{k-1})\|_2 \\ &\leq 2c\cdot \|\bpf(w_{k-1}) \nabla L(w_{k-1})\|_2 \\
         &\leq (1+2\underline{q}_{k-1})\cdot \|\bpf(w_k) \nabla L(w_{k})\|_2 ,
    \end{align}
    where the last inequality applies Lemma~\ref{lem: auxiliary inequality 1} and \eqref{eq: c small bound}.
    Therefore, by combining all the above inequalities, we can conclude the proof of Lemma~\ref{lem: induction 1}.
\end{proof}

\begin{lemma}[Induction argument 2]\label{lem: induction 2}
    Suppose that Assumptions~\ref{ass: existence} and \ref{ass: regularity} hold.
    Also, suppose that the induction conditions \eqref{eq: induction 2}, \eqref{eq: induction 3}, and \eqref{eq: induction 4} hold up to step $k-1$, plus that the induction condition \eqref{eq: induction 3} hold for step $k$, where $k\geq 1$, then it holds that for such a step $k$ and any $\tau\in[0,1]$,
    \begin{align}
        &\|\bpf(w_{k,\tau})\nabla L(w_{k,\tau})
        -\bpf(w_k)\nabla L(w_k)\|_2
        \leq\widetilde{\underline q}_k\|\bpf(w_k)\nabla L(w_k)\|_2, \notag\\
        &\|\bpf(w_{k,\tau})\nabla L(w_{k,\tau})\|_2
        \geq \big(1-\widetilde{\underline{q}}_k\big)\cdot
        \|\bpf(w_{k})\nabla L (w_{k})\|_2,\quad\text{where}\\
        &\widetilde{\underline{q}}_k := \eta\gflat\cdot \Bigg((1-\beta)\sum_{j=0}^k\beta^j +b_k\Bigg)\\
        &\qquad\quad+ 2\eta\gamma\kappa \big(1+2\eta\gmax+\eta\gamma/2\big)\cdot\Bigg((1-\beta)+\beta \Bigg((1-\beta)\sum_{j=0}^{k-1}\beta^j+b_{k-1}\Bigg) (1+2\underline{q}_{k-1})\Bigg).
    \end{align}
    Moreover, the above conclusion also holds for step $k=0$ supposing that the induction condition \eqref{eq: induction 4} holds for step $k=-1$ and that the induction condition \eqref{eq: induction 3} holds for step $k=0$.
\end{lemma}

\begin{proof}[Proof of Lemma~\ref{lem: induction 2}]
    Suppose that the induction conditions \eqref{eq: induction 2}, \eqref{eq: induction 3}, and \eqref{eq: induction 4} hold up to step $k-1$, plus that the induction condition \eqref{eq: induction 3} holds for step $k$, where $k\geq 1$.
    For readability, we display the calculation at $\tau=1$. For a general $\tau\in[0,1]$, replace every integral over $[0,1]$ below by an integral over $[0,\tau]$ and multiply the corresponding nonintegral first-order term by $\tau$. Since all subsequent estimates use the triangle inequality and uniform bounds along the segment, and since $\tau\leq1$, the same upper bound $\widetilde{\underline q}_k$ follows.
    At $\tau=1$,
    \allowdisplaybreaks
    \begin{align}
        &\bpf(w_{k+1})\nabla L(w_{k+1}) \\
        &\qquad= \bpf(w_{k})\nabla L(w_{k}) -\eta\cdot\int_0^1\bpf(w_{k,\tau})\nabla^2 L(w_{k,\tau})\bpf(w_{k,\tau})m_{k+1}\mathrm{d}\tau \\
        &\qquad\qquad- \eta\cdot \int_0^1\nabla \bpf(w_{k,\tau})[m_{k+1}]\nabla L(w_{k,\tau})\mathrm{d}\tau \\
        &\qquad=\bpf(w_{k})\nabla L(w_{k}) - \eta\cdot \bpf(w_k)\nabla^2 L(w_k)\bpf(w_k)m_{k+1}  \\
        &\qquad\qquad -\eta\cdot\int_0^1\Big(\bpf(w_{k,\tau})\nabla^2 L(w_{k,\tau})\bpf(w_{k,\tau}) -  \bpf(w_k)\nabla^2 L(w_k)\bpf(w_k)\Big)m_{k+1}\mathrm{d}\tau \\
        & \qquad\qquad - \eta\cdot \int_0^1\nabla \bpf(w_{k,\tau})[m_{k+1}]\nabla L(w_{k,\tau})\mathrm{d}\tau \\
        &\qquad= \bpf(w_{k})\nabla L(w_{k}) - \eta\cdot(1-\beta)\cdot\sum_{j=0}^k\beta^j\cdot \bpf(w_k)\nabla^2 L(w_k)\bpf(w_k)\bpf(w_k)\nabla L (w_k)  \\
        &\qquad\qquad +\eta\cdot \bpf(w_k)\nabla^2 L(w_k)\bpf(w_k)\Bigg((1-\beta)\cdot\sum_{j=0}^k\beta^j\cdot\bpf(w_k)\nabla L (w_k) -\bpf(w_k)m_{k+1}\Bigg) \\
        &\qquad\qquad -\eta\cdot\int_0^1\Big(\bpf(w_{k,\tau})\nabla^2 L(w_{k,\tau})\bpf(w_{k,\tau}) -  \bpf(w_k)\nabla^2 L(w_k)\bpf(w_k)\Big)m_{k+1}\mathrm{d}\tau \\
        &\qquad \qquad - \eta\cdot \int_0^1\nabla \bpf(w_{k,\tau})[m_{k+1}]\nabla L(w_{k,\tau})\mathrm{d}\tau,
    \end{align}
    Therefore, we can lower bound $\|\bpf(w_{k+1})\nabla L(w_{k+1})\|_2$ as
    \begin{align}
        &\|\bpf(w_{k+1})\nabla L(w_{k+1})\|_2 \\
        &\qquad\geq \Bigg(1-\eta\gflat\cdot (1-\beta)\cdot\sum_{j=0}^k\beta^j - \eta\gflat\cdot b_k\Bigg)\cdot \|\bpf(w_{k})\nabla L(w_{k})\|_2  \\
        &\qquad\qquad - \left\|\eta\cdot\int_0^1\Big(\bpf(w_{k,\tau})\nabla^2 L(w_{k,\tau})\bpf(w_{k,\tau}) -  \bpf(w_k)\nabla^2 L(w_k)\bpf(w_k)\Big)m_{k+1}\mathrm{d}\tau\right\|_2 \\
        &\qquad\qquad -  \left\|\eta\cdot\int_0^1\nabla \bpf(w_{k,\tau})[m_{k+1}]\nabla L(w_{k,\tau})\mathrm{d}\tau\right\|_2,
    \end{align}
    where we have applied item 2 of Assumption~\ref{ass: regularity} and the induction condition \eqref{eq: induction 3} at step $k$.
    Moreover, for the last two terms on the right hand above, we have the following upper bounds,
    \allowdisplaybreaks
    \begin{align}
        &\left\|\eta\cdot\int_0^1\Big(\bpf(w_{k,\tau})\nabla^2 L(w_{k,\tau})\bpf(w_{k,\tau}) -  \bpf(w_k)\nabla^2 L(w_k)\bpf(w_k)\Big)m_{k+1}\mathrm{d}\tau\right\|_2 \label{eq: proof induction argument 1}\\
        &\qquad \leq \eta\gamma\kappa\cdot\big(2\eta\gmax + \eta\gamma/2\big)\cdot \big(\beta\cdot\|m_k\|_2+(1-\beta)\cdot\|\nabla L(w_k)\|_2\big)\\
        &\qquad \leq \eta\gamma\kappa\cdot\big(2\eta\gmax + \eta\gamma/2\big)\cdot \Big(\beta\cdot\big(\|\bpf(w_{k-1})m_k\|_2 + \|\bps(w_{k-1})m_k\|_2\big)\\
        &\qquad\qquad +(1-\beta)\cdot\big(\|\bpf(w_k)\nabla L(w_k)\|_2 + \|\bps(w_k)\nabla L(w_k)\|_2\big)\Big)\\
        &\qquad \leq \eta\gamma\kappa\cdot(1+c) \big(2\eta\gmax + \eta\gamma/2\big)\Big(\beta\cdot\|\bpf(w_{k-1})m_k\|_2 +(1-\beta)\cdot\|\bpf(w_k)\nabla L(w_k)\|_2 \Big)\\
        &\qquad \leq  2\eta\gamma\kappa\cdot\big(2\eta\gmax + \eta\gamma/2\big)\\
        &\qquad\qquad \cdot\Bigg((1-\beta)+\beta\cdot \Bigg((1-\beta)\cdot\sum_{j=0}^{k-1}\beta^j+b_{k-1}\Bigg)\cdot (1+2\underline{q}_{k-1})\Bigg)\cdot \|\bpf(w_k)\nabla L(w_k)\|_2,
    \end{align}
    where the first inequality follows from the same arguments as in \eqref{eq: e k 2 1}, the third inequality applies induction condition \eqref{eq: induction 4} at step $k-1$ and step $k-2$, and the last inequality uses induction conditions \eqref{eq: induction 2} and \eqref{eq: induction 3} at step $k-1$, together with \eqref{eq: c small bound}.
    Similarly, we have that
    \begin{align}
         &\left\|\eta\cdot\int_0^1\nabla \bpf(w_{k,\tau})[m_{k+1}]\nabla L(w_{k,\tau})\mathrm{d}\tau\right\|_2& \\
         &\qquad \leq \eta\gamma\kappa\cdot \Big(\beta\cdot\|m_k\|_2+(1-\beta)\cdot\|\nabla L(w_k)\|_2\Big)\\
        &\qquad \leq 2\eta\gamma\kappa\cdot\Bigg((1-\beta)+\beta\cdot \Bigg((1-\beta)\sum_{j=0}^{k-1}\beta^j+b_{k-1}\Bigg) (1+2\underline{q}_{k-1})\Bigg)\cdot \|\bpf(w_k)\nabla L(w_k)\|_2,
    \end{align}
    where the first inequality uses Lemma~\ref{lem: river spinning}, and the second inequality applies the same argument for deriving the above \eqref{eq: proof induction argument 1}.
    Consequently, for every $\tau\in[0,1]$, we can conclude that
    \begin{align}
        &\|\bpf(w_{k,\tau})\nabla L(w_{k,\tau})
        -\bpf(w_k)\nabla L(w_k)\|_2
        \leq\widetilde{\underline q}_k\|\bpf(w_k)\nabla L(w_k)\|_2, \notag\\
        &\|\bpf(w_{k,\tau})\nabla L(w_{k,\tau})\|_2
        \geq \big(1-\widetilde{\underline{q}}_k\big)\cdot
        \|\bpf(w_{k})\nabla L (w_{k})\|_2,\quad\text{where}\\
        &\widetilde{\underline{q}}_k := \eta\gflat\cdot \Bigg((1-\beta)\sum_{j=0}^k\beta^j +b_k\Bigg)\\
        &\qquad\quad+ 2\eta\gamma\kappa \big(1+2\eta\gmax+\eta\gamma/2\big)\cdot\Bigg((1-\beta)+\beta \Bigg((1-\beta)\sum_{j=0}^{k-1}\beta^j+b_{k-1}\Bigg) (1+2\underline{q}_{k-1})\Bigg).
    \end{align}
    Finally, for the case of step $k=0$ given the correctness of induction conditions \eqref{eq: induction 3} at step $k=0$ and \eqref{eq: induction 4} at step $k=-1$, the conclusion can be proved similarly by additionally using the fact that $m_0=0$.
    This completes the proof of Lemma~\ref{lem: induction 2}.
\end{proof}

\begin{lemma}[Induction argument 3]\label{lem: induction 3}
    Suppose that the momentum iteration \eqref{eq: gd momentum} starts from an initial point $w_0\in\cM$ on the river and an initial momentum $m_0=0$.
    Suppose that Assumptions~\ref{ass: existence} and \ref{ass: regularity} hold. 
    Also, suppose that the induction conditions \eqref{eq: induction 1}, \eqref{eq: induction 2}, and \eqref{eq: induction 3} hold up to step $k$, plus that the induction condition \eqref{eq: induction 4} hold up to step $k-1$, where $k\geq 0$.
    Take the learning rate $\eta$ and the momentum parameter $\beta$ satisfying $\eta\in \cI_1\cap \cI_2\cap \cI_3\cap \cI_4$ (see definitions in Theorem~\ref{thm: main formal}).
    Then it holds that for such a step $k$,
    \begin{align}
    &\left\|\boldsymbol{P}_k^{\frac{1}{2}}\boldsymbol{O}^\top\boldsymbol{V}_k^\top\left(\begin{matrix}
        \bps(w_{k+1})m_{k+1}\\
        \bps(w_{k+1})\nabla L(w_{k+1})
    \end{matrix}
    \right) \right\|_2 \leq  3\mathsf{B}(\beta)\cdot\kappa^{3/4}\cdot  \|\bpf(w_{k+1})\nabla L(w_{k+1})\|_2,
    \end{align}
    and consequently it holds that
    \begin{align}
         &\max\Big\{\|\bps(w_{k+1})\nabla L(w_{k+1})\|_2, \|\bps(w_{k+1})  m_{k+1}\|_2\Big\} \\
         &\qquad \leq 3\mathsf{B}(\beta)\cdot\kappa^{3/4}\cdot \|\bpf(w_{k+1})\nabla L(w_{k+1})\|_2.
    \end{align}
    Here the matrices $\{\boldsymbol{P}_k, \boldsymbol{O}, \boldsymbol{V}_k\}_{k\geq 0}\subset\RR^{2d\times 2d}$ are defined in Lemma~\ref{lem: sharp direction}.
\end{lemma}

\begin{proof}[Proof of Lemma~\ref{lem: induction 3}]
    By the induction conditions  \eqref{eq: induction 1} and \eqref{eq: induction 2} for step $k$, we have that
    \begin{align}
        &\left\|\boldsymbol{P}_k^{\frac{1}{2}}\boldsymbol{O}^\top\boldsymbol{V}_k^\top\left(\begin{matrix}
        \bps(w_{k+1})m_{k+1}\\
        \bps(w_{k+1})\nabla L(w_{k+1})
    \end{matrix}
    \right) \right\|_2 \\
    &\qquad \leq\big(1-\tfrac{1}{2}\overline{q}_k\big)\cdot  \big(1+\mathsf{A}(\beta)\cdot \kappa^{3/16}\big)\cdot\left\|\boldsymbol{P}_{k-1}^{\frac{1}{2}}\boldsymbol{O}^\top\boldsymbol{V}_{k-1}^\top\!\!\left(\begin{matrix}
        \bps(w_{k})m_{k}\\
        \bps(w_{k})\nabla L(w_{k}) 
    \end{matrix}
    \right)\right\|_2 \\
    &\qquad \qquad + 
    \mathsf{B}(\beta)\cdot\kappa^{15/16}\cdot \left\|\left(\begin{matrix}
        \beta\cdot m_k\\
        (1-\beta)\cdot \nabla L(w_{k})
    \end{matrix}
    \right)\right\|_2,\label{eq: proof induction 3 1}
    \end{align}
    and that 
    \begin{align}
        \|\bpf(w_{k+1})\nabla L(w_{k+1})\|_2 \geq \Big(1-2\eta\gflat \cdot \big(1+6\mathsf{C}(\beta)\big)\Big)\cdot \|\bpf(w_{k})\nabla L(w_{k})\|_2.\label{eq: proof induction 3 2}
    \end{align}
    Note that here we also use $\eta\in\cI_1\cap\cI_2\cap\cI_3$ to apply the following bounds (see Proposition~\ref{prop: useful properties}) to obtain \eqref{eq: proof induction 3 2} from \eqref{eq: induction 2},
    \begin{align}
        1+2\eta\gmax+\frac{\eta\gamma}{2}\leq \mathsf{C}(\beta),\quad \underline{q}_k \leq 2\eta\gflat \cdot\big(1+6\mathsf{C}(\beta)\big).
    \end{align}
    By combining \eqref{eq: proof induction 3 1} and \eqref{eq: proof induction 3 2}, we have that, for any coefficient $c\in[0,1]$,
    \begin{align}
        &\left\|\boldsymbol{P}_k^{\frac{1}{2}}\boldsymbol{O}^\top\boldsymbol{V}_k^\top\left(\begin{matrix}
        \bps(w_{k+1})m_{k+1}\\
        \bps(w_{k+1})\nabla L(w_{k+1})
    \end{matrix}
    \right) \right\|_2  - c\cdot \|\bpf(w_{k+1})\nabla L(w_{k+1})\|_2 \\
        &\qquad  \leq \big(1-\tfrac{1}{2}\overline{q}_k\big)\cdot  \big(1+\mathsf{A}(\beta)\cdot \kappa^{3/16}\big)\cdot\left\|\boldsymbol{P}_{k-1}^{\frac{1}{2}}\boldsymbol{O}^\top\boldsymbol{V}_{k-1}^\top\!\!\left(\begin{matrix}
        \bps(w_{k})m_{k}\\
        \bps(w_{k})\nabla L(w_{k}) 
    \end{matrix}
    \right)\right\|_2 \\
    &\qquad\qquad - c\cdot \Big(1-2\eta\gflat \cdot \big(1+6\mathsf{C}(\beta)\big)\Big)\cdot \|\bpf(w_{k})\nabla L(w_{k})\|_2\\
    &\qquad\qquad + \mathsf{B}(\beta)\cdot \kappa^{15/16}\cdot \left\|\left(\begin{matrix}
        \beta\cdot m_k\\
        (1-\beta)\cdot \nabla L(w_{k})
    \end{matrix}
    \right)\right\|_2.\label{eq: proof induction 3 3}
    \end{align}
    Notice that the last term above can be bounded by the following,
    \begin{align}
        \left\|\left(\begin{matrix}
        \beta\cdot m_k\\
        (1-\beta)\cdot \nabla L(w_{k})
    \end{matrix}
    \right)\right\|_2 \leq \left\|\left(\begin{matrix}
        \beta\cdot \bps(w_{k}) m_k\\
        (1-\beta) \bps(w_{k}) \nabla L(w_{k})
    \end{matrix}
    \right)\right\|_2 + \left\|\left(\begin{matrix}
        \beta\cdot  \bpf(w_{k}) m_k\\
        (1-\beta)\bpf(w_{k}) \nabla L(w_{k})
    \end{matrix}
    \right)\right\|_2.
    \end{align}
    On the one hand, for the first term above, we have that 
    \begin{align}
        \left\|\left(\begin{matrix}
        \beta\cdot \bps(w_{k}) m_k\\
        (1-\beta)\cdot \bps(w_{k}) \nabla L(w_{k})
    \end{matrix}
    \right)\right\|_2 &\leq  \left\|\boldsymbol{P}_{k-1}^{\frac{1}{2}}\boldsymbol{O}^\top\boldsymbol{V}_{k-1}^\top\left(\begin{matrix}
        \bps(w_{k}) m_k\\
        \bps(w_{k}) \nabla L(w_{k})
    \end{matrix}
    \right)\right\|_2.
    \end{align}
    On the other hand, for the second term, we have, 
    \begin{align}
        &\left\|\left(\begin{matrix}
        \beta\cdot \bpf(w_{k}) m_k\\
        (1-\beta)\cdot \bpf(w_{k}) \nabla L(w_{k})
    \end{matrix}
    \right)\right\|_2 \\
    &\qquad \leq \|\beta\cdot \bpf(w_{k}) m_k\|_2 + \|\bpf(w_{k}) \nabla L(w_{k})\|_2 \\
    & \qquad \leq \left\|\beta\cdot \bpf(w_{k}) m_k - (1-\beta)\cdot\sum_{j=1}^{k} \beta^j \cdot \bpf(w_k)\nabla L(w_k)\right\|_2 \\
    &\qquad \qquad+ \left\|(1-\beta)\cdot\sum_{j=1}^{k}\beta^j \cdot \bpf(w_k)\nabla L(w_k)\right\|_2 + \|\bpf(w_{k}) \nabla L(w_{k})\|_2 \\
    &\qquad \leq \big(b_k + 2\big)\cdot \|\bpf(w_{k}) \nabla L(w_{k})\|_2 \\
    &\qquad \leq 3\cdot \|\bpf(w_{k}) \nabla L(w_{k})\|_2,
    \end{align}
    where the third inequality uses induction condition~\eqref{eq: induction 2} at step $k$ and that $b_k\leq 1$ due to $\eta\in\cI_3$ (see \eqref{eq: eta condition 3} in Proposition~\ref{prop: useful properties}).
    Consequently, we obtain that
    \begin{align}
        &\left\|\left(\begin{matrix}
        \beta\cdot m_k\\
        (1-\beta)\cdot \nabla L(w_{k})
    \end{matrix}
    \right)\right\|_2  \leq  \left\|\boldsymbol{P}_{k-1}^{\frac{1}{2}}\boldsymbol{O}^\top\boldsymbol{V}_{k-1}^\top\left(\begin{matrix}
        \bps(w_{k}) m_k\\
        \bps(w_{k}) \nabla L(w_{k})
    \end{matrix}
    \right)\right\|_2 + 3\cdot \|\bpf(w_{k}) \nabla L(w_{k})\|_2.\label{eq: proof induction 3 3+}
    \end{align}
    Combining \eqref{eq: proof induction 3 3+} with \eqref{eq: proof induction 3 3}, we can conclude that 
    \begin{align}
        &\left\|\boldsymbol{P}_k^{\frac{1}{2}}\boldsymbol{O}^\top\boldsymbol{V}_k^\top\left(\begin{matrix}
        \bps(w_{k+1})m_{k+1}\\
        \bps(w_{k+1})\nabla L(w_{k+1})
    \end{matrix}
    \right) \right\|_2 - c\cdot \|\bpf(w_{k+1})\nabla L(w_{k+1})\|_2 \\
        &\qquad  \leq \bigg(\big(1-\tfrac{1}{2}\overline{q}_k\big)\cdot\Big(1+\mathsf{A}(\beta)\cdot\kappa^{3/16}\Big) + \mathsf{B}(\beta)\cdot \kappa^{15/16}\bigg)\cdot\left\|\boldsymbol{P}_{k-1}^{\frac{1}{2}}\boldsymbol{O}^\top\boldsymbol{V}_{k-1}^\top\!\!\left(\begin{matrix}
        \bps(w_{k})m_{k}\\
        \bps(w_{k})\nabla L(w_{k})  
    \end{matrix}
    \right)\right\|_2 \\
    &\qquad\qquad - \bigg(c\cdot \Big(1-2\eta\gflat \cdot \big(1+6\mathsf{C}(\beta)\big)\Big) - 3 \mathsf{B}(\beta)\cdot \kappa^{15/16}\bigg)\cdot \|\bpf(w_{k})\nabla L(w_{k})\|_2.
    \end{align}
    Set $c_{\mathrm{end}}:=3\mathsf{B}(\beta)\kappa^{3/4}$ and take $c=c_{\mathrm{end}}$ in the preceding inequality. The induction closes once the coefficient multiplying $c_{\mathrm{end}}$ has the following lower bound:
    \begin{align}
        \tfrac{1}{2}\overline{q}_k\cdot\Big(1+\mathsf{A}(\beta)\cdot\kappa^{3/16}\Big)
        -\mathsf{A}(\beta)\cdot\kappa^{3/16} -\mathsf{B}(\beta)\cdot\kappa^{15/16}
        -2\eta\gflat\cdot\big(1+6\mathsf{C}(\beta)\big) 
        \geq \kappa^{3/16}.\label{eq: lower bound}
    \end{align}
    Indeed, multiplying this bound by $c_{\mathrm{end}}$ gives
    \begin{align}
        c_{\mathrm{end}}\kappa^{3/16}
        =3\mathsf{B}(\beta)\kappa^{15/16},
    \end{align}
    so the right-hand side of the preceding inequality is at most the sharp-direction contraction factor times the induction residual at step $k-1$.
    Since the norm-contraction term in \eqref{eq: lower bound} is $\tfrac{1}{2}\overline{q}_k$, a sufficient condition for that inequality is
    \begin{align}
        \max_{i\in[d-1]} \mathrm{Tr}(\boldsymbol{P}_{k,i}) \leq \frac{1}{\mathsf{D}_1(\beta)\cdot \kappa^{3/16} + \mathsf{D}_2(\beta)\cdot \gflat\gmax^{-1}}:=\frac{1}{\varsigma}.\label{eq: inequality for I 5}
    \end{align}
    Indeed, \eqref{eq: inequality for I 5} implies $\tfrac{1}{2}\overline{q}_k\geq \varsigma/2$. Using the definitions below, $\eta\gflat<2(1+\beta)(1-\beta)^{-1}\gflat\gmax^{-1}$ from $\eta\in\cI_1$, and $\kappa\leq1$, the left-hand side of \eqref{eq: lower bound} is at least
    \begin{align}
        &\big(1+\mathsf{A}(\beta)+\mathsf{B}(\beta)\big)\kappa^{3/16}
        +4\big(1+6\mathsf{C}(\beta)\big)\frac{1+\beta}{1-\beta}\frac{\gflat}{\gmax}
        -\mathsf{A}(\beta)\kappa^{3/16} 
        -\mathsf{B}(\beta)\kappa^{15/16}
        -2\eta\gflat\big(1+6\mathsf{C}(\beta)\big) \notag\\
        &\qquad \geq \big(1+\mathsf{B}(\beta)\big)\kappa^{3/16}
        -\mathsf{B}(\beta)\kappa^{15/16}
        \geq \kappa^{3/16}.
    \end{align}
    Here we denote $\mathsf{D}_1(\beta)$ and $\mathsf{D}_2(\beta)$ respectively as 
    \begin{align}
        \mathsf{D}_1(\beta) &:= 2\big(1 + \mathsf{A}(\beta) + \mathsf{B}(\beta)\big), \\
        \mathsf{D}_2(\beta) &:= 8\cdot \big(1+6\mathsf{C}(\beta)\big) \cdot\frac{1+\beta}{1-\beta}.
    \end{align}
    Solving the above inequality \eqref{eq: inequality for I 5} is equivalent to the following inequality on $\eta$,
    \begin{align}
        \frac{(2\beta^3-2\beta^2+3\beta-1)\cdot (\eta\lambda_{k,i})^2 + 2(1-\beta^2)\cdot\eta\lambda_{k,i} + 2(1-\beta)^2(1+\beta)}{ \eta \lambda_{k,i}(1-\beta)^2\cdot(2(1+\beta)-(1-\beta) \eta \lambda_{k,i})} \leq \frac{1}{\varsigma},\forall i\in[d-1].
    \end{align}
    This actually can be simplified to the following $d-1$ quadratic inequalities on $\eta$, 
    \begin{align}
        \mathsf{P}(\eta\lambda_{k,i}) &= c_2 \cdot (\eta\lambda_{k,i})^2 + c_1 \cdot \eta\lambda_{k,i} + c_0, \quad \text{where } i\in[d-1],  \\ 
    c_2 &= a(1-\beta)^3 + \bigl(2\beta^3 - 2\beta^2 + 3\beta - 1\bigr), \\
c_1 &= 2(1+\beta)(1-\beta)\bigl(1 - a(1-\beta)\bigr), \\
c_0 &= 2(1+\beta)(1-\beta)^2 .
    \end{align}
    Given the condition on $\beta$ that $\beta\leq 1-\varsigma = 1 - \mathsf{D}_1(\beta)\cdot \kappa^{3/16} - \mathsf{D}_2(\beta)\cdot \gflat\gmax^{-1}$, solving the above inequalities gives 
    \begin{align}
\frac{(1-\beta)\left(
  \begin{aligned}
  &(1+\beta)(a(1-\beta)-1)\\[-2pt]
  &\quad-\sqrt{(1+\beta)\,R(\beta)}
  \end{aligned}
  \right)}
  {a(1-\beta)^3 + (2\beta^3 - 2\beta^2 + 3\beta - 1)}
  \leq \eta\lambda_{k,i} \leq \frac{(1-\beta)\left(
  \begin{aligned}
  &(1+\beta)(a(1-\beta)-1)\\[-2pt]
  &\quad+\sqrt{(1+\beta)\,R(\beta)}
  \end{aligned}
  \right)}
  {a(1-\beta)^3 + (2\beta^3 - 2\beta^2 + 3\beta - 1)},
    \end{align}
    where we let $a:=1/\varsigma$ and $R(\beta)$ is defined as 
    \begin{align}
        R(\beta) := (a^2+2a-4)\cdot \beta^3
-(a^2+4a-4)\cdot \beta^2 +(-a^2+6a-5)\cdot \beta+ (a^2-4a+3).    
    \end{align}
    This effectively requires that 
    \begin{align}
        \eta\lambda_{k, 1} &\leq \frac{(1-\beta)((1+\beta)(a(1-\beta)-1) + \sqrt{(1+\beta)\,R(\beta)})}
{a(1-\beta)^3 + (2\beta^3 - 2\beta^2 + 3\beta - 1)}, \quad \text{and}\\ 
\eta\lambda_{k,d-1} &\geq \frac{(1-\beta)((1+\beta)(a(1-\beta)-1) - \sqrt{(1+\beta)\,R(\beta)})}
{a(1-\beta)^3 + (2\beta^3 - 2\beta^2 + 3\beta - 1)}.
    \end{align}
    For the first inequality above, since $\lambda_{k,1}\leq \gmax$, we obtain that $\eta$ need to satisfy 
    \begin{align}
        \eta &\leq  \frac{(1-\beta)((1+\beta)(a(1-\beta)-1) + \sqrt{(1+\beta)\,R(\beta)})}
{a(1-\beta)^3 + (2\beta^3 - 2\beta^2 + 3\beta - 1)} \cdot\frac{1}{\gmax} \\
&=  \bigg(\underbrace{\frac{2\cdot (1+\beta)}{1-\beta}}_{\displaystyle{\text{main term}}} - \underbrace{\frac{5\beta^4 -2\beta^3 +6\beta^2 -2\beta + 1}{(1-\beta)^4}\cdot\varsigma - \mathcal{O}(\varsigma^2)}_{\displaystyle{\text{small terms caused by river spinning}}}\bigg)  \cdot \frac{1}{\gmax}.\label{eq: cI 4 upper}
    \end{align}
    For the second inequality above, since $\lambda_{k,d-1}\geq \gamma$, we obtain that $\eta$ need to satisfy 
    \begin{align}
        \eta &\geq  \frac{(1-\beta)((1+\beta)(a(1-\beta)-1) - \sqrt{(1+\beta)\,R(\beta)})}
{a(1-\beta)^3 + (2\beta^3 - 2\beta^2 + 3\beta - 1)} \cdot\frac{1}{\gamma}.
    \end{align}
    By Assumption~\ref{ass: regularity} that $\gamma/\gmax \geq \kappa^{1/32}$ and that $\gflat / \gmax \leq \kappa^{1/2}$, an upper bound of the right hand side of above inequality is
    \begin{align}
         &\frac{(1-\beta)((1+\beta)(a(1-\beta)-1) - \sqrt{(1+\beta)\,R(\beta)})}
{a(1-\beta)^3 + (2\beta^3 - 2\beta^2 + 3\beta - 1)} \cdot\frac{1}{\gamma} \\ 
 &\qquad \leq  \frac{(1-\beta)((1+\beta)(a(1-\beta)-1) - \sqrt{(1+\beta)\,R(\beta)})}
{a(1-\beta)^3 + (2\beta^3 - 2\beta^2 + 3\beta - 1)} \cdot\frac{\kappa^{-1/32}}{\gmax} \\ 
& \qquad = \left(1+ \frac{\beta^2+3}{2(1-\beta^2)} \cdot \varsigma + \mathcal{O}(\varsigma^2)\right)\cdot \varsigma\cdot \frac{\kappa^{-1/32}}{\gmax} \\ 
& \qquad =\left(1+ \frac{\beta^2+3}{2(1-\beta^2)} \cdot \varsigma + \mathcal{O}(\varsigma^2)\right)\cdot  \left(\mathsf{D}_1(\beta)\cdot\kappa^{5/32} + \mathsf{D}_2(\beta)\cdot \frac{\gflat}{\gmax}\cdot\frac{1}{\kappa^{1/32}} \right) \cdot \frac{1}{\gmax} \\ 
&\qquad \leq \left(\mathsf{D}_1(\beta)\cdot\kappa^{5/32} + \mathsf{D}_2(\beta)\cdot \kappa^{15/32} \right) \cdot \frac{2}{\gmax},
    \end{align}
    This is the last interval $\cI_4$. By $\eta\in\cI_4$, we have
    \begin{align}
        \left(\mathsf{D}_1(\beta)\cdot\kappa^{5/32}
        + \mathsf{D}_2(\beta)\cdot \kappa^{15/32}\right)
        \cdot \frac{2}{\gmax}
        &\leq \eta \notag\\
        &\leq \bigg(\frac{1+\beta}{1-\beta}
        -\frac{5\beta^4 -2\beta^3 +6\beta^2 -2\beta + 1}{2(1-\beta)^4}\cdot\varsigma
        - \mathcal{O}(\varsigma^2)\bigg)\cdot \frac{2}{\gmax}.
    \end{align}
    Thus, \eqref{eq: lower bound} holds. The induction hypothesis at step $k-1$ (including the dummy initial step when $k=0$) makes the corresponding residual nonpositive. Therefore,
    \begin{align}
        &\left\|\boldsymbol{P}_k^{\frac{1}{2}}\boldsymbol{O}^\top\boldsymbol{V}_k^\top\left(\begin{matrix}
        \bps(w_{k+1})m_{k+1}\\
        \bps(w_{k+1})\nabla L(w_{k+1})
    \end{matrix}
    \right) \right\|_2 \leq  3\mathsf{B}(\beta)\cdot\kappa^{3/4}\cdot  \|\bpf(w_{k+1})\nabla L(w_{k+1})\|_2.
    \end{align}
    This completes the proof of Lemma~\ref{lem: induction 3}.
\end{proof}

\begin{lemma}[Auxiliary inequality 1]\label{lem: auxiliary inequality 1}
    Suppose that induction condition \eqref{eq: induction 2} hold for some $\underline{q}_k\leq 1/2$ at step $k$, then it holds that 
    \begin{align}
        \|\bpf(w_{k})\nabla L(w_{k})\|_2 \leq (1+2\underline{q}_k)\cdot  \|\bpf(w_{k+1})\nabla L(w_{k+1})\|_2 .
    \end{align}
\end{lemma}

\begin{proof}[Proof of Lemma~\ref{lem: auxiliary inequality 1}]
This follows from the fact that $1/(1-x)\leq 1+2x$ for $0\leq x\leq 1/2$.
\end{proof}

\subsection{Remaining Proofs in the Outline (Section~\ref{sec: proof outline})}

\subsubsection{Proof of Lemma~\ref{lem: gd momentum tracks the river closely}}\label{subsec: proof gd momentum tracks the river closely}

We first give the formal version of Lemma~\ref{lem: gd momentum tracks the river closely} in the following.

\begin{lemma}[The trajectory tracks the river closely]\label{lem: gd momentum tracks the river closely formal}
    Under the conditions and assumptions of Theorem~\ref{thm: main formal}, taking the learning rate $\eta$ and momentum parameter $\beta\leq 0.99$ satisfying $\eta\in\cI(\beta)$ defined in Theorem~\ref{thm: main formal}, then for any iteration $k\in\NN$ and $\tau\in[0,1]$ it holds that: (i) $\cB(w_{k,\tau},2g_{\max}/\gamma)\subset\cU$ and thus $\Phi(w_{k,\tau})$ exists; (ii) it holds that for any step $k\in\NN$ and $\tau\in[0,1]$,
    \begin{align}
    \|w_{k,\tau}  - \Phi(w_{k,\tau} )\|_2 \leq \frac{12\mathsf{B}(\beta)\cdot\kappa^{3/4}\cdot  \|\bpf(w_{k,\tau} )\nabla L(w_{k,\tau} )\|_2}{\gamma + 2\gamma_{\mathrm{flat}}}.
\end{align}
\end{lemma}

\begin{proof}[Proof of Lemma~\ref{lem: gd momentum tracks the river closely formal}]
    We prove Lemma~\ref{lem: gd momentum tracks the river closely formal} here by the following two steps.

\paragraph{Step 1: The GD trajectory has a well-defined projection onto the river.}
We first prove that the interpolated GD trajectory $\{w_{\lfloor t\rfloor, t - \lfloor t\rfloor}\}_{t\geq 0}$ can be projected onto the river $\cM$ by induction.
    According to Lemma~\ref{lem: existence of projection onto river and further properties}, it suffices to check that any point $w\in\{w_{\lfloor t\rfloor, t - \lfloor t\rfloor}\}_{t\geq 0}$ satisfies $\cB(w, 2g_{\max}/\gamma)\subset \cU$.

    For the initial point $w_0\in\cM$, it holds directly due to item 1 of Assumption~\ref{ass: regularity}.
    Then assume that $\{w_{\lfloor t\rfloor, t - \lfloor t\rfloor}\}_{0\leq t\leq k}$ for some $k\in\mathbb{N}$ satisfies the desired property, which validates the projection of $w_k$ onto $\cM$.
    Under this induction hypothesis, $w_j\in\cU$ for all $j\leq k$, and Lemma~\ref{lem: uniform momentum bound} gives $\|m_{k+1}\|_2\leq g_{\max}$.
    For step $k+1$, we have that
\begin{align}
    \|w_{k+1} - \Phi(w_k)\|
    &\leq \|w_{k+1} - w_k\|_2 + \|w_k - \Phi(w_k)\|_2 \notag\\
    &= \eta\cdot\|m_{k+1}\|_2 + \|w_k - \Phi(w_k)\|_2 \\
    &\leq \frac{1+\beta}{1-\beta}\cdot\frac{2}{\gmax}\cdot g_{\max}
    + \frac{2 g_{\max}}{\gamma} \notag\\
    &\leq \frac{400g_{\max}}{\gmax} + \frac{2 g_{\max}}{\gamma}
    \leq \frac{4g_{\max}}{\gamma}
\end{align}
Here the first displayed upper bound uses $\eta\in\cI_1$, the uniform momentum bound, and item 2 of Lemma~\ref{lem: existence of projection onto river and further properties}.
The next inequality uses $\beta\leq0.99$, while the final inequality uses the scale separation $\gamma\leq\gmax/200$ in item 2 of Assumption~\ref{ass: regularity}.
Moreover, since $w_{k,\tau}=(1-\tau)w_k+\tau w_{k+1}$, the induction hypothesis and the preceding bound give
\begin{align}
    \|w_{k,\tau}-\Phi(w_k)\|_2
    \leq (1-\tau)\|w_k-\Phi(w_k)\|_2
    +\tau\|w_{k+1}-\Phi(w_k)\|_2
    \leq \frac{4g_{\max}}{\gamma}.
\end{align}
Consequently, $\cB(w_{k,\tau},2g_{\max}/\gamma)$ is contained in $\cB(\Phi(w_k),6g_{\max}/\gamma)\subset\cU$ for every $\tau\in[0,1]$, where the last inclusion follows from item 1 of Assumption~\ref{ass: regularity} and $\Phi(w_k)\in\cM$.
This extends the desired property from $t\in[0,k]$ to $t\in[0,k+1]$.


\paragraph{Step 2: Control the distance between $w_k$ and $\Phi(w_k)$.}
By item 2 of Lemma~\ref{lem: existence of projection onto river and further properties}, we can obtain that for any $k\in\mathbb{N}$ and $\tau\in[0,1]$,
\begin{align}
    \|w_{k,\tau} - \Phi(w_{k,\tau} )\|_2 \leq \frac{2\|\bps(w_{k,\tau} )\nabla L(w_{k,\tau} )\|_2}{\gamma + 2\gamma_{\mathrm{flat}}}.\label{eq: gd tracks the river closely 1}
\end{align}
This shows that the distance between $w_{k,\tau} $ and its projection $\Phi(w_{k,\tau} )$ is bounded by the norm of the gradient in the sharp direction at $w_{k,\tau} $.
With this, we now further invoke Lemma~\ref{lem: flat dominance momentum} to show that the gradient is almost fully living in the flat direction, meaning that the norm of the gradient in the sharp direction is dominated by its norm in the flat direction.
More specifically, with $\eta\in\cI(\beta)$, it holds that 
\begin{align}
    \|\bps(w_{k,\tau} )\nabla L(w_{k,\tau} )\|_2 \leq 6\mathsf{B}(\beta)\cdot\kappa^{3/4}\cdot  \|\bpf(w_{k,\tau} )\nabla L(w_{k,\tau} )\|_2,\label{eq: gd tracks the river closely 2}
\end{align}
for any $k\in\mathbb{N}$.
With \eqref{eq: gd tracks the river closely 1} and \eqref{eq: gd tracks the river closely 2}, we finally obtain that 
\begin{align}
    \|w_{k,\tau}  - \Phi(w_{k,\tau} )\|_2 \leq \frac{12\mathsf{B}(\beta)\cdot\kappa^{3/4}\cdot \|\bpf(w_{k,\tau} )\nabla L(w_{k,\tau} )\|_2}{\gamma + 2\gamma_{\mathrm{flat}}},
\end{align}
for any $k\in\mathbb{N}$. 
Combing Step 1 and Step 2, we can conclude the proof of Lemma~\ref{lem: gd momentum tracks the river closely}.
\end{proof} 

\subsubsection{Proof of Lemma~\ref{lem: time index grows almost linearly with learning rate}}
\label{subsec: proof lem time index grows almost linearly with learning rate}

We first give the formal version of Lemma~\ref{lem: time index grows almost linearly with learning rate} in the following.

\begin{lemma}[Time index grows almost linearly with learning rate]\label{lem: time index grows almost linearly with learning rate formal}
    Under the conditions and assumptions of Theorem~\ref{thm: main formal}, taking the learning rate $\eta$ and the momentum parameter $\beta\leq 0.99$ satisfying $\eta\in\cI(\beta)$ defined in Theorem~\ref{thm: main formal}, then for any iteration $k\geq \log(\beta\eta\gflat/(1-\beta))/\log\beta$ and $\tau\in(0,1)$, by letting $t=\tau + k$, the derivative of $S(t)$ exists and satisfies $|S'(t) - \eta|\leq \epsilon(\beta)\cdot\eta$.
    Equivalently, this bound holds almost everywhere on each interpolation interval, with the integer breakpoints understood through one-sided derivatives.
    Here  $\epsilon(\beta)$ is defined as
    \begin{align}
        \epsilon(\beta):=9\kappa + \Bigg(6\big(1+6\mathsf{C}(\beta)\big)+ 4 + \frac{100\beta\mathsf{C}(\beta)}{1-\beta}\Bigg)\cdot\eta\gflat + o(\kappa + \eta\gflat).
    \end{align}
\end{lemma}

\begin{proof}[Proof of Lemma~\ref{lem: time index grows almost linearly with learning rate formal}]
    
We prove the lemma by the following steps.

\paragraph{Step 1: Differentiate two sides of \eqref{eq: x S t phi}.}
Except $t\in\mathbb{N}$, $\Phi(w_{\lfloor t\rfloor, t - \lfloor t\rfloor})$ is a continuously differentiable function of $t$.
As for the $t\in\mathbb{N}$, $\Phi(w_{\lfloor t\rfloor, t - \lfloor t\rfloor})$ is continuous.
Then by the implicit function theorem, we know that $S(t)$ is also a continuously differentiable function of $t$ whenever $t\notin\mathbb{N}$. 
Therefore, for any $t=k+\tau$ with $k\in\mathbb{N}$ and $\tau\in(0,1)$, we can differentiate Equation~\eqref{eq: x S t phi} and apply the chain rule to obtain that 
\begin{align}
    x'(S(t))S'(t) = \nabla \Phi(w_{k,\tau}) \frac{\mathrm{d}}{\mathrm{d}t}w_{\lfloor t\rfloor, t - \lfloor t\rfloor} = -\eta\cdot \nabla \Phi(w_{k,\tau})m_{k+1},\label{eq: time index grows almost linearly with learning rate 1}
\end{align}
where the last equality is due to \eqref{eq: interpolation}, and $\nabla \Phi(\cdot)\in\mathbb{R}^{d\times d}$ is the Jacobian matrix of $\Phi(\cdot):\mathbb{R}^d\mapsto\mathbb{R}^d$.
With \eqref{eq: time index grows almost linearly with learning rate 1}, in order to prove Lemma~\ref{lem: time index grows almost linearly with learning rate formal}, which controls the magnitude of $S'(t)$, it suffices to compare the two vectors $x'(S(t))$ and $-\nabla \Phi(w_{k,\tau})m_{k+1}$.
For simplicity, we denote 
\begin{align}
    u(t):=x'(S(t)),
\end{align}
which is the tangent vector of the manifold under the parametrization \eqref{eq: manifold parametrization} $\{x(t)\}_{t\geq 0}$ at reference-flow time $S(t)$.

\paragraph{Step 2: Compare the two vectors $u(t)$ and $-\nabla \Phi(w_{k,\tau}) m_{k+1}$.}
Consider the decomposition:
\begin{align}
    &\|u(t)-(-\nabla \Phi(w_{k,\tau})m_{k+1})\|_2  \\
    &\qquad \leq \underbrace{\|u(t) - (-\bpf(w_{k,\tau})m_{k+1})\|_2}_{\displaystyle{\text{Term (i)}}} + \underbrace{\|\nabla \Phi(w_{k,\tau})m_{k+1}-\bpf(w_{k,\tau})m_{k+1}\|_2}_{\displaystyle{\text{Term (ii)}}}.\label{eq: time index grows almost linearly with learning rate 2}
\end{align}
The approach is to show that along the \eqref{eq: gd momentum} trajectory the right hand side of \eqref{eq: time index grows almost linearly with learning rate 2} is bounded by $\cO(\kappa\cdot\|u(t)\|_2)$, which means that the compared two vectors are actually quite similar.
We approach Term (i) and Term (ii) in \eqref{eq: time index grows almost linearly with learning rate 2} respectively in the following.

\paragraph{Step 2.1: Bound Term (i) in \eqref{eq: time index grows almost linearly with learning rate 2}.}
Bounding this term is nearly about to say that the momentum in the flat direction is similar to the tangent vector of the projection of $w_{k,\tau}$ on the river.
To put it formal, consider the following decomposition,
\begin{align}
    \|u(t)-(-\bpf(w_{k,\tau}) m_{k+1})\|_2 &\leq \|u(t)-(- \bpf(w_{k,\tau})\nabla L(w_{k,\tau}))\|_2 \\
    &\qquad + \| \bpf(w_{k,\tau})\nabla L(w_{k,\tau})- \bpf(w_k)\nabla L(w_k)\|_2\\
    &\qquad + \|\bpf(w_k)\nabla L(w_k) - \bpf(w_{k, \tau})m_{k+1}\|_2.\label{eq: time index grows almost linearly with learning rate 3}
\end{align}
The first term in \eqref{eq: time index grows almost linearly with learning rate 3} measures how close the flat-direction gradient at $w_{k,\tau}$ is to the tangent vector of the river at $\Phi(w_{k,\tau})$; recall that $u(t)=x'(S(t))$ and $x(S(t))=\Phi(w_{k,\tau})$.
The second term in \eqref{eq: time index grows almost linearly with learning rate 3} is the difference of the gradient in the flat direction between time $k$ and $k+\tau$.
The last term in \eqref{eq: time index grows almost linearly with learning rate 3} characterizes the difference between the momentum projected onto the flat direction and the gradient in the flat direction.
We handle these terms respectively in the following. 

For the first term, we invoke Lemma~\ref{lem: tangent vector flat direction} with $\mathsf{F}(\beta):=12\mathsf{B}(\beta)$ and $\iota:=\kappa^{3/4}$. Define
\begin{align}
    \rho_\beta(\kappa)
    :=12\mathsf{B}(\beta)\kappa^{3/4}
    \frac{\gflat+\gamma\kappa}{\gamma+2\gflat}.
\end{align}
By item 2 of Assumption~\ref{ass: regularity},
\begin{align}
    \rho_\beta(\kappa)
    &\leq12\mathsf{B}(\beta)\left(\kappa^{39/32}+\kappa^{7/4}\right).
\end{align}
Moreover, the condition on $\beta$ in Theorem~\ref{thm: main formal} gives $\mathsf{B}(\beta)\kappa^{3/16}\leq1-\beta\leq1$. Hence
\begin{align}
    \rho_\beta(\kappa)
    \leq12\left(\kappa^{33/32}+\kappa^{25/16}\right)
    =o(\kappa).
\end{align}
In particular, $\rho_\beta(\kappa)<1$ in the sufficiently small $\kappa$ regime. Let
\begin{align}
    \zeta_\beta(\kappa)
    :=\frac{\rho_\beta(\kappa)(1+4\kappa)}{1-\rho_\beta(\kappa)}
    =o(\kappa).
\end{align}
The conclusion of Lemma~\ref{lem: gd momentum tracks the river closely formal} now verifies the distance hypothesis of Lemma~\ref{lem: tangent vector flat direction}, which yields
\begin{align}
    \|u(t) - (-\bpf(w_{k,\tau})\nabla L(w_{k,\tau}))\|_2
    \leq \big(4\kappa+\zeta_\beta(\kappa)\big)\|u(t)\|_2.\label{eq: time index grows almost linearly with learning rate 4}
\end{align}
For the second term in \eqref{eq: time index grows almost linearly with learning rate 3}, \eqref{eq: interpolated flat gradient variation}, \eqref{eq: interpolation flat endpoint comparison 1}, and Proposition~\ref{prop: useful properties} show that
\begin{align}
     \| \bpf(w_{k,\tau})\nabla L(w_{k,\tau})- \bpf(w_k)\nabla L(w_k)\|_2 \leq 6\eta\gflat\cdot \big(1+6\mathsf{C}(\beta)\big)\cdot  \| \bpf(w_{k,\tau})\nabla L(w_{k,\tau})\|_2,
\end{align}
which, combined with \eqref{eq: time index grows almost linearly with learning rate 4}, further gives that 
\begin{align}
     &\| \bpf(w_{k,\tau})\nabla L(w_{k,\tau})- \bpf(w_k)\nabla L(w_k)\|_2  \\
     &\qquad \leq 6\eta\gflat\cdot \big(1+6\mathsf{C}(\beta)\big)\cdot \big(1+4\kappa+\zeta_\beta(\kappa)\big)\cdot \|u(t)\|_2.\label{eq: time index grows almost linearly with learning rate 5}
\end{align}
Finally, for the last term in \eqref{eq: time index grows almost linearly with learning rate 3}, we have the following upper bound, 
\begin{align}
    & \|\bpf(w_k)\nabla L(w_k) - \bpf(w_{k, \tau})m_{k+1}\|_2 \\
    &\qquad \leq \|\bpf(w_k)\nabla L(w_k) - \bpf(w_{k})m_{k+1}\|_2 +\|\bpf(w_k) m_{k+1} - \bpf(w_{k,\tau})m_{k+1}\|_2.\label{eq: p f nabla l momentum bound}
\end{align}
For the first term on the right hand side above, we have that 
\begin{align}
    &\|\bpf(w_k)\nabla L(w_k) - \bpf(w_{k})m_{k+1}\|_2 \\
    &\qquad \leq \Bigg\|(1-\beta)\cdot\sum_{j=0}^{k}\beta^j\cdot \bpf(w_k)\nabla L(w_k) - \bpf(w_k)m_{k+1}\Bigg\|_2 + \beta^{k+1}\cdot \|\bpf(w_k)\nabla L (w_k)\|_2 \\
    &\qquad \leq \left(\frac{24\beta\mathsf{C}(\beta)}{1-\beta}\cdot  \eta\gflat  + \beta^{k+1}\right)\cdot \|\bpf(w_k)\nabla L(w_k)\|_2 \\
    &\qquad \leq \frac{25\beta\mathsf{C}(\beta)}{1-\beta}\cdot \eta\gflat\cdot \|\bpf(w_k)\nabla L(w_k)\|_2,\label{eq: bpf nabla l momentum}
\end{align}
where the second inequality uses the second conclusion of Lemma~\ref{lem: flat dominance momentum} and \eqref{eq: b k refined bound}, and the last inequality uses $k\geq \log(\beta\eta\gflat / (1-\beta)) / \log\beta$.
For the second term on the right hand side of \eqref{eq: p f nabla l momentum bound}, we have
\begin{align}
    \|\bpf(w_k)m_{k+1} - \bpf(w_{k,\tau})m_{k+1}\|_2 &\leq \eta\gamma\kappa\cdot \|m_{k+1}\|_2 \label{eq: bpf interpolation momentum}\\
    &\leq \eta\gamma\kappa \cdot\|\bpf(w_k)m_{k+1}\|_2 + \eta\gamma\kappa\cdot \|\bpf(w_k)\nabla L(w_k)\|_2 \\
    &\leq \eta\gamma\kappa\cdot \left(2 + \frac{25\beta\mathsf{C}(\beta)}{1-\beta}\cdot \eta\gflat\right)\cdot  \|\bpf(w_k)\nabla L(w_k)\|_2 \\
    &\leq  \eta\gflat\cdot \left(2 + \frac{25\beta\mathsf{C}(\beta)}{1-\beta}\cdot \eta\gflat\right)\cdot  \|\bpf(w_k)\nabla L(w_k)\|_2,
\end{align}
where the first inequality applies Lemma~\ref{lem: change sharp}, the second inequality applies the last conclusion of Lemma~\ref{lem: flat dominance momentum}, the third inequality uses \eqref{eq: bpf nabla l momentum}, and the last inequality uses item 2 of Assumption~\ref{ass: regularity}.
With \eqref{eq: bpf nabla l momentum} and \eqref{eq: bpf interpolation momentum}, we can upper bound the last term in \eqref{eq: time index grows almost linearly with learning rate 3} as 
\allowdisplaybreaks
\begin{align}
&\|\bpf(w_k)\nabla L(w_k) - \bpf(w_{k, \tau})m_{k+1}\|_2 \\&\qquad \leq \left(2 + \frac{50\beta\mathsf{C}(\beta)}{1-\beta}\right)\cdot \eta\gflat\cdot  \|\bpf(w_k)\nabla L(w_k)\|_2\\
&\qquad \leq  \left(2 + \frac{50\beta\mathsf{C}(\beta)}{1-\beta}\right)\cdot \eta\gflat\\
&\qquad\qquad \cdot\Big(1+6\eta\gflat \big(1+6\mathsf{C}(\beta)\big)\Big)\cdot \big(1+4\kappa+\zeta_\beta(\kappa)\big)\cdot \|u(t)\|_2\\
&\qquad \leq \left(4 + \frac{100\beta\mathsf{C}(\beta)}{1-\beta} \right)\cdot \eta\gflat\cdot  \big(1+4\kappa+\zeta_\beta(\kappa)\big)\cdot \|u(t)\|_2.\label{eq: time index grows almost linearly with learning rate 6}
\end{align}
Here the first inequality is by \eqref{eq: bpf nabla l momentum} and \eqref{eq: bpf interpolation momentum}, the second inequality is by \eqref{eq: time index grows almost linearly with learning rate 4} and \eqref{eq: time index grows almost linearly with learning rate 5}, and the last inequality is by $\eta\in\cI_2$ (see the second conclusion of Proposition~\ref{prop: useful properties}).
Consequently, combining \eqref{eq: time index grows almost linearly with learning rate 4}, \eqref{eq: time index grows almost linearly with learning rate 5}, and \eqref{eq: time index grows almost linearly with learning rate 6}, we conclude that the right hand side of \eqref{eq: time index grows almost linearly with learning rate 3} can be upper bounded by 
\begin{align}
    &\|u(t)-(-\bpf(w_{k,\tau}) m_{k+1})\|_2 \\
    &\qquad \leq \left(4\kappa + \Bigg(6\big(1+6\mathsf{C}(\beta)\big)+ 4 + \frac{100\beta\mathsf{C}(\beta)}{1-\beta}\Bigg)\cdot\eta\gflat + o(\kappa + \eta\gflat)\right)\cdot \|u(t)\|_2.\label{eq: conclusion of step 2.1 lemma linear growth}
\end{align}
Here $o(\cdot)$ also absorbs $\zeta_\beta(\kappa)=o(\kappa)$ and terms with multiplicative factor $\eta\gflat\kappa$ or higher orders.

\paragraph{Step 2.2: Bound Term (ii) in \eqref{eq: time index grows almost linearly with learning rate 2}.}
Bounding this term means to prove that the Jacobian of the projection onto the river is similar to the projection onto the flat direction. 
This is shown through Lemma~\ref{lem: change of projection to the river}, by which we have that 
\begin{align}
    \|\nabla \Phi(w_{k,\tau})m_{k+1}-\bpf(w_{k,\tau})m_{k+1}\|_2 &\leq 5\kappa\cdot \|\bpf(w_{k,\tau})m_{k+1}\|_2\\
    &\leq  5\kappa\cdot \big(\|\bpf(w_{k,\tau})m_{k+1} - (-u(t))\|_2 +  \|-u(t)\|_2\big)\\
    &\leq \Big(5\kappa + o(\kappa +\eta\gflat)\Big)\cdot\|u(t)\|_2,
\end{align}
where the first inequality is by Lemma~\ref{lem: change of projection to the river} and the last inequality is by Step 2.1, i.e., \eqref{eq: conclusion of step 2.1 lemma linear growth}.

\paragraph{Step 3: Summary up.}
In conclusion, combining Step 1 and Step 2, we can arrive at 
\begin{align}
\left|S'(t) - \eta\right| \leq \epsilon(\beta)\cdot\eta, 
\end{align}
where the coefficient $\epsilon(\beta)$ is defined as following,
\begin{align}
    \epsilon(\beta):=9\kappa + \Bigg(6\big(1+6\mathsf{C}(\beta)\big)+ 4 + \frac{100\beta\mathsf{C}(\beta)}{1-\beta}\Bigg)\cdot\eta\gflat + o(\kappa + \eta\gflat).
\end{align}
This completes the proof of Lemma~\ref{lem: time index grows almost linearly with learning rate}.
\end{proof}

\newpage 

\section{Technical Lemmas}

\subsection{Analysis for Eigen-Decomposition}

In this section, we study the eigen-decomposition of matrix in the form of  
\begin{align}
    \boldsymbol{T} := \left(\begin{matrix}
         \beta & 1-\beta \\
         -\beta\cdot \eta\lambda & 1 - (1-\beta)\cdot \eta\lambda  
    \end{matrix}\right)\in \mathbb{R}^{2\times 2},\label{eq: etails of eigen decomposition A}
\end{align}
where $\eta, \lambda\geq 0$ and $\beta\in(0,1)$. 
The characteristic polynomial $\psi_{\boldsymbol{T}}(\alpha)$ of $\boldsymbol{T}$ is given by 
\begin{align}
    \psi_{\boldsymbol{T}}(\alpha) = \det(\alpha\cdot  \boldsymbol{I}_2 - \boldsymbol{T}) = \alpha^2 - \big((1+\beta)-(1-\beta)\cdot \eta\lambda\big)\cdot \alpha + \beta. \label{eq: characteristic polynomial of T}
\end{align}
The discriminant of $\psi_{\boldsymbol{T}}$ is then given by 
\begin{align}
    \Delta_{\boldsymbol{T}} = (1-\beta)^2\cdot (1-\eta\lambda)^2 - 4\beta(1-\beta)\cdot \eta\lambda = (1-\beta)[(\eta\lambda -1)^2 - \beta(1+\eta\lambda)^2].\label{eq: details of eigen decomposition determinant}
\end{align}
Therefore, the two eigenvalues of $\boldsymbol{T}$ is calculated as 
\begin{align}
    \alpha_{\boldsymbol{T}, \pm}=\frac{(1+\beta)-(1-\beta)\cdot \eta\lambda \pm \sqrt{(1-\beta)^2\cdot(1-\eta\lambda)^2 - 4 \beta(1-\beta)\cdot \eta\lambda}}{2}\in\mathbb{C}
\end{align}











\begin{lemma}[Schur unit-disk stability test for a quadratic]\label{lem: Schur unit-disk stability test for a quadratic}
    A real-coefficient quadratic
\begin{align}
    z^2+a_1 z+a_0
\end{align}
has both roots strictly inside the unit disk $(|z|<1)$ \textbf{iff} the Jury (Schur) inequalities hold:
\begin{align}
 \left|a_0\right|<1, \quad 1+a_1+a_0>0, \quad 1-a_1+a_0>0 .   
\end{align}
\end{lemma}

\begin{proposition}[Spectral radius of matrix $\boldsymbol{T}$]\label{prop: eigen values}
    Regarding the matrix $\boldsymbol{T}$, the following holds: 
    \begin{align}
        \eta \lambda<\frac{2(1+\beta)}{1-\beta}\quad  \Longrightarrow \quad \rho(\boldsymbol{T})<1.
    \end{align}
    Moreover by the discrete-time Lyapunov stability criterion:
    \begin{align}
        \rho(\boldsymbol{T})<1 \quad \Longleftrightarrow \quad \exists \boldsymbol{P} \succ  0 \text \,\,{ s.t. }\,\,\boldsymbol{T}^{\top} \boldsymbol{P} \boldsymbol{T}-\boldsymbol{P} = -\boldsymbol{I}\prec 0.
    \end{align}
\end{proposition}

\begin{proof}[Proof of Proposition~\ref{prop: eigen values}]
We check the Schur unit-disk stability test Lemma~\ref{lem: Schur unit-disk stability test for a quadratic} for the characteristic polynomial ~\eqref{eq: characteristic polynomial of T}:
\begin{align}
    |a_0| &=|\beta|<1 , \\
    1+a_1+a_0 &=  1-(1+\beta-(1-\beta) \eta\lambda) +\beta >0 \Leftrightarrow (1-\beta) \eta\lambda>0 ,\\
    1-a_1+a_0 &=  1+(1+\beta-(1-\beta) \eta\lambda) +\beta>0 \Leftrightarrow \eta\lambda<\frac{2(1+\beta)}{1-\beta} .
\end{align}
This proves Proposition~\ref{prop: eigen values}.
\end{proof}

\begin{proposition}[Sensitivity of Lyapunov equation solution $\boldsymbol{P}$]\label{prop: sensitivity of P}
    Given $\boldsymbol{P}_1, \boldsymbol{P}_2$ with shared parameter $\eta $ and $\beta$ and different $\lambda_1,\lambda_2$.
    Denote 
    \begin{align}
        D(\lambda)&:=2(1+\beta)-(1-\beta) \eta \lambda \\
        C(\lambda_1, \lambda_2)&:=-\frac{2}{\lambda_1\lambda_2\eta^2}\cdot \left(2(1+\beta)^2+\left(\beta^2-1\right)(\lambda_1 + \lambda_2)\eta-\beta \lambda_1\lambda_2\eta^2 \right).
    \end{align}
    Then the difference operator norm can be bounded by:
    \begin{align}
        \|\boldsymbol{P}_1 - \boldsymbol{P}_2\|_{\mathrm{Op}} \le \frac{|\lambda_1-\lambda_2|\eta}{|D(\lambda_1) D(\lambda_2)|}\cdot \sqrt{\frac{8 \beta^2\left(\beta^2+1\right)^2}{(1-\beta)^2}\left(\frac{2 \beta^2}{(1-\beta)^2}+1\right)+C^2(\lambda_1,\lambda_2)}.
    \end{align}
\end{proposition}

\begin{proof}[Proof of Proposition~\ref{prop: sensitivity of P}]
By definition, we have
\begin{align}
    \boldsymbol{P}_{1}-\boldsymbol{P}_{2}&=\frac{(\lambda_1-\lambda_2)\eta}{D(\lambda_1) D(\lambda_2)} \cdot \boldsymbol{M}(\lambda_1, \lambda_2),\\ 
     \boldsymbol{M}(\lambda_1, \lambda_2)&:=\left(\begin{array}{cc}
\frac{4 \beta^2(\beta^2+1)}{(\beta-1)^2}  &-\frac{2 \beta(\beta^2+1)}{1-\beta} \\
\frac{2 \beta(\beta^2+1)}{1-\beta} & C(\lambda_1, \lambda_2)
\end{array}\right),\\
C(\lambda_1, \lambda_2)&:=-\frac{2}{\lambda_1\lambda_2\eta^2}\cdot \left(2(1+\beta)^2+\left(\beta^2-1\right)(\lambda_1 + \lambda_2)\eta-\beta \lambda_1\lambda_2\eta^2 \right).
\end{align}
Therefore we have the following upper bound,
\begin{align}
    \left\|\boldsymbol{P}_{1}-\boldsymbol{P}_{2}\right\|_{\mathrm{Op}}&\le\frac{|\lambda_1-\lambda_2|\eta}{|D(\lambda_1) D(\lambda_2)|}\cdot \|\boldsymbol{M}(\lambda_1, \lambda_2)\|_{F}\\
    &=\frac{|\lambda_1-\lambda_2|\eta}{|D(\lambda_1) D(\lambda_2)|}\cdot \sqrt{\frac{8 \beta^2\left(\beta^2+1\right)^2}{(1-\beta)^2}\cdot\left(\frac{2 \beta^2}{(1-\beta)^2}+1\right)+C^2(\lambda_1,\lambda_2)}.
\end{align}
This completes the proof of Proposition~\ref{prop: sensitivity of P}.
\end{proof}

\subsection{Basics of the River}

Recall that given $w\in\cU$, the projection ODE flow is defined as 
\begin{align}
    \phi(w, 0) = w, \quad \frac{\mathrm{d}}{\mathrm{d}t}\phi(w, t) = -\bps(\phi(w, t))\nabla L(\phi(w, t)),\quad t\geq 0.\label{eq: projection ode flow restate}
\end{align}

\begin{lemma}[Existence of projection onto the river \& further properties]\label{lem: existence of projection onto river and further properties}
    Under Assumptions~\ref{ass: existence} and \ref{ass: regularity}, for any $w$ satisfying $\cB(w, 2g_{\max}/\gamma)\subset\mathcal{U}$, it holds that 
    \vspace{-1mm}
    \begin{enumerate}[nosep, leftmargin=6mm]
        \item $\Phi(w):=\lim_{t\rightarrow\infty}\phi(w, t)$ exists and $\Phi(w)\in\cM$; 
        \item it holds that 
        \begin{align}
            \|w - \Phi(w)\|_2 \leq \frac{2\|\bps(w)\nabla L(w)\|_2}{\gamma + 2\gamma_{\mathrm{flat}}}; 
        \end{align}
        \item the movement along the ODE flow \eqref{eq: projection ode flow restate} decays exponentially, that is,
        \begin{align}
            \|\bps(\phi(w, t))\nabla L(\phi(w, t))\|_2^2 \leq \exp(-\gamma t/2)\cdot \|\bps(w)\nabla L(w)\|_2;
        \end{align}
        \item finally, the Jacobian $\nabla \Phi(w)$ is well defined.
    \end{enumerate}
\end{lemma}

\begin{proof}[Proof of Lemma~\ref{lem: existence of projection onto river and further properties}]
    See Lemmas C.4 and C.5 in \cite{wen2025understanding} for a proof of Lemma~\ref{lem: existence of projection onto river and further properties}.
\end{proof}

\subsection{Analysis of the River and the Projections}

\begin{lemma}[River spinning]\label{lem: river spinning}
    Under Assumptions~\ref{ass: existence} and \ref{ass: regularity}, it holds that for any vector $v$ and weight $w$, 
    \begin{align}
        \|\nabla \bps(w)[v]\|_{\mathrm{Op}}\leq \frac{\gamma\kappa}{g_{\max}}\cdot\|v\|_2.
    \end{align}
    The same conclusion also holds for the flat direction projection $\bpf$.
\end{lemma}

\begin{proof}[Proof of Lemma~\ref{lem: river spinning}]
    Please refer to Lemma C.2 in \cite{wen2025understanding} for a proof of Lemma~\ref{lem: river spinning}.
\end{proof}

\begin{lemma}[Change of projection matrix]\label{lem: change sharp}
    Under Assumptions~\ref{ass: existence} and \ref{ass: regularity}, consider the GD-momentum trajectory \eqref{eq: gd momentum} initialized with $m_0=0$.
    Suppose that $w_j\in\cU$ for all $j\leq k$ and $w_{k,s}\in\cU$ for all $s\in[0,1]$.
    For any $k\in\mathbb{N}$ and $\tau, \tau'\in[0,1]$, it holds that
    \begin{align}
        \|\bps(w_{k,\tau}) - \bps(w_{k,\tau'})\|_{\mathrm{Op}}\leq \eta\gamma\kappa.
    \end{align}
    The same conclusion also holds for the flat direction projection $\bpf$.
\end{lemma}

\begin{proof}[Proof of Lemma~\ref{lem: change sharp}]
    By Lemma~\ref{lem: uniform momentum bound}, $\|m_{k+1}\|_2\leq g_{\max}$.
    Since $w_{k,\tau}-w_{k,\tau'}=-(\tau-\tau')\eta m_{k+1}$, Lemma~\ref{lem: river spinning} gives
    \begin{align}
        \|\bps(w_{k,\tau})-\bps(w_{k,\tau'})\|_{\mathrm{Op}}
        &\leq \int_{\min\{\tau,\tau'\}}^{\max\{\tau,\tau'\}}\|\nabla\bps(w_{k,s})[-\eta m_{k+1}]\|_{\mathrm{Op}}\,\mathrm{d}s\\
        &\leq |\tau-\tau'|\eta\frac{\gamma\kappa}{g_{\max}}\|m_{k+1}\|_2
        \leq \eta\gamma\kappa.
    \end{align}
    The proof for $\bpf$ is identical.
\end{proof}

\begin{lemma}[Tangent direction of the river]\label{lem: tangent direction}
    Under Assumptions~\ref{ass: existence} and \ref{ass: regularity}, for any point $w\in\cM$ on the river, it holds that 
    \begin{align}
        \|\boldsymbol{P}_{\cM}(w)\nabla L(w) - \nabla L(w)\|_2\leq 4\kappa\cdot\|\boldsymbol{P}_{\cM}(w)\nabla L(w) \|_2.
    \end{align}
\end{lemma}

\begin{proof}[Proof of Lemma~\ref{lem: tangent direction}]
    Please see Lemma C.10 in \cite{wen2025understanding} for a proof of Lemma~\ref{lem: tangent direction}.
\end{proof}

\begin{lemma}[Auxiliary inequality 2]\label{lem: auxiliary 2}
    Under Assumptions~\ref{ass: existence} and \ref{ass: regularity},
    let $w$ satisfy $\cB(w,2g_{\max}/\gamma)\subset\cU$.
    Then
    \begin{align}
        \|\bpf(w)\nabla L(w) - \nabla L(\Phi(w)) \|\leq (\gflat + \gamma\kappa)\cdot \|w-\Phi(w)\|_2.
    \end{align}
\end{lemma}

\begin{proof}[Proof of Lemma~\ref{lem: auxiliary 2}]
    Please see Lemma C.11 in \cite{wen2025understanding} for a proof of Lemma~\ref{lem: auxiliary 2}.
\end{proof}

\begin{lemma}[Tangent vector and gradient in the flat direction]\label{lem: tangent vector flat direction}
     Under Assumptions~\ref{ass: existence} and \ref{ass: regularity}, suppose that $w$ with $\cB(w,2g_{\max}/\gamma)\subset\cU$ satisfies that 
     \begin{align}
         \|w - \Phi(w)\|_2\leq \frac{\mathsf{F}(\beta)\cdot\iota\cdot\|\bpf(w)\nabla L(w)\|_2}{\gamma + 2\gflat},
     \end{align}
     for some function $\mathsf{F}(\beta)$ and $\iota>0$. Define
     \begin{align}
         \rho_{\beta,\iota}:=\mathsf{F}(\beta)\cdot\iota\cdot
         \frac{\gflat+\gamma\kappa}{\gamma+2\gflat},
     \end{align}
     and suppose that $\rho_{\beta,\iota}<1$. Then it holds that
     \begin{align}
         \left\|\frac{\mathrm{d}}{\mathrm{d}T}x(T(w)) + \bpf(w)\nabla L(w)\right\|_2 \leq \left(4\kappa + \frac{\rho_{\beta,\iota}(1+4\kappa)}{1-\rho_{\beta,\iota}}\right)\cdot \left\|\frac{\mathrm{d}}{\mathrm{d}T}x(T(w))\right\|_2,
     \end{align}
     where $x(T(w))=\Phi(w)$.
\end{lemma}

\begin{proof}[Proof of Lemma~\ref{lem: tangent vector flat direction}]
    According to Assumption~\ref{ass: existence}, the tangent vector can be alternatively represented as 
    \begin{align}
        \frac{\mathrm{d}}{\mathrm{d}T}x(t)= - \boldsymbol{P}_{\cM}(x(t))\nabla L(x(t)). \label{eq: proof tangent flat 1}
    \end{align}
    By Lemma~\ref{lem: tangent direction}, we have that 
    \begin{align}
        \|\boldsymbol{P}_{\cM}(x(T(w)))\nabla L(x(T(w))) - \nabla L(x(T(w)))\|_2 \leq 4\kappa\cdot \|\boldsymbol{P}_{\cM}(x(T(w)))\nabla L(x(T(w))) \|_2.\label{eq: proof tangent flat 2}
    \end{align}
    By combining \eqref{eq: proof tangent flat 1} and \eqref{eq: proof tangent flat 2}, we have that 
    \begin{align}
        \left\|\frac{\mathrm{d}}{\mathrm{d}T}x(T(w)) + \bpf(w)\nabla L(w)\right\|_2 \leq  4\kappa \cdot \left\|\frac{\mathrm{d}}{\mathrm{d}T}x(T(w))\right\|_2 + \|\bpf(w)\nabla L(w) - \nabla L(\Phi(w))\|_2.  \label{eq: proof tangent flat 3}
    \end{align}
    Now invoking Lemma~\ref{lem: auxiliary 2}, we have that 
    \begin{align}
        \|\bpf(w)\nabla L(w) - \nabla L(\Phi(w)) \|&\leq (\gflat + \gamma\kappa)\cdot \|w-\Phi(w)\|_2 \\
        &\leq \rho_{\beta,\iota}\cdot \|\bpf(w)\nabla L(w)\|_2. \label{eq: proof tangent flat 4}
     \end{align}
    By the triangle inequality and $\rho_{\beta,\iota}<1$, this further gives
     \begin{align}
        \|\bpf(w)\nabla L(w)\|_2
        &\leq \frac{1}{1-\rho_{\beta,\iota}}\cdot \|\nabla L(\Phi(w))\|_2 \\
        &\leq \frac{1+4\kappa}{1-\rho_{\beta,\iota}}\cdot \left\|\frac{\mathrm{d}}{\mathrm{d}T}x(T(w))\right\|_2. \label{eq: proof tangent flat 5}
    \end{align}
    Consequently, with \eqref{eq: proof tangent flat 3}, \eqref{eq: proof tangent flat 4}, and \eqref{eq: proof tangent flat 5}, we have that 
    \begin{align}
        \left\|\frac{\mathrm{d}}{\mathrm{d}T}x(T(w)) + \bpf(w)\nabla L(w)\right\|_2 \leq \left(4\kappa + \frac{\rho_{\beta,\iota}(1+4\kappa)}{1-\rho_{\beta,\iota}}\right)\cdot\left\|\frac{\mathrm{d}}{\mathrm{d}T}x(T(w))\right\|_2.
    \end{align}
    This completes the proof of Lemma~\ref{lem: tangent vector flat direction}.
\end{proof}

\begin{lemma}[Jacobian of projection to the river]\label{lem: change of projection to the river}
    Under Assumptions~\ref{ass: existence} and \ref{ass: regularity}, it holds that for any $w$ such that $\cB(w,2g_{\max}/\gamma)\subset \cU$ and any direction $u$, 
    \begin{align}
        \|\nabla \Phi(w) u - \bpf(w)u\|_2\leq 5\kappa\cdot\|\bpf(w)u\|_2.
    \end{align}
\end{lemma}

\begin{proof}[Proof of Lemma~\ref{lem: change of projection to the river}]
    See Lemmas C.8 and C.9 in \cite{wen2025understanding} for a proof of Lemma~\ref{lem: change of projection to the river}.
\end{proof}

\subsection{Matrix Inequalities}

\begin{lemma}[Eigenvalue perturbation of symmetric matrices (Corollary 4.3.15 in \cite{horn2012matrix})]\label{lem: eigen perturbation}
    Let $\boldsymbol{\Sigma}$ and $\widehat{\boldsymbol{\Sigma}}\in\mathbb{R}^{d\times d}$ be two symmetric matrices with real eigenvalues $\lambda_1\geq \cdots\geq \lambda_d$ and $\widehat{\lambda}_1\geq \cdots\geq \widehat{\lambda}_d$ respectively. Then for any $i\in[d]$, it holds that 
    \begin{align}
        \Big|\lambda_i - \widehat{\lambda}_i\Big| \leq \Big\|\boldsymbol{\Sigma} - \widehat{\boldsymbol{\Sigma}}\Big\|_\mathrm{Op}.
    \end{align}
\end{lemma}

\begin{lemma}[Davis--Kahan $\sin(\theta)$ theorem \citep{davis1970rotation}]
\label{thm:davis-kahan-sin-theta}
Let $\boldsymbol{\Sigma},\widehat{\boldsymbol{\Sigma}}\in\mathbb{R}^{p\times p}$ be symmetric, with eigenvalues
$\lambda_1\ge \cdots \ge \lambda_p$ and $\widehat{\lambda}_1\ge \cdots \ge \widehat{\lambda}_p$ respectively.
Fix $1\le r\le s\le p$, let $d:=s-r+1$, and let
\begin{align}
\boldsymbol{V}=(v_r,v_{r+1},\ldots,v_s)\in\mathbb{R}^{p\times d},
\qquad
\widehat{\boldsymbol{V}}=(\widehat{v}_r,\widehat{v}_{r+1},\ldots,\widehat{v}_s)\in\mathbb{R}^{p\times d}
\end{align}
have orthonormal columns satisfying
\begin{align}
\boldsymbol{\Sigma} v_j=\lambda_j v_j,\qquad
\widehat{\boldsymbol{\Sigma}}\,\widehat{v}_j=\widehat{\lambda}_j\,\widehat{v}_j,
\qquad j=r,r+1,\ldots,s.
\end{align}
Define
\begin{align}
\Delta := \min\Big\{\max\{0,\lambda_s-\widehat{\lambda}_{s+1}\},\;\max\{0,\widehat{\lambda}_{r-1}-\lambda_r\}\Big\},
\end{align}
where $\widehat{\lambda}_0:=+\infty$ and $\widehat{\lambda}_{p+1}:=-\infty$.
Then for any unitary invariant norm $\|\cdot\|_\ast$,
\begin{align}
\Delta\cdot \bigl\|\sin\Theta(\boldsymbol{V},\widehat{\boldsymbol{V}})\bigr\|_\ast \;\le\; \bigl\|\widehat{\boldsymbol{\Sigma}}-\boldsymbol{\Sigma}\bigr\|_\ast.
\end{align}
Here $\Theta(\boldsymbol{V},\widehat{\boldsymbol{V}})\in\mathbb{R}^{d\times d}$ is diagonal with
\begin{align}
\Theta(\boldsymbol{V},\widehat{\boldsymbol{V}})_{j,j}=\arccos(\sigma_j),\quad j\in[d],
\end{align}
and $\Theta(\boldsymbol{V},\widehat{\boldsymbol{V}})_{i,j}=0$ for $i\neq j$, where
$\sigma_1\ge \sigma_2\ge \cdots \ge \sigma_d$ are the singular values of $\widehat{\boldsymbol{V}}^{\top}\boldsymbol{V}$.
The matrix $\sin\Theta(\boldsymbol{V},\widehat{\boldsymbol{V}})$ is defined entrywise by
$[\sin\Theta]_{i,j}=\sin(\Theta_{i,j})$.
\end{lemma}



    





\newpage

\section{Proofs for Improved Analysis of Vanilla GD}
\label{sec: proof gd}

\subsection{Formal Statement of Theorem~\ref{thm: gd}}\label{subsec: formal statement gd}

\begin{theorem}[GD in river-valley loss landscape (improved version of Theorem 3.2 in \citealt{wen2025understanding})]\label{thm: gd formal}
    Suppose Assumptions~\ref{ass: existence} and \ref{ass: regularity} hold. 
    Let $\eta$ be a learning rate such that $\eta<\eta_{\max}^{\mathrm{GD}}$, where
    \begin{align}\label{eq:max_lr_gd_formal}
        \eta_{\max}^{\mathrm{GD}}= \frac{1.9 - 2\gflat\gmax^{-1} - 12\kappa}{\gmax}.
    \end{align}
    Then there exists a time index $T_0$ such that  iteration \eqref{eq: gd} with initialization $w_0\in\cM$ on the river satisfies that for any step $k$, there exists another $T(k)$ s.t. the following two things hold:
    \begin{enumerate}[nosep, leftmargin=6mm]
        \item GD stays close to the river: $\|x(T_0 + T(k))-w_k\|_2\leq \cO(\kappa \cdot g_{\max}/\gamma)$; 
        \item The speed on the river is nearly proportional to $\eta$: $|T(k) - \eta\cdot k| \leq \epsilon\cdot \eta k$, with $\epsilon:=\cO(\kappa + \eta\gflat)$.
    \end{enumerate}
\end{theorem}

\subsection{Proofs for Improved Analysis}\label{subsec: improved analysis}

The only new ingredient needed to improve the largest tolerable learning rate of GD over \citet{wen2025understanding} is the following sharp-direction estimate.
It provides a tighter learning-rate condition under which the sharp-direction dynamics remain stable.
Once this estimate is established, the remaining projection, trajectory-tracking, and time-reparametrization arguments follow the same steps as in the proof of Theorem~\ref{thm: main formal}, specialized to vanilla GD, with all momentum-specific terms removed.
We therefore present only the proof of this key lemma.
In particular, GD has no momentum transient: using the absolute time map $S(t)$ from \eqref{eq: x S t phi}, we take $T_0=S(0)$ and $T(k)=S(k)-T_0$, and integrate the corresponding derivative estimate over $[0,k]$ to obtain $|T(k)-\eta k|\leq\epsilon\eta k$.

\begin{lemma}[Gradient norm in the flat direction dominates]\label{lem: sharp dominance}
    Under Assumptions~\ref{ass: existence} and \ref{ass: regularity}, with learning rate $\eta\leq (1.9 - 2\gflat\gmax^{-1} - 12\kappa)/\gmax$, it holds that 
    for any $k\in\mathbb{N}$ and $\tau\in[0,1]$,
    \begin{align}
        \|\bps(w_{k,\tau})\nabla L(w_{k,\tau})\|_2 \leq 120\kappa \cdot \|\bpf(w_{k,\tau})\nabla L(w_{k,\tau})\|_2,\label{eq: sharp dominance}
    \end{align}
\end{lemma}

\begin{proof}[Proof of Lemma~\ref{lem: sharp dominance}]
We present the proof for $\tau=1$, but the proof holds for general $\tau\in[0,1]$. 
    To facilitate presentation, we break the proof into three steps.

    \paragraph{Step 1: upper bounding $\|\bps(w_{k+1})\nabla L(w_{k+1})\|_2$.}
    By the fundamental theorem of calculus,
    \begin{align}
        &\bps(w_{k+1})\nabla L(w_{k+1}) - \bps(w_{k})\nabla L(w_{k})\\
        &\qquad = -\eta\cdot \int_{0}^1 \nabla \bps(w_{k,\tau})[\nabla L(w_k)]\nabla L(w_{k,\tau}) + \bps(w_{k,\tau})\nabla^2 L(w_{k,\tau})\nabla L(w_{k})\mathrm{d}\tau \\
        &\qquad = -\eta\cdot \int_{0}^1 \nabla \bps(w_{k,\tau})[\nabla L(w_k)]\nabla L(w_{k,\tau}) + \bps(w_{k,\tau})\nabla^2 L(w_{k,\tau})\bps(w_{k,\tau})\nabla L(w_{k})\mathrm{d}\tau\\
        &\qquad = -\eta\cdot \int_{0}^1\bps(w_{k,\tau})\nabla^2 L(w_{k,\tau}) \bps(w_{k,\tau})\bps(w_k)\nabla L(w_{k})\mathrm{d}\tau + \mathrm{Err}_1^{\mathrm{s}} + \mathrm{Err}_2^{\mathrm{s}}.\label{eq: proof lem sharp dominance 0}
    \end{align}
    where the first equality is from the fundamental theorem of calculus, the second equality uses the fact that $\bps(w_{k,\tau})\nabla^2 L(w_{k, \tau})\bpf(w_{k,\tau})=\mathbf{0}$.
    The terms $\mathrm{Err}_1$ and $\mathrm{Err}_2$ are given respectively by
    \begin{align}
        \mathrm{Err}_1^{\mathrm{s}} &= -\eta\cdot \int_{0}^1  \bps(w_{k,\tau})\nabla^2 L(w_{k,\tau})\bps(w_{k,\tau})\bpf(w_{k})\nabla L(w_k)\mathrm{d}\tau,\\
        \mathrm{Err}_2^{\mathrm{s}} &= -\eta\cdot\int_{0}^1 \nabla \bps(w_{k,\tau})[\nabla L(w_k)]\nabla L(w_{k,\tau}) \mathrm{d}\tau.
    \end{align}
    We now upper bound the norms of $\mathrm{Err}_1$ and $\mathrm{Err}_2^{\mathrm{s}}$ respectively. 
    For $\mathrm{Err}_1^{\mathrm{s}}$, we have 
    \begin{align}
        \|\mathrm{Err}_1^{\mathrm{s}}\|_2 & = \eta\cdot \left\|\int_{0}^1  \bps(w_{k,\tau})\nabla^2 L(w_{k,\tau})\bps(w_{k,\tau})\bpf(w_{k})\nabla L(w_k)\mathrm{d}\tau\right\|_2 \\
        & = \eta\cdot \left\|\int_{0}^1  \bps(w_{k,\tau})\nabla^2 L(w_{k,\tau})\Big(\bps(w_{k,\tau}) - \bps(w_k)\Big)\bpf(w_{k})\nabla L(w_k)\mathrm{d}\tau\right\|_2\\
        &\leq \eta\cdot \int_0^1\|\nabla^2 L(w_{k,\tau})\|_2\cdot \|\bps(w_{k,\tau}) - \bps(w_k)\|_2\cdot\|\nabla L(w_k)\|_2\mathrm{d}\tau\\
        &\leq \eta^2\gmax\gamma\cdot \kappa \cdot\|\nabla L(w_k)\|_2,\label{eq: proof lem sharp dominance 1}
    \end{align}
    where the last inequality uses Lemma~\ref{lem: change of projection to the river} and Assumption~\ref{ass: regularity} (2). 
    For $\mathrm{Err}_2$, we have 
    \begin{align}
        \|\mathrm{Err}_2^{\mathrm{s}}\|_2 \leq \eta\gamma\cdot \kappa\cdot \|\nabla L(w_k)\|_2,\label{eq: proof lem sharp dominance 2}
    \end{align}
    where we use Lemma~\ref{lem: river spinning} and item 3 of Assumption~\ref{ass: regularity}. 
    By \eqref{eq: proof lem sharp dominance 0}, \eqref{eq: proof lem sharp dominance 1}, and \eqref{eq: proof lem sharp dominance 2}, we have
    \begin{align}
        &\|\bps(w_{k+1})\nabla L(w_{k+1})\|_2\label{eq: proof lem sharp dominance 2+}\\
        &\qquad \leq \left\|\left(\boldsymbol{I}_d - \eta\cdot \int_{0}^1\bps(w_{k,\tau})\nabla^2 L(w_{k,\tau})\bps(w_{k,\tau})\mathrm{d}\tau\right)\bps(w_{k})\nabla L(w_{k})\right\|_2  \\
        &\qquad\qquad +  \|\mathrm{Err}^{\mathrm{s}}_1\|_2 +  \|\mathrm{Err}^{\mathrm{s}}_2\|_2 \\
        &\qquad \leq \max\Big\{\big|1-\eta\gmax\big|, \big|1-\eta(\gamma + 4\gflat)\big|\Big\}\cdot  \|\bps(w_{k})\nabla L(w_{k})\|_2 \\
        &\qquad \qquad + \eta\gamma\big(1+\eta\gmax\big)\cdot\kappa \cdot\|\nabla L(w_k)\|_2.
    \end{align}
    This upper bounds the gradient norm in the sharp direction. 

    \paragraph{Step 2: lower bounding $\|\bpf(w_{k+1})\nabla L(w_{k+1})\|_2$.} We use a similar way. Consider
    \allowdisplaybreaks
    \begin{align}
        &\bpf(w_{k+1})\nabla L(w_{k+1}) - \bpf(w_{k})\nabla L(w_{k})\\
        &\qquad = -\eta\cdot \int_{0}^1 \nabla \bpf(w_{k,\tau})[\nabla L(w_k)]\nabla L(w_{k,\tau}) + \bpf(w_{k,\tau})\nabla^2 L(w_{k,\tau})\nabla L(w_{k})\mathrm{d}\tau \\
        &\qquad = -\eta\cdot \int_{0}^1 \nabla \bpf(w_{k,\tau})[\nabla L(w_k)]\nabla L(w_{k,\tau}) + \bpf(w_{k,\tau})\nabla^2 L(w_{k,\tau})\bpf(w_{k,\tau})\nabla L(w_{k})\mathrm{d}\tau \\
        &\qquad = -\eta\cdot \int_{0}^1  \bpf(w_{k,\tau})\nabla^2 L(w_{k,\tau})\bpf(w_{k,\tau})\bpf(w_{k})\nabla L(w_{k})\mathrm{d}\tau + \mathrm{Err}_1^{\mathrm{f}} + \mathrm{Err}_2^{\mathrm{f}},\label{eq: proof lem sharp dominance 3}
    \end{align}
    where
    \begin{align}
        \mathrm{Err}_1^{\mathrm{f}} &= -\eta\cdot \int_{0}^1  \bpf(w_{k,\tau})\nabla^2 L(w_{k,\tau})\bpf(w_{k,\tau})\bps(w_{k})\nabla L(w_k)\mathrm{d}\tau,\\
        \mathrm{Err}_2^{\mathrm{f}} &= -\eta\cdot\int_{0}^1 \nabla \bpf(w_{k,\tau})[\nabla L(w_k)]\nabla L(w_{k,\tau}) \mathrm{d}\tau.
    \end{align}
    Similar to \eqref{eq: proof lem sharp dominance 1} and \eqref{eq: proof lem sharp dominance 2}, we can show that 
    \begin{align}
        \|\mathrm{Err}_1^{\mathrm{f}}\|_2 \leq \eta^2\gmax\gamma\cdot \kappa\cdot \|\nabla L(w_k)\|_2,\quad  \|\mathrm{Err}_2^{\mathrm{f}}\|_2 \leq\eta\gamma\cdot \kappa\cdot\|\nabla L(w_k)\|_2.\label{eq: proof lem sharp dominance 4}
    \end{align}
    Thus combining \eqref{eq: proof lem sharp dominance 3} and \eqref{eq: proof lem sharp dominance 4}, we obtain that 
    \begin{align}
        &\|\bpf(w_{k+1})\nabla L(w_{k+1})\|_2\\
        &\qquad \geq \left\|\left(\boldsymbol{I}_d - \eta\cdot \int_{0}^1\bpf(w_{k,\tau})\nabla^2 L(w_{k,\tau})\bpf(w_{k,\tau})\mathrm{d}\tau\right)\bpf(w_{k})\nabla L(w_{k})\right\|_2 -  \|\mathrm{Err}^{\mathrm{f}}_1\|_2 -  \|\mathrm{Err}^{\mathrm{f}}_2\|_2 \\
        &\qquad \geq (1-\eta\gflat)\cdot \|\bpf(w_{k})\nabla L(w_{k})\|_2 - \eta\gamma\big(1+\eta\gmax\big)\cdot\kappa \cdot\|\nabla L(w_k)\|_2.\label{eq: proof lem sharp dominance 4+}
    \end{align}
    This lower bounds the gradient norm in the flat direction. 

    \paragraph{Step 3: summary up.}
    Now with \eqref{eq: proof lem sharp dominance 2+} and \eqref{eq: proof lem sharp dominance 4+}, we have that, for any coefficient $\alpha\in[0,1]$, 
    \begin{align}
        &\|\bps(w_{k+1})\nabla L(w_{k+1})\|_2 - \alpha\cdot \|\bpf(w_{k+1})\nabla L(w_{k+1})\|_2 \\
        &\qquad \leq \max\Big\{\big|1-\eta\gmax\big|, \big|1-\eta(\gamma + 4\gflat)\big|\Big\}\cdot  \|\bps(w_{k})\nabla L(w_{k})\|_2 \\
        &\qquad\qquad - \alpha\cdot  (1-\eta\gflat)\cdot \|\bpf(w_{k})\nabla L(w_{k})\|_2  + 2\eta\gamma\big(1+\eta\gmax\big)\cdot\kappa \cdot\|\nabla L(w_k)\|_2.\label{eq: proof lem sharp dominance 5}
    \end{align}
    By the fact that 
    \begin{align}
        \|\nabla L(w_k)\|_2 \leq \|\bps(w_{k})\nabla L(w_{k})\|_2 + \|\bpf(w_{k})\nabla L(w_{k})\|_2,
    \end{align}
    we can further upper bound the right hand side of \eqref{eq: proof lem sharp dominance 5} by 
    \begin{align}
        &\|\bps(w_{k+1})\nabla L(w_{k+1})\|_2 - \alpha\cdot \|\bpf(w_{k+1})\nabla L(w_{k+1})\|_2 \\
        &\qquad \leq \bigg(2\eta\gamma\big(1+\eta\gmax\big)\kappa + \max\Big\{\big|1-\eta\gmax\big|, \big|1-\eta(\gamma + 4\gflat)\big|\Big\}\bigg)\cdot  \|\bps(w_{k})\nabla L(w_{k})\|_2 \\
        &\qquad \qquad - \Big(\alpha\cdot (1-\eta\gflat) - 2\eta\gamma\big(1+\eta\gmax\big)\kappa\Big)\cdot \|\bpf(w_{k})\nabla L(w_{k})\|_2. \label{eq: proof lem sharp dominance 6}
    \end{align}
    Now we let the right hand side of \eqref{eq: proof lem sharp dominance 6} satisfying the equation
    \begin{align}
        \text{R.H.S. of \eqref {eq: proof lem sharp dominance 6}} &= \bigg(2\eta\gamma\big(1+\eta\gmax\big)\kappa + \max\Big\{\big|1-\eta\gmax\big|, \big|1-\eta(\gamma + 4\gflat)\big|\Big\}\bigg)\\
        &\qquad \cdot \Big(\|\bps(w_{k})\nabla L(w_{k})\|_2 - \alpha\cdot \|\bpf(w_{k})\nabla L(w_{k})\|_2\Big).\label{eq: proof lem sharp dominance 7}
    \end{align}
    and finding for sufficient conditions on $\eta$ to bound $\alpha$. A sufficient condition for \eqref{eq: proof lem sharp dominance 7} is
    \begin{align}
        &\alpha\cdot \bigg(2\eta\gamma\big(1+\eta\gmax\big)\kappa + \max\Big\{\big|1-\eta\gmax\big|, \big|1-\eta(\gamma + 4\gflat)\big|\Big\}\bigg)\\
        &\qquad = \alpha\cdot (1-\eta\gflat) - 2\eta\gamma\big(1+\eta\gmax\big)\kappa,
    \end{align}
    or equivalently, 
    \begin{align}
        &\bigg(1 - \eta\gflat - 2\eta\gamma\big(1+\eta\gmax\big)\kappa - \max\Big\{\big|1-\eta\gmax\big|, \big|1-\eta(\gamma + 4\gflat)\big|\Big\}\bigg)\cdot \alpha \\
        &\qquad = 2\eta\gamma(1+\eta\gmax)\kappa.
    \end{align}
    Let's now narrow our focus on the regime $\eta <2/\gmax$ (other wise in the sharp direction the iteration explores).
    Consider two different sub-cases.
    Firstly, for $0<\eta \leq 2/(\gmax+\gamma + 4\gflat)$, the above equation reduces to
    \begin{align}
        \Big(\eta\gamma + 3\eta\gflat - 2\eta\gamma\big(1+\eta\gmax\big)\kappa\Big)\cdot \alpha = 2\eta\gamma(1+\eta\gmax)\kappa.
    \end{align}
    This solves $\alpha$ as 
    \begin{align}
        \alpha = \frac{2\gamma(1+\eta\gmax)}{\gamma + 3\gflat - 2\gamma(1+\eta\gmax)\kappa}\cdot \kappa \leq \frac{6\gamma}{\gamma - 6\gamma\kappa}\cdot \kappa \leq 7\kappa
    \end{align}
    Secondly, for $\eta> 2/(\gmax+\gamma + 4\gflat)$, the equation reduces to 
    \begin{align}
        \Big(2 - \eta\gflat - 2\eta\gamma\big(1+\eta\gmax\big)\kappa - \eta\gmax\Big)\cdot \alpha = 2\eta\gamma(1+\eta\gmax)\kappa.
    \end{align}
    Letting $\eta$ further satisfy 
    \begin{align}
        \eta < \frac{1.9 - 2\gflat\gmax^{-1} - 12\kappa}{\gmax},\label{eq: proof lem sharp dominance 8}
    \end{align}
    then we can solve $\alpha$ as 
    \begin{align}
        \alpha &= \frac{2\eta\gamma(1+\eta\gmax)}{2 - \eta\gflat - 2\eta\gamma\big(1+\eta\gmax\big)\kappa - \eta\gmax}\cdot \kappa,\label{eq: proof lem sharp dominance 9}
    \end{align}
    and under \eqref{eq: proof lem sharp dominance 8}
    \begin{align}
        \alpha \leq \frac{12}{2-2\gflat\gmax^{-1} - 12\kappa - 1.9 + 2\gflat\gmax^{-1} + 12\kappa}\cdot \kappa \leq 120\kappa.
    \end{align}
    Thus in conclusion, for learning rate $\eta$ satisfying \eqref{eq: proof lem sharp dominance 8}, we can guarantee that
    \begin{align}
        &\|\bps(w_{k+1})\nabla L(w_{k+1})\|_2 - \alpha\cdot \|\bpf(w_{k+1})\nabla L(w_{k+1})\|_2  \\
        &\qquad \leq \bigg(2\eta\gamma\big(1+\eta\gmax\big)\kappa + \max\Big\{\big|1-\eta\gmax\big|, \big|1-\eta(\gamma + 4\gflat)\big|\Big\}\bigg)\\
        &\qquad \qquad \cdot \Big(\|\bps(w_{k})\nabla L(w_{k})\|_2 - \alpha\cdot \|\bpf(w_{k})\nabla L(w_{k})\|_2\Big),
    \end{align}
    for some $\alpha\leq 120\kappa$. Iterating the above inequality and use the fact that 
    \begin{align}
        \|\bps(w_{0})\nabla L(w_{0})\|_2 - \alpha\cdot \|\bpf(w_{0})\nabla L(w_{0})\|_2 = -\alpha \cdot \|\bpf(w_{0})\nabla L(w_{0})\|_2\leq 0,
    \end{align}
    we can conclude the proof.
\end{proof}


\newpage 

\section{More Results and Details of Experiments}\label{sec: experiment appendix}

\subsection{Further Experiment Details}

The training set of the TinyStories dataset is of about $2.2$ million instances, and the validation set is of about $22000$ instances.
Training loss is recorded every $100$ steps on the corresponding training batch.
In the figures below, faint traces show these recorded losses, solid curves show their $900$-step moving averages, and the insets enlarge steps $30{,}000$--$40{,}000$.
Full-validation loss is evaluated every $5{,}000$ steps and shown with filled markers.
All of the experiments are trained using a single NVIDIA H100 (80G PCIe) GPU.

\subsection{More Experiment Results}

We continue the study of Section~\ref{sec: llm experiment} with learning rates
$0.0001$, $0.0002$, $0.0004$, $0.0006$, and $0.0008$, and momentum parameters
$0$, $0.85$, $0.9$, and $0.95$, plotting the corresponding training and
validation loss curves.
In every figure, the left and right panels show training and validation loss, respectively.

\subsubsection{Fix the learning rate and change the momentum}

We first fix the learning rate $\eta$ and compare the loss curves under different choices of momentum $\beta_1$, see Figures~\ref{fig: eta_01} to \ref{fig: eta_08}.
The results match observations in Section~\ref{sec: llm experiment} that momentum enables more stable and faster training under larger learning rates.
The effect is more significant for larger learning rates. 

\subsubsection{Fix the momentum and change the learning rate}
We then fix the momentum $\beta_1$ and compare the loss curves under different choices of learning rate $\eta$, see Figures~\ref{fig: beta_0} to \ref{fig: beta_095}.
The results show two trends, under the selected $\beta_1$: (i) for larger momentum, the training and validation loss curves are more stable (especially for the validation loss), demonstrating the role of stabilizer of momentum as predicted by the theory; 
(ii) for the same momentum $\beta_1$, during the early phase of the training, the larger the learning rate, the lower the training and validation loss, which matches our prediction of the speed on river by theory.

\begin{figure}[!ht]
    \centering
    \captionsetup{font=footnotesize, skip=2pt}
    \includegraphics[width=\textwidth]{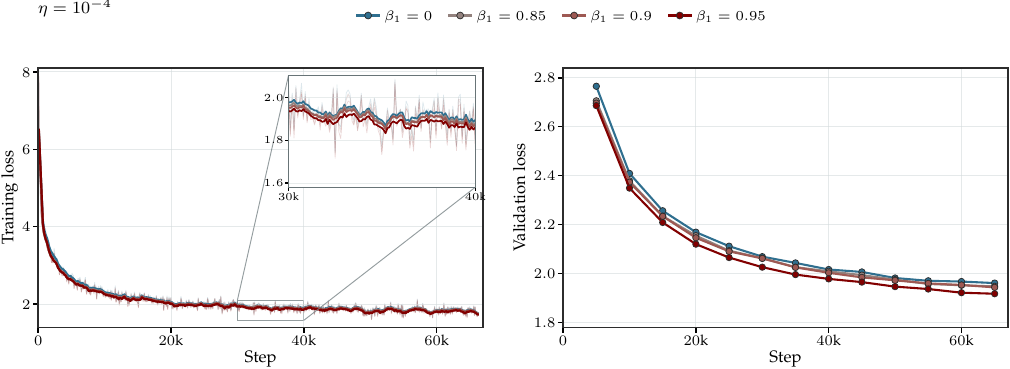}
    \caption{Training and validation losses for $\eta=0.0001$ and $\beta_1\in\{0,0.85,0.9,0.95\}$.}\label{fig: eta_01}
\end{figure}

\begin{figure}[!ht]
    \centering
    \captionsetup{font=footnotesize, skip=2pt}
    \includegraphics[width=\textwidth]{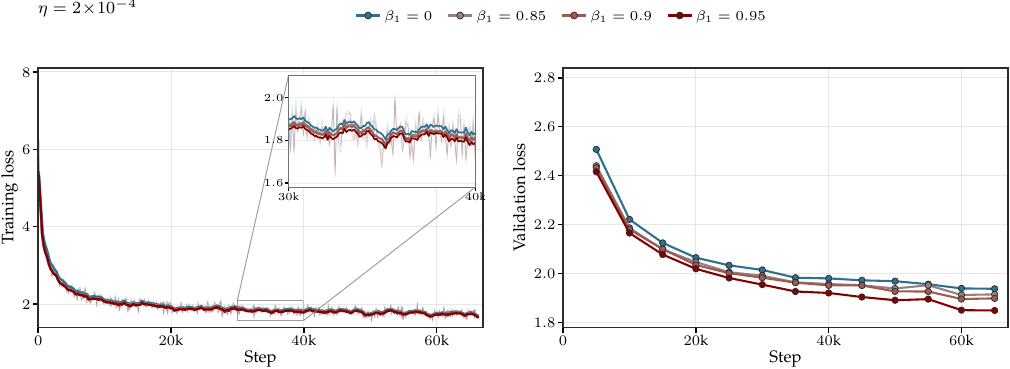}
    \caption{Training and validation losses for $\eta=0.0002$ and $\beta_1\in\{0,0.85,0.9,0.95\}$.}\label{fig: eta_02}
\end{figure}

\begin{figure}[!ht]
    \centering
    \captionsetup{font=footnotesize, skip=2pt}
    \includegraphics[width=\textwidth]{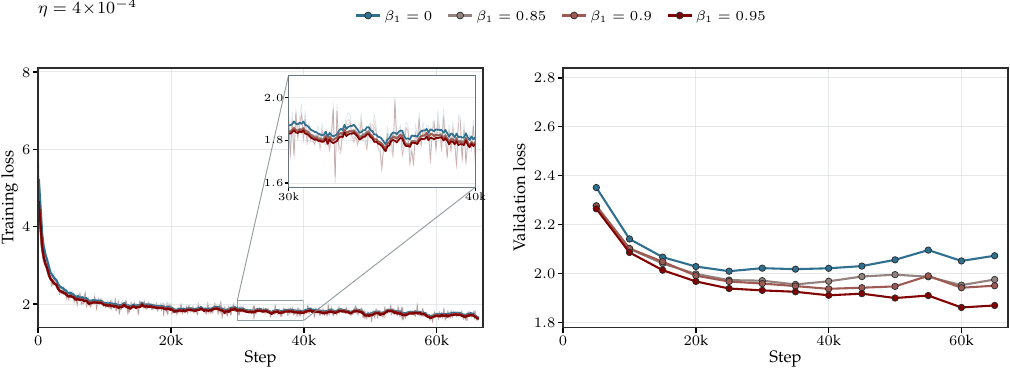}
    \caption{Training and validation losses for $\eta=0.0004$ and $\beta_1\in\{0,0.85,0.9,0.95\}$.}\label{fig: eta_04}
\end{figure}

\begin{figure}[!ht]
    \centering
    \captionsetup{font=footnotesize, skip=2pt}
    \includegraphics[width=\textwidth]{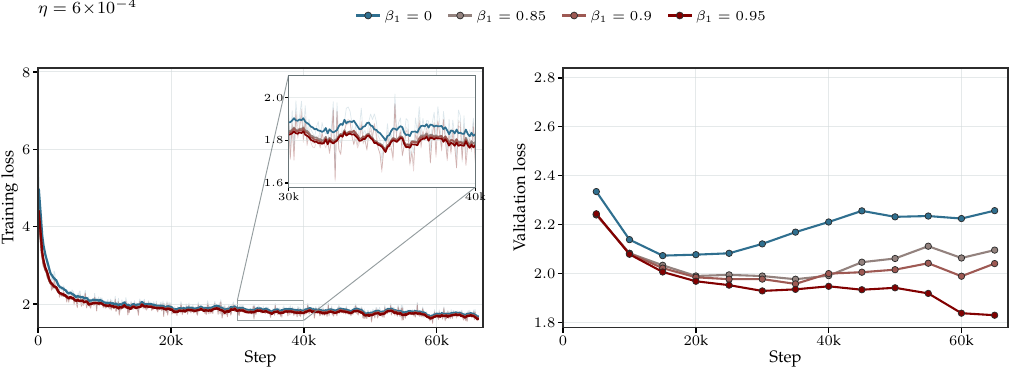}
    \caption{Training and validation losses for $\eta=0.0006$ and $\beta_1\in\{0,0.85,0.9,0.95\}$.}\label{fig: eta_06}
\end{figure}

\begin{figure}[!ht]
    \centering
    \captionsetup{font=footnotesize, skip=2pt}
    \includegraphics[width=\textwidth]{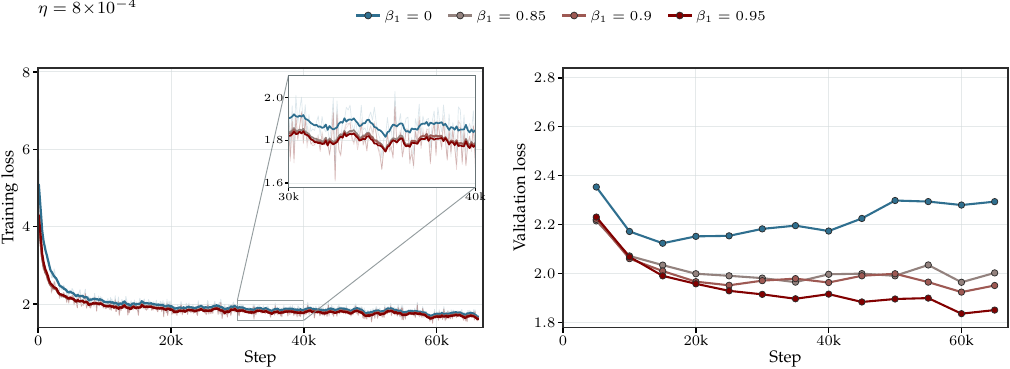}
    \caption{Training and validation losses for $\eta=0.0008$ and $\beta_1\in\{0,0.85,0.9,0.95\}$.}\label{fig: eta_08}
\end{figure}

\newpage

\begin{figure}[!ht]
    \centering
    \captionsetup{font=footnotesize, skip=2pt}
    \includegraphics[width=\textwidth]{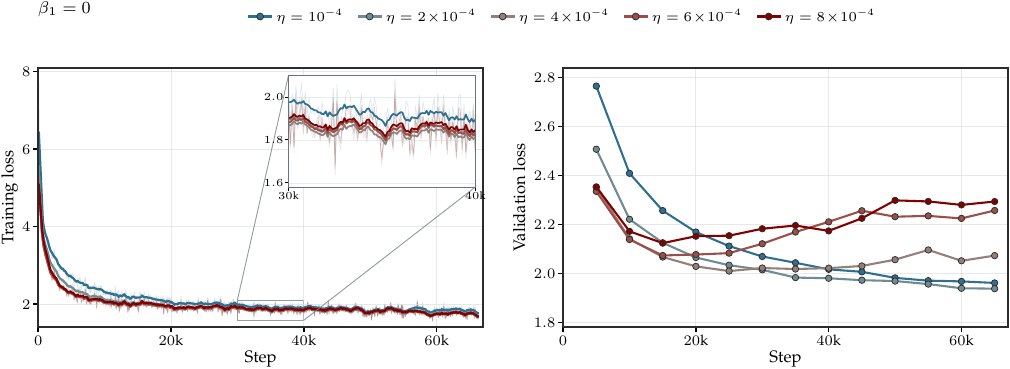}
    \caption{Training and validation losses for $\beta_1=0$ and learning rates $\{1,2,4,6,8\}\times10^{-4}$.}\label{fig: beta_0}
\end{figure}

\begin{figure}[!ht]
    \centering
    \captionsetup{font=footnotesize, skip=2pt}
    \includegraphics[width=\textwidth]{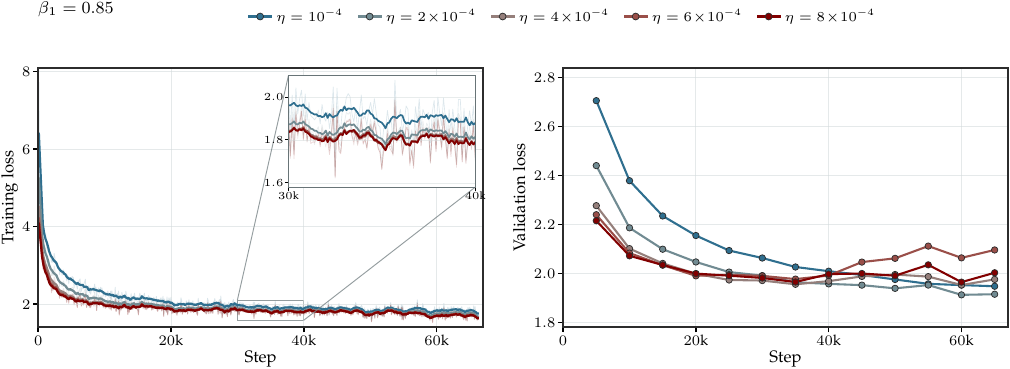}
    \caption{Training and validation losses for $\beta_1=0.85$ and learning rates $\{1,2,4,6,8\}\times10^{-4}$.}\label{fig: beta_085}
\end{figure}

\begin{figure}[!ht]
    \centering
    \captionsetup{font=footnotesize, skip=2pt}
    \includegraphics[width=\textwidth]{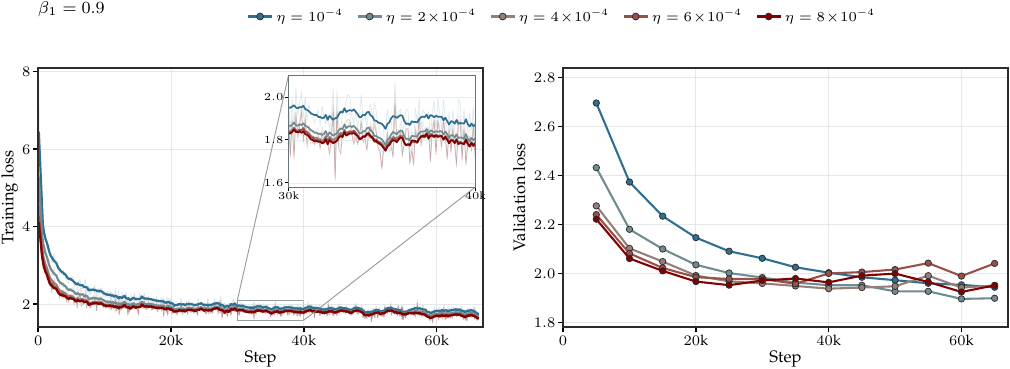}
    \caption{Training and validation losses for $\beta_1=0.9$ and learning rates $\{1,2,4,6,8\}\times10^{-4}$.}\label{fig: beta_09}
\end{figure}

\newpage 

\begin{figure}[!t]
    \centering
    \captionsetup{font=footnotesize, skip=2pt}
    \includegraphics[width=\textwidth]{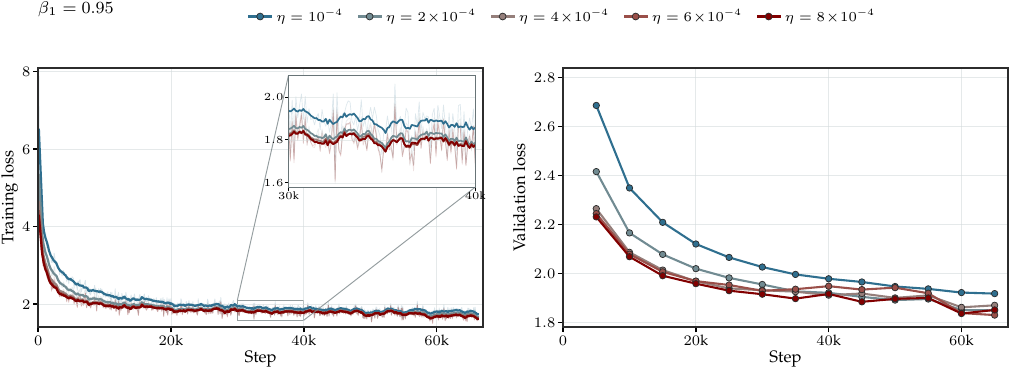}
    \caption{Training and validation losses for $\beta_1=0.95$ and learning rates $\{1,2,4,6,8\}\times10^{-4}$.}\label{fig: beta_095}
\end{figure}

\end{document}